\documentclass[review,onefignum,onetabnum]{siamonline190516}

\usepackage{lipsum}
\usepackage{amsfonts}
\usepackage{graphicx}
\usepackage{epstopdf}
\usepackage{algorithmic}
\ifpdf
  \DeclareGraphicsExtensions{.eps,.pdf,.png,.jpg}
\else
  \DeclareGraphicsExtensions{.eps}
\fi

\usepackage{enumitem}
\setlist[enumerate]{leftmargin=.5in}
\setlist[itemize]{leftmargin=.5in}

\newsiamremark{remark}{Remark}
\newsiamremark{hypothesis}{Hypothesis}
\crefname{hypothesis}{Hypothesis}{Hypotheses}
\newsiamthm{claim}{Claim}

\headers{}{}

\title{Image Segmentation with Adaptive Spatial Priors from Joint Registration
}

\author{Haifeng Li\thanks{Laboratory of Mathematics and Complex Systems (Ministry of Education of China), School of Mathematical Sciences, Beijing Normal University, Beijing, 100875, People's Republic of China.
  (\email{hfli@mail.bnu.edu.cn}, \email{jliu@bnu.edu.cn}, \email{licui@bnu.edu.cn}).}
\and Weihong Guo\thanks{Department of Mathematics, Applied Mathematics and Statistics, Case Western Reserve University, Cleveland, OH, 44106, USA. 
  (\email{wxg49@case.edu}).}
\and Jun Liu\footnotemark[1]
\and Li Cui\footnotemark[1]
\and Dongxing Xie\thanks{Department of Orthopaedics, Xiangya Hospital, Central South University, Changsha, Hunan, 410008, People's Republic of China (87 Xiangya Rd, Changsha, Hunan, China 410008). (\email{xdx1024@csu.edu.cn}).}}

\usepackage{amsopn}

\ifpdf
\hypersetup{
  pdftitle={Image Segmentation with Adaptive Spatial Priors from Joint Registration},
  pdfauthor={Haifeng Li, Weihong Guo, Jun Liu, Li Cui and Dongxing Xie}
}
\fi

\usepackage{subfigure}
\usepackage{caption}
\usepackage{bm}
\usepackage{makecell}
\usepackage{multirow}
\usepackage{amsmath}
\usepackage{autobreak}
\usepackage{graphicx}

\begin{document}

\maketitle

\begin{abstract}
  Image segmentation is a crucial but challenging task that has many applications. In medical imaging for instance, intensity inhomogeneity and noise are common. In thigh muscle images, different muscles are closed packed together and there are often no clear boundaries between them. Intensity based segmentation models cannot separate one muscle from another. To solve such problems, in this work we present a segmentation model with adaptive spatial priors from joint registration. This model combines segmentation and registration in a unified framework to leverage their positive mutual influence. The segmentation is based on a modified Gaussian mixture model (GMM), which integrates intensity inhomogeneity and spacial smoothness. The registration plays the role of providing a shape prior. We adopt a modified sum of squared difference (SSD) fidelity term and Tikhonov regularity term for registration, and also utilize Gaussian pyramid and parametric method for robustness. The connection between segmentation and registration is guaranteed by the cross entropy metric that aims to make the segmentation map (from segmentation) and deformed atlas (from registration) as similar as possible. This joint framework is implemented within a constraint optimization framework, which leads to an efficient algorithm. We evaluate our proposed model on synthetic and thigh muscle MR images. Numerical results show the improvement as compared to segmentation and registration performed separately and other joint models.
\end{abstract}



\section{Introduction}
Image segmentation is a classical problem in image processing. It becomes very challenging when the boundaries between objects of interest are of similar intensity and texture. In medical analysis for example, the segmentation of different thigh muscles is an essential task for the evaluation of musculoskeletal diseases such as osteoarthritis. In addition, noise and intensity inhomogeneity make it even harder. \cref{fig:exampleIM} shows an example of thigh muscle MR image.
\begin{figure}[htp]
\centering
\subfigure{
\begin{minipage}[b]{0.23\textwidth}
\includegraphics[width=1\textwidth]{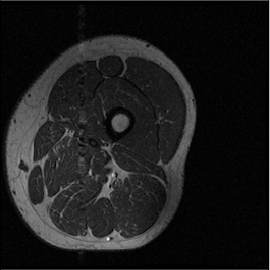}
\end{minipage}}
\subfigure{
\begin{minipage}[b]{0.23\textwidth}
\includegraphics[width=1\textwidth]{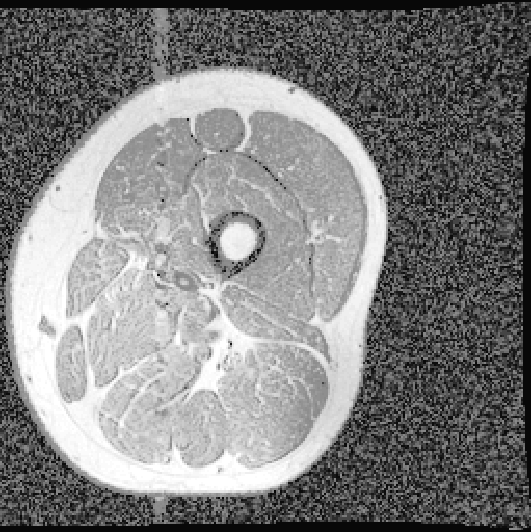}
\end{minipage}}
\subfigure{
\begin{minipage}[b]{0.23\textwidth}
\includegraphics[width=1\textwidth]{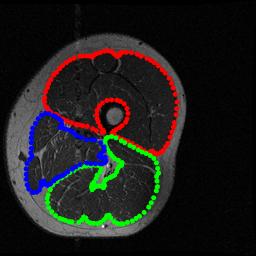}
\end{minipage}}
\caption{A sample axial slice of thigh muscle in T1-weight MR image. Left: original MR image, middle: histogram equalized image, right: manually segmented objects. The muscle groups of quadriceps, hamstrings and others are in red, green, and blue, respectively.}
\label{fig:exampleIM}
\end{figure}

In this case, segmentation methods only depending on intensity can hardly distinguish different muscles, so it is necessary to introduce auxiliary information as priors. So far, several approaches have been devoted to solve such problems. One kind of approaches is semi-automatic methods (e.g. \cite{ahmad2014atlas, jolivet2014skeletal, ogier2017individual, ogier2020novel,molaie2020knowledge}), that obtain priors by the manual delineation of lines, polygons, or manual segmentation of beginning slices, and so on. These semi-automatic methods have proven some efficiency but need manual intervention that can be time consuming and tedious. Another kind of approaches is fully automatic methods. Most of them are shape based methods (e.g. \cite{baudin2012prior, andrews2015generalized, kemnitz2017validation}, and \cite{chen2003using, lotjonen2010fast, sharma2019mammogram} for more general images), or atlas based methods (e.g. \cite{le2016volume, yokota2018automated, mesbah2019novel}). However, these automatic methods are either single segmentation/registration techniques or sequential methods (like ``pipeline" procedures), which ignore the mutual promotion between segmentation and registration. Recently, approaches based on deep learning are emerging, such as using AlexNet network \cite{ghosh2017structured}, U-Net architecture \cite{2020Clinical}, bounding boxes with 3D U-Net \cite{ni2019automatic}, and edge-aware network based on U-Net \cite{guo2021fully}. Although deep learning has great potential in muscle segmentation, we should note that this data driven method needs large amount of annotated data (manual segmentation of thigh muscles has been recognized as time consuming) and a specific network tuning for each dataset, which make it not very suitable and competitive for this task. Considering these facts, in this current work, we devote to propose a fully automatic model-based segmentation model, which integrates adaptive spatial priors from joint registration. 

A joint segmentation and registration model can exploit the strong correlation between them, thus achieving more accurate results. For the segmentation, the registration can be viewed as a prior (e.g. shape or topology prior) to guide the segmentation process. In turn, relevant segmented structures can promote registration by providing a reliable estimation of the deformation between the encoded structures, making the registration not only based on intensity matching (local criterion), but also on geometrical and shape pairing (nonlocal character) \cite{debroux2020variational}. 
There have been some related work discussing about joint segmentation and registration.  Most of them adopt active contour or level set based variational approaches. For example, in the pioneering work, Yezzi et al. \cite{yezzi2001variational} introduced a variational framework that integrates an active contour segmentation model with a rigid registration technique to simultaneously segment and register features from multiple images. In \cite{le2011combined}, the authors combined a matching criterion based on the active contour without edges for segmentation and a nonlinear elasticity based smoother on the displacement vector field to perform joint segmentation and registration. In \cite{atta2018joint}, the authors formulated the joint problem as a minimization of a functional that integrates a nonlinear elastic registration with a geodesic active contours which is introduced together with a weighted total variation term to segment the deforming template image. Pawar et al. \cite{pawar2019joint} presented a joint approach using bidirectional composition based level set formulation, in which the implicit level set function defining the segmentation contour and the displacement field for registration are both defined using B-splines. Swierczynski et al. \cite{swierczynski2018level} proposed an algorithm based on a level-set formulation, which merges a classic Chan-Vese segmentation with an active dense displacement field estimation. Debroux et al. \cite{debroux2018joint} established a joint model which is based on a nonlocal characterization of weighted total variation and nonlocal shape descriptors inspired by the piecewise constant Mumford-Shah model. However, these variational methods have some drawbacks. First, all the above methods \cite{yezzi2001variational, le2011combined, atta2018joint, pawar2019joint, swierczynski2018level, debroux2018joint} are designed for segmentation of two classes, and the extension to multi classes is not straightforward. Second, some of the joint models \cite{yezzi2001variational, atta2018joint, debroux2018joint} are still intensity based methods, as they only consider the pixel value information from two images, without using any atlas as a shape prior. Third, in \cite{le2011combined, swierczynski2018level}, the registration is a matching of segmentation maps, without considering intensity matching of the two images, thus making a rough registration. Therefore, these joint models can not accomplish the segmentation of thigh muscle images with regions of interests are closely packed together without clear boundaries in between. 

In addition to the variational methods, there are also statistical methods for joint segmentation and registration, and most of them are based on the expectation maximization (EM) algorithm. For example, Wyatt et al. \cite{wyatt2003map} applied a Markov random field framework, a mathematical technique for embedding local spatial information, within which they seek to obtain a maximum a posteriori estimate of the segmentation and registration. 
Ashburner et al. \cite{ashburner2005unified} presented a probabilistic framework based on GMM that enables image registration, tissue classification, and intensity bias correction to be combined within the same generative model. 
In \cite{pohl2006bayesian}, the authors developed a statistical model that combines the registration of an atlas with the segmentation of MR images. The model is solved by an EM algorithm which simultaneously estimates image artifacts, anatomical label maps, and a structure-dependent hierarchical mapping from the atlas to the image space. Gooya et al. \cite{gooya2012glistr} presented a generative method for simultaneously registering a probabilistic atlas of a healthy population to brain MR scans showing glioma and segmenting the scans into tumor as well as healthy tissue labels, by integrating the EM algorithm with a glioma growth model, such that EM iteratively refines the estimates of the posterior probabilities of tissue labels, the deformation field and the tumor growth model parameters. 
These statistical methods are more natural and good at dealing with big data, as well as segmentation of multi classes. However, in \cite{wyatt2003map, ashburner2005unified, pohl2006bayesian, gooya2012glistr}, registrations only consider the matching between the segmentation map and the deformed atlas, thus making a rough registration. Besides, statistical methods are susceptible to noise and are not easy to combine with excellent variational properties such as spatial regularization and geometric structures. In contrast, variational methods can incorporate these properties more flexibly. In \cite{Jun2013Image}, the authors proposed a variational framework to solve GMM based methods for image segmentation, thus making it flexible to combine with spatial regularization and bias correction, but it didn't consider a registration. Inspired by the work \cite{ashburner2005unified} and \cite{Jun2013Image}, we combine the GMM based segmentation with a nonrigid registration in a variational framework, which integrates not only the merits of both variational and statistical methods, but also the advantages of joint model. To the best of our knowledge, this is the first time such approach has been used in the joint segmentation and registration framework.

The main work of the current paper can be summarized as follows.
\begin{itemize}
\item We present a fully automatic model for joint segmentation and registration, which is designed for multi classes segmentation. By interpreting the GMM based model in a variational sense, this joint model absorbs the merits of both variational and statistical methods. The segmentation component combines GMM with bias correction and spatial regularization in a variational framework. The registration component consists of a modified SSD and the Tikhonov regularization. The coupling between segmentation and registration is achieved by a cross entropy metric, making the segmentation map as close as possible to the deformed atlas.
\item In this joint model, the registration can provide an adaptive spatial prior for the segmentation process. In turn, the segmentation can promote registration by providing geometrical structures encoded in the segmentation map. Therefore, the registration process considers not only the matching of image intensity, but also geometrical pairing through the matching of the segmentation map and the deformed atlas. This two levels matching make a more accurate registration, and further a more reliable spatial prior for segmentation.
\item This proposed joint model shows improvement performance on synthetic images and thigh muscle MR images as compared to sequential methods (segmentation and registration done separately) and other joint models.
\end{itemize}

The remainder of this paper is organized as follows. In \cref{sec:reviews}, we introduce the related work and the motivation of this work. Then we introduce the proposed joint segmentation and registration model, and its optimization algorithm in \cref{sec:proposedmodel}.  We show some numerical experiments to demonstrate the effectiveness of the proposed model in \cref{sec:experiments}. Finally, a brief conclusion is drawn in \cref{sec:conclusion}.

\section{Related work and motivation}\label{sec:reviews}
In this section, we will review some related work based on GMM, introduce the motivation of our work, and make a description of image registration.

\subsection{Gaussian mixture model}\label{sec:GMM}
Gaussian mixture models have been widely used in many classification problems (e.g. \cite{bishop1995neural} and \cite{mclachlan2019finite}), such as image segmentation. For a gray image $I:\Omega\subset\mathbb{R}^d\rightarrow\mathbb{R}$ (e.g. $d=2$ for two-dimensional images and $d=3$ for three-dimensional images), the probability of pixel intensity $z$ can be modeled as
\begin{equation}\label{eq:GMM}
p(z)=\sum\limits_{k=1}^K\gamma_k p_k(z; c_k,\sigma_k^2), \quad \mathrm{with}\quad \sum\limits_{k=1}^K\gamma_k=1, \gamma_k\geq0,
\end{equation}
where $K$ is the total number of mixtures, $\gamma_k$ is the weight of $k$-th Gaussian component, and $p_k$ is the Gaussian probability density function parameterized by mean $c_k$ and variance $\sigma_k^2$. Denote parameters in GMM as 
\begin{displaymath}
\Theta=\{\gamma_1,\cdots,\gamma_K,c_1,\cdots,c_K,\sigma_1^2,\cdots,\sigma_K^2\},
\end{displaymath} 
and assume that all pixels are independent, then the probability density function of the entire image $I$ can be written as
\begin{displaymath}
p(I|\Theta)=\prod\limits_{x\in\Omega}\sum\limits_{k=1}^K\frac{\gamma_k}{\sqrt{2\pi}\sigma_k}\mathrm{exp}\left\{-\frac{[I(x)-c_k]^2}{2\sigma_k^2}\right\}.
\end{displaymath}
This probability can be maximized with respect to the unknown parameter $\Theta$, by minimizing the following negative log-likelihood function 
\begin{equation}\label{eq:loglikelihood}
\mathcal{L}(\Theta)=-\int_\Omega\log\sum\limits_{k=1}^K\frac{\gamma_k}{\sqrt{2\pi}\sigma_k}\mathrm{exp}\left\{-\frac{[I(x)-c_k]^2}{2\sigma_k^2}\right\}dx.
\end{equation}
The parameters in the GMM can be efficiently estimated using the EM algorithm \cite{mclachlan2007algorithm}, and then the clusters can be obtained through the estimated parameters.

\subsection{Unified segmentation}
A probabilistic framework is presented in \cite{ashburner2005unified} that combines image segmentation, image registration and bias correction within a GMM based model.  

In this work, bias correction is included in GMM by extra parameters that account for smooth intensity variations. The bias field at pixel $x$ is denoted as $\rho(x; \bm\beta)$, where $\bm\beta$ is a vector of unknown parameters. Assume that the bias is multiplicative, and the $k$-th cluster is normally distributed with mean $c_k/\rho(x; \bm\beta)$ and variance $(\sigma_k/\rho(x;\bm\beta))^2$. The negative log-likelihood function \eqref{eq:loglikelihood} can be rewritten as  
\begin{equation}\label{eq:loglikelihoodbias}
\mathcal{L}_{\bm\beta}(\Theta)=-\int_\Omega\log\sum\limits_{k=1}^K\frac{\gamma_k\rho(x;\bm\beta)}{\sqrt{2\pi}\sigma_k}\mathrm{exp}\left\{-\frac{[\rho(x;\bm\beta)I(x)-c_k]^2}{2\sigma_k^2}\right\}dx.
\end{equation}

In GMM, the prior probability of any pixel, irrespective of its intensity, belonging to the k-th Gaussian is the stationary mixing proportion $\gamma_k$. For medical images, additional information usually can be obtained from other subjects' images with tissue probability maps. These maps give the probability of any pixel being of any of the tissue classes. For example, in the brain image segmentation, the tissue probability maps $b_k(x), k=1,2,3,4$, could represent the probability of pixel $x$ being of classes of grey matter, white matter, cerebrospinal fluid (CSF), and ``other'', respectively. The $b_k(x)$ can be further deformed according to a parameter vector $\bm\alpha$, which allows the tissue probability maps register to the image to be segmented. The deformed tissue probability map $b_k(x;\bm\alpha)$ therefore gives a prior probability of any pixel in a registered image being of any of the tissue classes. To combine with this prior, in \cite{ashburner2005unified}, the authors replace the stationary mixing proportion $\gamma_k$ by $\frac{\gamma_kb_k(x;\bm\alpha)}{\sum\limits_{l=1}^K\gamma_lb_l(x;\bm\alpha)}$, which introduces a segmentation prior at each pixel $x$, and the model \eqref{eq:loglikelihoodbias} can be modified to 
\begin{equation}\label{eq:loglikelihoodbiasreg}\nonumber
\mathcal{L}_{\bm\alpha,\bm\beta}(\Theta)=-\int_\Omega\log\sum\limits_{k=1}^K\frac{\gamma_kb_k(x;\bm\alpha)\rho(x;\bm\beta)}{\left[\sum\limits_{l=1}^K\gamma_lb_l(x;\bm\alpha)\right]\sqrt{2\pi}\sigma_k}\mathrm{exp}\left\{-\frac{[\rho(x;\bm\beta)I(x)-c_k]^2}{2\sigma_k^2}\right\}dx.
\end{equation}

Furthermore, to guarantee the smoothness of the bias field and the deformations, two regularization terms $p(\bm\beta)$ and $p(\bm\alpha)$ are considered in \cite{ashburner2005unified}. $p(\bm\beta)$ and $p(\bm\alpha)$ are zero-mean multivariate Gaussian probability density functions with covariances $\Sigma_{\bm\beta}$, $\Sigma_{\bm\alpha}$ respectively. 

Above all, the objective function to be minimized in \cite{ashburner2005unified} can be written as
\begin{equation}\label{eq:UnifiedSeg}
\mathcal{E}(\Theta)=\mathcal{L}_{\bm\alpha,\bm\beta}(\Theta)-\log p(\bm\beta)-\log p(\bm\alpha).
\end{equation}
The unknown parameters are 
$$\Theta=\{\gamma_1,\cdots,\gamma_K,c_1,\cdots,c_K,\sigma_1^2,\cdots,\sigma_K^2,\bm\alpha,\bm\beta\}.$$
The optimization of \eqref{eq:UnifiedSeg} involves the partial derivatives of the objective function and alternating among classification, bias correction and registration \cite{ashburner2005unified}.

This method effectively integrates intensity inhomogeneity and image registration into GMM, which can make great help for tissue classification. However, this method still has some drawbacks. First, the assumption in GMM that pixels are independent is inappropriate, which ignores the smoothness of clusters. Although the spatial priors embody a degree of spatial dependency, it is not enough as illustrated in the simulations in \cite{ashburner2005unified}. Second, although this method could accurately model the intensity inhomogeneity, the EM algorithm lacks spatial regularization and the additional regularization term for the bias parameter $\bm\beta$ is computationally time-consuming. Third, the registration in this model is performed by only matching the segmentation map with tissue probability prior, which is somehow a weak registration. Therefore it needs a preprocessing step to roughly align the images with the tissue probability prior to avoid strange results. At last, this model is embedded in the probabilistic framework of GMM, which makes it less straightforward to incorporate with spatial regularity. 
In this work, we will try to derive a more general framework to tackle these problems by considering the contributions of \cite{Jun2013Image}, introduced in the next \cref{sec:LGMM}.

\subsection{Image segmentation using a local GMM in a variational framework} \label{sec:LGMM}
A variational framework to solve GMM based methods was proposed in \cite{Jun2013Image}. Under this framework, the GMM based method can be extended more easily, regularization terms for instance can be added. To achieve this, the following conclusion of convex analysis \cite{rockafellar1970convex} is used.
\begin{theorem}[Commutativity of Log-sum operations]\label{thm:logsum}
 Given a function $\mathcal{A}_k(x)>0$, for any function $\mathcal{B}_k(x)>0$, we have
$$-\log\sum\limits_{k=1}^K\mathcal{A}_k(x)\mathrm{exp}[-\mathcal{B}_k(x)]=\mathop{\mathrm{min}}\limits_{\bm{ u}\in \mathbb{ U}}\left\{\sum\limits_{k=1}^K[\mathcal{B}_k(x)-\log\mathcal{A}_k(x)]u_k(x)+\sum\limits_{k=1}^Ku_k(x)\log u_k(x)\right\}$$
where $\mathbb{ U}=\left\{\bm{ u}(x)=(u_1(x),u_2(x),\cdots,u_K(x))\in [0,1]^K:\sum\limits_{k=1}^Ku_k(x)=1\right\}.$
\end{theorem}
By defining $\mathcal{A}_k(x):=\frac{\gamma_k}{\sqrt{2\pi}\sigma_k}, ~\mathcal{B}_k(x):=\frac{[I(x)-c_k]^2}{2\sigma_k^2},$ and applying \cref{thm:logsum}, we have that minimizing the negative log-likelihood function \eqref{eq:loglikelihood} is equivalent to minimizing functional 
\begin{equation}\label{eq:VarGMM}
\mathcal{E}(\Theta,\bm u)=\int_\Omega\sum\limits_{k=1}^K\left[\frac{[I(x)-c_k]^2}{2\sigma_k^2}-\log \frac{\gamma_k}{\sqrt{2\pi}\sigma_k}\right]u_k(x)dx+\int_\Omega\sum\limits_{k=1}^Ku_k(x)\log u_k(x)dx.
\end{equation}
This problem can be solved by alternately solving $\bm u$ and $\Theta$ with closed forms. One can easily check that the formulas for solving \eqref{eq:VarGMM} are the same as the EM algorithm for solving \eqref{eq:loglikelihood}.

In this variational form, the spatial regularization can be easily handled. In \cite{Jun2013Image}, the authors adopted a special smoothness term which is related to the boundary length of the clusters, called threshold dynamics regularization \cite{merriman1992diffusion}\cite{merriman1994motion}\cite{esedog2015threshold}, 
\begin{equation}\label{eq:regularization}
 \mathcal{R}(\bm u)=\int_\Omega\sum\limits_{k=1}^Ku_k(x)(\omega\ast(1-u_k))(x)dx.
\end{equation}
The symbol “*” stands for convolution and $\omega$ is a kernel function (usually a Gaussian kernel). Compared to the classical total variation (TV) regularization, a superiority of \cref{eq:regularization} is that it has no singularity and $\bm u$ still has an explicit updating formulation by using a linearization in some sense. 
 
Based on the variational framework \cref{eq:VarGMM} of GMM, combining with regularization \cref{eq:regularization} and a bias field for intensity inhomogeneity, an effective segmentation model is proposed in \cite{Jun2013Image}.

\subsection{The motivation of our work}\label{sec:Motivation}
As mentioned above, on the one hand, the unified segmentation model in \cite{ashburner2005unified} integrates GMM with intensity inhomogeneity and registration prior. However, under this probabilistic framework, it is not easy to combine with some excellent variational properties, such as spatial regularization. Besides, the registration in this model only considers the matching of tissue probability maps, without considering other information such as image intensity, thus making it a rough registration.

On the other hand, the local GMM model in \cite{Jun2013Image} combines GMM with spatial regularization and intensity inhomogeneity in a variational framework. Compared to \cite{ashburner2005unified}, this model can flexibly combine some additional excellent variational properties. However, it lacks spatial priors from an atlas, which is indispensable in some difficult segmentation tasks, such as thigh muscle segmentation. 

Therefore, a combination of both advantages will be more powerful. Inspired by the technique used in \cite{ashburner2005unified}, one can introduce registration into the variational framework in a similar way. First, let us see what the replacement used in \cite{ashburner2005unified} means under the variational framework. For simplicity, denote $\mathcal{H}(\bm u)=\int_\Omega\sum\limits_{k=1}^Ku_k(x)\log u_k(x)dx$, and then replace $\gamma_k$ with $\frac{b_k(x;\bm\alpha)}{\sum\limits_{l=1}^Kb_l(x;\bm\alpha)}$ in \cref{eq:VarGMM} (this is equivalent to the replacement used in \cite{ashburner2005unified}). One has 
\begin{equation}\nonumber
\begin{split}
&\tilde{\mathcal{E}}(\Theta,\bm u)
=\int_\Omega\sum\limits_{k=1}^K\left[\frac{[I(x)-c_k]^2}{2\sigma_k^2}-\log\frac{b_k(x;\bm\alpha)}{\left[\sum\limits_{l=1}^Kb_l(x;\bm\alpha)\right]\sqrt{2\pi}\sigma_k}\right]u_k(x)dx+\mathcal{H}(\bm u)\\
&=\int_\Omega\sum\limits_{k=1}^K\left[\frac{[I(x)-c_k]^2}{2\sigma_k^2}-\log\frac{1}{\sqrt{2\pi}\sigma_k}\right]u_k(x)dx-\int_\Omega\sum\limits_{k=1}^Ku_k(x)\log\frac{b_k(x;\bm\alpha)}{\sum\limits_{l=1}^Kb_l(x;\bm\alpha)}dx+\mathcal{H}(\bm u).
\end{split}
\end{equation}
It is not difficult to find that replacing $\gamma_k$ with $\frac{b_k(x;\bm\alpha)}{\sum\limits_{l=1}^Kb_l(x;\bm\alpha)}$ in the probabilistic framework is equivalent to modify the cross entropy metric in the variational framework, i.e. 
replacing $-\int_\Omega\sum\limits_{k=1}^Ku_k(x)\log\gamma_kdx$ with $-\int_\Omega\sum\limits_{k=1}^Ku_k(x)\log\frac{b_k(x;\bm\alpha)}{\sum\limits_{l=1}^Kb_l(x;\bm\alpha)}dx$. The former cross entropy only considers the volume of classes, not including information at pixels ($\gamma_k$ only depends on $k$ but not $x$). But the latter introduced a prior (deformed tissue probability map) at each pixel $x$ ($b_k(x;\bm\alpha)$ depends on both $k$ and $x$) by combining with a registration problem. 

With this observation, in the variational form, the registration model can be easily extended. In \cref{sec:proposed}, we will present a more general variational framework for joint segmentation and registration, and the two components are connected by a cross entropy metric, in which we use a ground truth segmentation $s_k$ serves as the tissue probability map $b_k$ here.

\subsection{Image registration}\label{sec:registration}
Before presenting the proposed framework, for completeness, we would like to introduce first image registration briefly. Image registration is a process to establish spatial correspondence between different images. The goal of registration is to estimate an optimal displacement field $T:\Omega\subset\mathbb{R}^d\rightarrow\mathbb{R}^d$, that maps a moving image $I_m: \Omega\rightarrow\mathbb{R}$ to a fixed image $I_f: \Omega\rightarrow\mathbb{R}$. Image registration can commonly be formulated as the following problem
$$\mathop{\mathrm{min}}\limits_T E_D(I_f, I_m\circ T)+\eta E_R(T),$$
where $I_m\circ T$ is the deformed image of $I_m$, $E_D$ is an image dissimilarity metric that quantifies the level of alignment between $I_f$ and $I_m\circ T$, and $E_R$ is a regularization term to favor specific properties that the user needs and tackle the difficulty of ill-posed problem, while $\eta$ controls the amount of regularization. A large number of researches have been dedicated to image registration, such as \cite{sederberg1986free}\cite{rueckert2015non}\cite{vishnevskiy2016isotropic}, to name a few. Comprehensive reviews of image registration techniques can be found in \cite{zitova2003image}\cite{sotiras2013deformable}. The selection of registration method depends on the anatomical properties of the tissue and the features of images to be registered. In this paper, we mainly focus on the registration of human thigh muscle images, which usually have small deformation between different subjects. Therefore, in this paper, we just choose a simple registration model to demonstrate the effectiveness of the proposed joint segmentation and registration framework. Specifically, we use the sum of squared difference (SSD) dissimilarity metric and the Tikhonov regularity, that is
\begin{equation}\label{eq:Reg}
\mathop{\mathrm{min}}\limits_{T} \frac{1}{2}\int_\Omega\left[I_f(x)-I_m(x+T(x))\right]^2dx+\frac{\eta}{2}\int_\Omega|\nabla T(x)|^2dx.
\end{equation}
In this case, any gradient-based solver can be used to minimize \cref{eq:Reg} by computing the derivatives of the smooth metric $E_D$ and regularizer $E_R$ with respect to the displacement field $T$. Further, we combine this registration model with Gaussian pyramid and parametric technique for robustness. In small deformation case, this registration method can usually provide a not bad spatial shape prior. More sophisticated registration models, such as diffeomorphic registration, feature based registration, may obtain further benefits and lead to better overall results but that is not the focus of this work, which is reserved for follow-up work.

\section{The proposed joint segmentation and registration model}\label{sec:proposedmodel}
When there is not enough boundary information to separate one region from another, one needs some prior spatial information. We adopt an automatic spatial prior by registering the image to be segmented to an image with ground truth segmentation. Let $I_m:\Omega \subset \mathbb{R}^d \rightarrow \mathbb{R}$ be such an image with ground truth segmentation $\bm{ s}=(s_{1},\cdots, s_{K})$, here $K$ stands for the number of clusters, and $s_k$ the segmentation map for the $k$-th cluster, $k=1,\cdots, K$. Denote $I_f:\Omega\subset\mathbb{R}^d\rightarrow\mathbb{R}$ the given image to be segmented. First, let's introduce the GMM based segmentation model for intensity inhomogeneity image.

\subsection{The GMM with intensity inhomogeneity}\label{LGMMpropose}

As introduced in the introduction and shown in \cref{fig:exampleIM}, thigh muscle MR images usually have the problem of intensity inhomogeneity, which is often caused by a smooth, spatially varying illumination, etc. Although it is not usually a problem for visual inspection, it can hinder the performance of automatic intensity-based image segmentation. According to its contributing factors, many algorithms assume that the bias is multiplicative. In this work, similar to \cite{Jun2013Image}, we model the intensity inhomogeneity image as
\begin{equation}\label{eq:biasimage}
I_f(x)=\beta(x)g(x),
\end{equation}
where $I_f(x)$ is the observed image, $g(x)$ represents the ground truth image and $\beta(x)$ is a smoothly varying intensity bias field. 

By assuming that the true image $g$ can be well segmented by the GMM like \eqref{eq:GMM}, with the relationship \eqref{eq:biasimage}, the probability of pixel intensity $I_f(x)$ can be modeled as
$$p(I_f(x))=\sum\limits_{k=1}^K\frac{\gamma_k}{\sqrt{2\pi}\sigma_k\beta(x)}\mathrm{exp}\left\{-\frac{[I_f(x)-c_k\beta(x)]^2}{2\sigma_k^2\beta^2(x)}\right\}.$$

Besides, since the bias field $\beta$ is varying smoothly, i.e. $\beta(x)\approx\beta(y)$ when $x$ is in a small neighborhood $O_y$ centered at $y$. The intensities of image $I_f$ within the neighborhood $O_y$ share the same probability density function parameterized by $\gamma_k, c_k, \sigma_k^2, \beta(y)$. By the independence assumption, we can get a local negative log-likelihood function in $O_y$
\begin{equation}\nonumber
\mathcal{L}_y(\Theta)=-\int_{O_y}\log\sum\limits_{k=1}^K\frac{\gamma_k}{\sqrt{2\pi}\sigma_k\beta(y)}\mathrm{exp}\left\{-\frac{[I_f(x)-c_k\beta(y)]^2}{2\sigma_k^2\beta^2(y)}\right\}dx,
\end{equation}
here $\Theta=\{\gamma_1,\cdots,\gamma_K,c_1,\cdots,c_K,\sigma_1^2,\cdots,\sigma_K^2, \beta(y)\}$.

If we consider different contributions of each $x\in O_y$ to the local cost function $\mathcal{L}_y$ in terms of the distance to centering point $y$, we can assign a weight for each pixel. A common choice is to use a Gaussian kernel $G_\sigma$ with an appropriate standard deviation $\sigma$, such that $G_\sigma(y-x)\approx0$ when $x\notin O_y$. Then the local cost function in $O_y$ becomes
\begin{equation}\label{eq:localcostfunction}
\mathcal{L}_y(\Theta)=-\int_{\Omega}G_\sigma(y-x)\log\sum\limits_{k=1}^K\mathrm{exp}\left\{\log\frac{\gamma_k}{\sqrt{2\pi}\sigma_k\beta(y)}-\frac{[I_f(x)-c_k\beta(y)]^2}{2\sigma_k^2\beta^2(y)}\right\}dx.
\end{equation}

Note that the above function \eqref{eq:localcostfunction} has the form of log-sum-exp.  Now we introduce the dual formulation of log-sum-exp functional, which can be regarded as a generalization of \cref{thm:logsum}.
\begin{theorem}\label{thm:logsumexp}
Let functional
$$\mathcal{F}(\bm {z})=\varepsilon\log\sum\limits_{k=1}^K\exp\left\{\frac{z_k}{\varepsilon}\right\},$$
then its Fenchel-Legendre transformation is
\begin{equation}\nonumber
\begin{array}{lll}
\mathcal{F}^{*}(\bm u)&=\max\limits_{\bm{z}}\left\{<\bm{z},\bm u>-\mathcal{F}(\bm{z})\right\}\\
&=\left\{
\begin{array}{lll}
\varepsilon\displaystyle\sum\limits_{k=1}^{K}u_{k}\log u_{k},& \bm u\in\mathbb{U},\\
+\infty,& else.
\end{array}
\right.
\end{array}
\end{equation}
where $\mathbb{ U}=\left\{\bm{ u}=(u_1,u_2,\cdots,u_K)\in [0,1]^K:\sum\limits_{k=1}^Ku_k=1\right\}.$ Moreover, $\mathcal{F}(\bm{z})$ is convex with respect to $\bm z$ and thus 
$$\mathcal{F}(\bm{z})=\mathcal{F}^{**}(\bm{z})=\max\limits_{\bm u\in\mathbb{U}}\left\{<\bm{z},\bm u>-\varepsilon\displaystyle\sum\limits_{k=1}^{K}u_k\log u_k\right\}.$$
\end{theorem}
The proof of this theorem is a standard argument of convex optimization, we leave it to the readers. 

With \cref{thm:logsumexp}, problem \eqref{eq:localcostfunction} has the following variational form
\begin{equation}\nonumber
\begin{split}
\mathcal{L}_y(\Theta)=&\min\limits_{\bm u\in\mathbb{U}}
\int_\Omega G_\sigma(y-x)\left[\sum\limits_{k=1}^K\left(\frac{[I_f(x)-c_k\beta(y)]^2}{2\sigma_k^2\beta^2(y)}-\log\frac{\gamma_k}{\sqrt{2\pi}\sigma_k\beta(y)}\right)u_k(x)\right]dx\\
&+\int_\Omega\sum\limits_{k=1}^Ku_k(x)\log u_k(x)dx.
\end{split}
\end{equation} 

Considering the global information, then the total cost function can be written as
\begin{equation}\nonumber
\begin{split}
\mathcal{L}(\Theta)=&\min\limits_{\bm u\in\mathbb{U}}
\int_\Omega\int_\Omega G_\sigma(y-x)\left[\sum\limits_{k=1}^K\left(\frac{[I_f(x)-c_k\beta(y)]^2}{2\sigma_k^2\beta^2(y)}-\log\frac{\gamma_k}{\sqrt{2\pi}\sigma_k\beta(y)}\right)u_k(x)\right]dxdy\\
&+\int_\Omega\sum\limits_{k=1}^Ku_k(x)\log u_k(x)dx.
\end{split}
\end{equation} 
here parameter set $\Theta=\{\gamma_1,\cdots,\gamma_K,c_1,\cdots,c_K,\sigma_1^2,\cdots,\sigma_K^2, \cup_y\beta(y)\}$.

In this formulation, we can easily combine with a regularization like \cref{eq:regularization}. Finally, we have the following image segmentation model
\begin{equation}\label{eq:ProposedSeg}
\min\limits_{\Theta, \bm u\in\mathbb{U}} \hat{\mathcal{E}}(\Theta, \bm u)=\mathcal{F}(\Theta, \bm u)+\lambda\mathcal{R}(\bm u), 
\end{equation}
where
\begin{equation}\nonumber
\begin{split}
\mathcal{F}(\Theta, \bm u)=&
\int_\Omega\int_\Omega G_\sigma(y-x)\left[\sum\limits_{k=1}^K\left(\frac{[I_f(x)-c_k\beta(y)]^2}{2\sigma_k^2\beta^2(y)}-\log\frac{\gamma_k}{\sqrt{2\pi}\sigma_k\beta(y)}\right)u_k(x)\right]dxdy\\
&+\int_\Omega\sum\limits_{k=1}^Ku_k(x)\log u_k(x)dx,
\end{split}
\end{equation} 
$$\mathcal{R}(\bm u)=\int_\Omega\sum\limits_{k=1}^Ku_k(x)(\omega\ast(1-u_k))(x)dx.$$

The minimizer of $\hat{\mathcal{E}}(\Theta, \bm u)$ can be computed by the following alternating algorithm
\begin{equation}\label{eq:iteration1}
\left\{
  \begin{array}{ll}
    \Theta^{t+1}=\mathop{\mathrm{arg~min}}\limits_{\Theta}~\mathcal{F}(\Theta, \bm u^t), & \hbox{}\\
    \bm u^{t+1}=\mathop{\mathrm{arg~min}}\limits_{\bm u\in\mathbb{U}}~\mathcal{F}(\Theta^{t+1}, \bm u)+\lambda\mathcal{R}(\bm u; \bm u^t), & \hbox{} 
  \end{array} \right.
\end{equation}
where $t=0,1,2,\cdots$ stands for the iteration number, and
\begin{equation}\nonumber
\mathcal{R}(\bm u; \bm u^t)=\int_\Omega\sum\limits_{k=1}^Ku_k(x)(\omega\ast(1-2u^t_k))(x)dx
\end{equation} 
is the linearization of $\mathcal{R}(\bm u)$ at $\bm u^t$. We can get the following energy descent theorem.

\begin{theorem}[Energy descent]\label{thm:EnergyDescent}
The sequence $(\Theta^t, \bm u^t)$ produced by iteration scheme \cref{eq:iteration1} satisfies
$$\hat{\mathcal{E}}(\Theta^{t+1}, \bm u^{t+1})\leq\hat{\mathcal{E}}(\Theta^{t}, \bm u^{t}).$$
\end{theorem}
\begin{proof} 
See \cref{sec:proof}.
\end{proof} 

Compared with segmentation model \cref{eq:UnifiedSeg} in \cite{ashburner2005unified}, the proposed segmentation model \cref{eq:ProposedSeg} integrates GMM with a spatial regularization, and this model doesn't need an extra regularization term for the bias field $\beta$ by using an Gaussian kernel controls the smoothness of the bias field.
 
\subsection{The proposed joint model}\label{sec:proposed}
With the segmentation model introduced above and the motivation described in \cref{sec:Motivation}, our proposed variational framework for joint segmentation and registration is to optimize the following energy function 
\begin{equation}\label{eq:proposed}
\mathcal{E}(\Theta,\bm{ u},T)=\mathcal{E}_{\mathrm{Seg}}(\Theta, \bm{ u})+\mathcal{E}_{\mathrm{CE}}(\bm{ u},\bm{ s}\circ T)+\mathcal{E}_{\mathrm{Reg}}(\Theta, T),
\end{equation}
where
\begin{equation}\nonumber
\begin{split}
\mathcal{E}_{\mathrm{Seg}}(\Theta, \bm{ u})=&
\int_\Omega\int_\Omega G_\sigma(y-x)\left[\sum\limits_{k=1}^K\left(\frac{[I_f(x)-c_k\beta(y)]^2}{2\sigma_k^2\beta^2(y)}-\log\frac{\gamma_k}{\sqrt{2\pi}\sigma_k\beta(y)}\right)u_k(x)\right]dxdy\\
&+\varepsilon\int_\Omega\sum\limits_{k=1}^Ku_k(x)\log u_k(x)dx+\lambda\int_\Omega\sum\limits_{k=1}^Ku_k(x)(\omega\ast(1-u_k))(x)dx,
\end{split}
\end{equation} 
$$\mathcal{E}_{\mathrm{CE}}(\bm{ u},\bm{ s}\circ T)=-\xi\int_\Omega\sum\limits_{k=1}^Ku_k(x)\log\frac{s_k(x+T(x))}{\sum\limits_{l=1}^Ks_{l}(x+T(x))}dx,$$
$$\mathcal{E}_{\mathrm{Reg}}(\Theta, T)=\frac{\zeta}{2}\int_\Omega\left[\frac{I_f(x)}{\beta(x)}-I_m(x+T(x))\right]^2dx+\frac{\eta}{2}\int_\Omega|\nabla T(x)|^2dx,$$
and $\Theta=\{\gamma_1,\cdots,\gamma_K,c_1,\cdots,c_K,\sigma_1^2,\cdots,\sigma_K^2,\cup_y\beta(y)\}$. Parameters $\varepsilon, \lambda, \xi, \zeta, \eta$ are used to balance the weight of different terms. 

In this model, the first part $\mathcal{E}_{\mathrm{Seg}}(\Theta, \bm{ u})$ is designed to segment the fixed image $I_f$ through an extended GMM, which integrates intensity inhomogeneity and spatial regularization into the classical GMM. After explaining GMM based methods in this variational formulation, it is more flexible to cooperate with other variational models, such as image registration, as shown in this framework.

The third part $\mathcal{E}_{\mathrm{Reg}}(\Theta, T)$ is a registration model that maps the moving image $I_m$ to the intensity corrected image $I_f/\beta$. Here the registration model can come in various and be chosen according to the characteristics of the specific dataset. In this work, we just use a simple registration model ( i.e. composed of a modified SSD and Tikhonov regularization) as an example to show the effectiveness of this joint framework. 

The second part $\mathcal{E}_{\mathrm{CE}}(\bm{ u},\bm{ s}\circ T)$ plays a bridge role between segmentation and registration, which enforces the segmentation map $\bm u$ and the normalized deformed tissue probability map $\bm s\circ T$ as similar as possible under the cross entropy metric. For the segmentation, $\bm s\circ T$ can be viewed as a prior to guid the segmentation process. In turn, the segmentation map $\bm u$ can promote the registration procedure by providing a reliable geometric structure, since the registration procedure considers not only the alignment of image intensities but also segmentation maps. Here, we choose cross entropy as the similarity metric for several reasons. First, as interpreted in \cref{sec:Motivation}, this cross entropy constraint is equivalent to replacing the constant mixture ratio $\gamma_k$ in GMM by tissue probability priors $\frac{s_k(x+T(x))}{\sum\limits_{l=1}^Ks_l(x+T(x))}$, which has a close relationship with GMM based segmentation model. Second, it can well measure the similarity between distributions and is widely used in machine learning. At last, the cross entropy is linear with respect to the segmentation map $\bm u$ and derivable with respect to the displacement field $T$, thus won't bring additional difficulties in solving $\bm u$ and $T$. 

Overall, we propose a joint model for segmentation and registration to leverage their positive mutual influence. Compared with other joint models, the proposed model has some advantages. First, for the segmentation component, it efficiently combines GMM with intensity correction (doesn't need an extra regularization for the bias field), spatial regularization (based on variational regularization) and adaptive spatial priors from variational based registration. This model integrates the merits of probabilistic based and variational based methods, and can be used for multi class segmentation. Second, for the registration component, two levels of matching are considered, i.e. geometrical structure matching (between the segmentation map and the normalized deformed probability map) and intensity matching (between the intensity corrected image and the moving image). Thus this registration can provide more reliable priors and further promote the segmentation. 
Therefore, with intensity correction, spatial regularization, and spatial priors from joint registration, this proposed model has the potential to find an accurate segmentation of images which are degraded by intensity inhomogeneity, noise, and weak boundaries, such as thigh muscle MR images.

\subsection{Optimization}\label{sec:optimization}
In this section, the optimization scheme for minimizing objective function \cref{eq:proposed} is described. This is implemented within a constraint optimization framework, using the following iterated conditional modes,
$$\left\{
  \begin{array}{ll}
    \Theta^{t+1}=\mathop{\mathrm{arg~min}}\limits_{\Theta}~\mathcal{E}(\Theta,\bm u^{t},T^{t}), & \hbox{(i)}\\
    \bm u^{t+1}=\mathop{\mathrm{arg~min}}\limits_{\bm u\in\mathbb{U}}~\mathcal{E}(\Theta^{t+1},\bm u, T^t), & \hbox{(ii)} \\
    T^{t+1}=\mathop{\mathrm{arg~min}}\limits_{T}~\mathcal{E}(\Theta^{t+1},\bm u^{t+1}, T), & \hbox{(iii)} 
  \end{array}
\right.$$
where $t$ is the iteration number.

\textbf{(i)} The $\Theta$-subproblem is to solve the following optimization problem, which is similar to the M-step in EM algorithm.
\begin{equation}\nonumber
\begin{split}
\mathop{\mathrm{arg~min}}\limits_{\Theta}~&
\int_\Omega\int_\Omega G_\sigma(y-x)\left[\sum\limits_{k=1}^K\left(\frac{[I_f(x)-c_k\beta(y)]^2}{2\sigma_k^2\beta^2(y)}-\log\frac{\gamma_k}{\sqrt{2\pi}\sigma_k\beta(y)}\right)u^t_k(x)\right]dxdy\\
&+\frac{\zeta}{2}\int_\Omega\left[\frac{I_f(x)}{\beta(x)}-I_m(x+T^t(x))\right]^2dx. 
\end{split}
\end{equation} 
Recall that $\sum\limits_{k=1}^K\gamma_k^{t+1}=1$, one can easily get the updating formulations, 
\begin{equation}\label{eq:SolveTheta}
\left\{
  \begin{array}{ll}
    \gamma_k^{t+1}=\frac{\int_\Omega u_k^t(x)dx}{\int_\Omega 1 dx},\\
    \quad\\
    c_k^{t+1}=\frac{\int_\Omega u_k^t(x)I_f(x)\int_\Omega G_\sigma(y-x)\frac{1}{\beta^t(y)}dydx}{\int_\Omega u_k^t(x)dx},\\
    \quad\\
    (\sigma_k^2)^{t+1}=\frac{\int_\Omega u_k^t(x)\int_\Omega G_\sigma(y-x)\left[\frac{I_f(x)}{\beta^t(y)}-c_k^{t+1}\right]^2dydx}{\int_\Omega u_k^t(x)dx}, \\
    \quad\\
    \beta^{t+1}(y)=\frac{-[s^{t+1}(y)+p^{t+1}(y)]+\sqrt{[s^{t+1}(y)+p^{t+1}(y)]^2+4[v^{t+1}(y)+\zeta I_f^2(y)]}}{2},
  \end{array}
\right.
\end{equation}
where
\begin{equation}\nonumber
\begin{split}
&s^{t+1}(y)=\sum\limits_{l=1}^K\frac{c_l^{t+1}}{(\sigma_l^2)^{t+1}}\int_\Omega G_\sigma(y-x)I_f(x)u_l^t(x)dx,\\
&v^{t+1}(y)=\sum\limits_{l=1}^K\frac{1}{(\sigma_l^2)^{t+1}}\int_\Omega G_\sigma(y-x)I_f^2(x)u_l^t(x)dx,\\
&p^{t+1}(y)=\zeta I_f(y)I_m(y+T^t(y)).
\end{split}
\end{equation} 

\textbf{(ii)} The $\bm u$-subproblem is to optimize the following problem, which can be considered as a regularized E-step with priors in EM algorithm.
\begin{small}
\begin{equation}\nonumber
\begin{split}
\mathop{\mathrm{arg~min}}\limits_{\bm u\in\mathbb{U}}~&\int_\Omega\int_\Omega G_\sigma(y-x)\left[\sum\limits_{k=1}^K\left(\frac{[I_f(x)-c_k^{t+1}\beta^{t+1}(y)]^2}{2(\sigma_k^{t+1})^2(\beta^{t+1}(y))^2}-\log\frac{\gamma^{t+1}_k}{\sqrt{2\pi}\sigma^{t+1}_k\beta^{t+1}(y)}\right)u_k(x)\right]dxdy\\
&+\varepsilon\int_\Omega\sum\limits_{k=1}^Ku_k(x)\log u_k(x)dx+\lambda\int_\Omega\sum\limits_{k=1}^Ku_k(x)(\omega\ast(1-2u_k^t))(x)dx\\
&-\xi\int_\Omega\sum\limits_{k=1}^Ku_k(x)\log\frac{s_k(x+T^t(x))}{\sum\limits_{l=1}^Ks_{l}(x+T^t(x))}dx.
\end{split}
\end{equation} 
\end{small}
Here we use the linearization of the regularization term. Now the $\bm u$-subproblem is convex with respect to $\bm u$, and one can easily get the updating formulation
\begin{equation}\label{eq:SolveU}
u_k^{t+1}(x)=\frac{q_k^{t+1}(x)}{\sum\limits_{\ell=1}^Kq_\ell^{t+1}(x)},
\end{equation}
where
\begin{small}
\begin{equation}\nonumber
\begin{split}
q_k^{t+1}(x)=&\left(\frac{\gamma_k^{t+1}}{\sigma_k^{t+1}}\right)^\frac{1}{\varepsilon}\left(\frac{s_k(x+T^t(x))}{\sum\limits_{l=1}^K s_l(x+T^t(x))}\right)^\frac{\xi}{\varepsilon}\mathrm{exp}\left\{-\frac{1}{2\varepsilon(\sigma_k^{t+1})^{2}}\int_\Omega G_\sigma(y-x)\left[\frac{I_f(x)}{\beta^{t+1}(y)}-c_k^{t+1}\right]^2dy \right. \\
&\left. -\frac{\lambda}{\varepsilon} (\omega\ast(1-2u_k^t))(x)\right\}.
\end{split}
\end{equation}
\end{small}

\textbf{(iii)} The $T$-subproblem is to solve the following registration problem.
\begin{small}
\begin{equation}\label{eq:Tsubproblem}
\begin{split}
\mathop{\mathrm{arg~min}}\limits_{T}~&-\xi\int_\Omega\sum\limits_{k=1}^Ku^{t+1}_k(x)\log\frac{s_k(x+T(x))}{\sum\limits_{l=1}^Ks_{l}(x+T(x))}dx+\frac{\zeta}{2}\int_\Omega\left[\frac{I_f(x)}{\beta^{t+1}(x)}-I_m(x+T(x))\right]^2dx\\
&+\frac{\eta}{2}\int_\Omega|\nabla T(x)|^2dx.
\end{split}
\end{equation} 
The partial derivative with respect to $T(x)$ is
\begin{equation}\label{eq:SolveT}
\begin{split}
\frac{\partial \mathcal{E}(\Theta^{t+1},\bm u^{t+1}, T)}{\partial T(x)}&=
-\xi\sum\limits_{k=1}^Ku^{t+1}_k(x)\frac{\nabla s_k(x+T(x))}{s_k(x+T(x))}+\xi\frac{\sum\limits_{l=1}^K \nabla s_{l}(x+T(x))}{\sum\limits_{l=1}^K s_{l}(x+T(x))}\\
&+\zeta\left[I_m(x+T(x))-\frac{I_f(x)}{\beta^{t+1}(x)}\right]\nabla I_m(x+T(x))-\eta\triangle T(x).
\end{split}
\end{equation}
\end{small}
It is a smooth optimization problem that can be solved using any gradient-based optimization technique. In this work, we solve it with the quasi-Newton limited-memory BFGS (LBFGS) method as in \cite{vishnevskiy2016isotropic}. Since non-parametric registration methods are highly susceptible to local minima and not robust during optimization, we employ a parametric method to reduce the dimensionality of the registration optimization problem, thus providing physically plausible transformations in robust schemes with large displacement capture ranges. Concretely, as in \cite{vishnevskiy2016isotropic} and \cite{vishnevskiy2014total}, we parametrize the displacement field $T$ using bilinear (for 2D images) or trilinear (for 3D images) interpolation with displacements $D$ on control points that are placed with $N$-pixel spacing. Then impose regularization on the displacement control grid points $D$ instead of the displacement field $T$ itself. In addition, Gaussian pyramid is used to downsample images, and the registration is implemented from coarse-to-fine. At each consecutive image level, the displacement is initialized by interpolating from the previous level’s control grid displacements. We implement the registration details based on a modified version of the MATLAB toolbox \texttt{pTVreg} \cite{toolbox-web}. 

In this optimization process, the first two steps (i) and (ii) are to solve the segmentation problem, and the last step (iii) is to solve the registration problem. Therefore, this process can be considered as alternating update segmentation and registration, and is summarized in \cref{alg:algorithm}.

\renewcommand{\algorithmicrequire}{\textbf{Input:}}
\renewcommand{\algorithmicensure}{\textbf{Output:}}
\begin{algorithm}[htb]
\caption{The proposed joint segmentation and registration algorithm.}\label{alg:algorithm}
\begin{algorithmic}
\REQUIRE{$I_f$, $I_m$, $s$, $\varepsilon$, $\lambda$, $\xi$, $\zeta$, $\eta$;
 $L$ - number of image pyramid levels; $M_{\mathrm{iter}}$ - max number of alternate iterations, $M_{\mathrm{LGMM}}$ - max number of LGMM iterations, $M_{\mathrm{LBFGS}}$ - max number of LBFGS iterations, $Tol$ - tolerance error.}
\ENSURE{Segmentation map $\bm u$ or deformed tissue probability map $ s\circ T$.}
\renewcommand{\algorithmicensure}{\textbf{Initialization:}}
\ENSURE{$ u^0$ by K-means, $\beta^0=1$, $T^0=0$.}
\STATE{Compute objective function $\mathcal{E}^0$.}
\FOR{$t=0,1,2,\cdots, M_{\mathrm{iter}}-1$}
\STATE{$ u^{t+1,0}= u^{t}$, $\beta^{t+1,0}=\beta^t$.}
\FOR{$i=0,1,2,\cdots, M_{\mathrm{LGMM}}-1$}
\STATE{Compute $\Theta^{t+1,i+1}$ by \cref{eq:SolveTheta} with $ u^{t+1,i}$, $\beta^{t+1,i}$, $T^t$.}
\STATE{Compute $ u^{t+1,i+1}$ by \cref{eq:SolveU} with $ u^{t+1,i}$, $\Theta^{t+1,i+1}$, $T^t$.}
\ENDFOR
\STATE{$ u^{t+1}= u^{t+1,M_{\mathrm{LGMM}}}$, $\beta^{t+1}=\beta^{t+1,M_{\mathrm{LGMM}}}$.}
\STATE{Initialize parametric displacement field $D^{t+1,L}$(the coarsest level) with $T^t$.}
\FOR{$j=1,\cdots,L$}
\STATE{Compute $I_m^j$: ($L-j+1$)-th Gaussian pyramid level of $I_m$.}
\STATE{Compute $I_f^j$: ($L-j+1$)-th Gaussian pyramid level of $\frac{I_f}{\beta^{t+1}}$.}
\STATE{Compute $s^j$: ($L-j+1$)-th Gaussian pyramid level of $s$.}
\STATE{Compute $u^{t+1,j}$: ($L-j+1$)-th Gaussian pyramid level of $u^{t+1}$.}
\STATE{Compute $D^{t+1,L-j+1}$: solution of \cref{eq:Tsubproblem} with LBFGS optimization scheme (using $I_m^j$, $I_f^j$, $s^j$, $u^{t+1,j}$, $M_{\mathrm{LBFGS}}$).}
\STATE{Initialize $D^{t+1,L-j}$: linealy upsample $D^{t+1,L-j+1}$.}
\ENDFOR
\STATE{Compute $T^{t+1}$ with parametric displacement field $D^{t+1,1}$.}
\STATE{Compute objective function $\mathcal{E}^{t+1}$.}
\STATE{Convergence check. If $\frac{(\mathcal{E}^{t+1}-\mathcal{E}^t)^2}{(\mathcal{E}^t)^2}< Tol$, break.}
\ENDFOR
\STATE{Segmentation function:
$$label(I_f(x))=\mathop{\mathrm{arg~max}}\limits_{1\leq k\leq K}\{ u_{k}(x)\} ~~ \mathrm{or} ~~label(I_f(x))=\mathop{\mathrm{arg~max}}\limits_{1\leq k\leq K}\{ s_{k}(x+T(x))\}.$$}
\RETURN{Segmentation result $label(I_f(x))$.}
\end{algorithmic}
\end{algorithm}

\subsection{Modification for thigh muscle segmentation}\label{sec:ModForMuscle}
In our proposed model \cref{eq:proposed}, the cross entropy metric enforces the segmentation map $\bm{u}$ and the normalized deformed tissue probability map $\bm{s}\circ T$ as similar as possible for every class $k$ ($k=1,\cdots,K$). 
For different applications, the proposed model can be modified to accommodate the specific needs.

In thigh muscle MR images, as shown in \cref{fig:exampleIM}, there are three main intensity classes: ``dark" (background, cortical bone), ``gray" (muscle), and ``light" (fat, bone marrow) \cite{andrews2015generalized}. In our proposed model, the segmentation model $\mathcal{E}_{\mathrm{Seg}}$ is an intensity based method. So it is more suitable to classify a muscle image into three classes (dark, gray, light) but not four classes (one for each of the three muscle groups and one for others) which is actually the goal of muscle segmentation. While for the registration model, the classification is based on a deformed atlas as the shape prior, thus it can distinguish different muscles and classify a muscle image into the desired four classes. Besides, in the proposed model, segmentation and registration play different roles. In some sense, the role of the segmentation model is to obtain accurate segmentation boundaries, and the registration model is to provide accurate segmentation categories, which can be illustrated by the later simulations in \cref{sec:TestSynthetic}. Therefore, the classification numbers of segmentation and registration could be different. In \cref{sec:TestMuscle}, for the thigh muscle segmentation, we set the classification number to be three for segmentation, and four for registration. 

 In order to achieve this, we need to modify the cross entropy metric in the proposed model \cref{eq:proposed} to the following form,   
$$\mathcal{E}_{\mathrm{CE}}(u, s\circ T)=-\xi\int_\Omega\sum\limits_{k=1}^2p_k(x; u)\log q_k(x; s\circ T)dx,$$ 
where 
$$p_1(x; u)=u_2(x), ~~ p_2(x; u)=u_1(x)+u_3(x),$$  
$$q_1(x; s\circ T)=\frac{s_1(x+T(x))+s_2(x+T(x))+s_3(x+T(x))}{\sum\limits_{l=1}^4s_{l}(x+T(x))},~~ q_2(x; s\circ T)=\frac{s_4(x+T(x))}{\sum\limits_{l=1}^4s_{l}(x+T(x))}.$$ 
Here, according to intensity, we set $u_1$, $u_2$, $u_3$ to be the segmentation maps for ``dark", ``gray" and ``light", respectively. Meanwhile, we set $s_1+s_2+s_3$, $s_4$ to be the ground truth segmentation maps of the moving image for ``the whole muscle" and ``others", respectively. Therefore, the modified cross entropy metric guarantees the similarity of the segmentation map and normalized deformed atlas for ``the whole muscle" and ``others", not every classes. With this modification, it is not difficult to get all of the updating formulations. This also reflects the flexibility of our proposed model in specific applications.

\section{Numerical experiments}\label{sec:experiments}

\subsection{Parameter selection and implementation details}\label{sec:parameters}
For the segmentation component, the Gaussian kernel $G_\sigma$ in the data fidelity term controls the smoothness of the bias field $\bm \beta$. Bigger $\sigma$ leads to smoother bias field. The entropy parameter $\varepsilon$ controls the smoothness of the segmentation map $\bm u$. Small $\varepsilon$ tends to lead to a binary segmentation, and large $\varepsilon$ tends to lead to a smooth segmentation according to the maximum entropy principle. The kernel function $\omega$ in the regularization term controls the scale of the spatial regularization. The regularization parameter $\lambda$ balances the data fidelity term, the regularization term and the cross entropy term in the segmentation procedure, which is important for obtaining a desired segmentation result. In this paper, without specification, we set $\sigma=20$, $\varepsilon=1$ and the kernel function $\omega$ to be a $7\times7$ matrix with equal weights. The regularization parameter $\lambda$ is image dependent, usually between $0.001$ and $0.01$. 

The cross entropy parameter $\xi$ acts as a bridge between segmentation and registration. According to our experience, the value between $0.001$ and $0.03$ is suitable for most of the cases. 

For the registration component, the fidelity parameter $\zeta$ and the regularization parameter $\eta$ control the weight of the modified SSD and the regularization term respectively. In this paper, we set $\zeta=1$ and make the value of $\eta$ depend on images, usually between $0.001$ and $0.05$. For the Gaussian pyramid multi-scales, the number of levels $L$ can be determined by the size of the input images. In our experiments, we set $L=8$ and $0.7$ for downscaling factor. For the displacement parameterization in each level, the control grid spacing (the gap in pixels between knots) controls the flexibility of the displacement field, and we set it to be $4$ and use bilinear interpolation for 2D (or trilinear for 3D) image warping. 

 In our experiments, both the moving and the fixed images are converted to double precision and mapped to $[0,1]$. We set the maximum iteration number $M_\mathrm{iter}=10$, $M_{\mathrm{LGMM}}=100$, $M_{\mathrm{LBFGS}}=100$, and the tolerance $Tol=10^{-5}$. Although the proposed model is a joint segmentation and registration work, which is more complex and more difficult to implement than a sequential one (registration followed by segmentation), the running time is still acceptable due to several reasons. First, the convolution appears in the equations can be efficiently calculated by FFT. Second, note that the proposed algorithm has inner and outer loops. Most of the computation time is spent on dealing with the first outer loop of segmentation and registration. The following segmentation and registration processes take little time, as only a few number of inner loops are needed when given good initializations. The computation is done on an Intel processor (Core i5, 2GHz, 16GB RAM) with non-optimized MATLAB code. The average computation time is about $2$ minutes for a 3D MR image of size $256\times256\times15$, which is much less than manual segmentation.

\subsection{Test on synthetic images}\label{sec:TestSynthetic}

\begin{figure}[htp]
\centering
\subfigure[Image $I_m$]{\label{fig:syntheticIMa}
\begin{minipage}[b]{0.23\textwidth}
\includegraphics[width=1\textwidth]{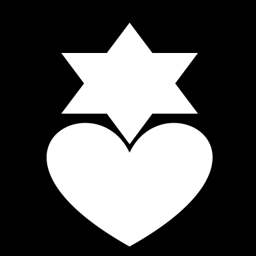}
\end{minipage}}
\subfigure[Image $I_f$]{\label{fig:syntheticIMb}
\begin{minipage}[b]{0.23\textwidth}
\includegraphics[width=1\textwidth]{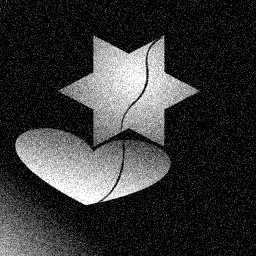}
\end{minipage}}
\subfigure[GT of $I_m$]{\label{fig:syntheticIMc}
\begin{minipage}[b]{0.23\textwidth}
\includegraphics[width=1\textwidth]{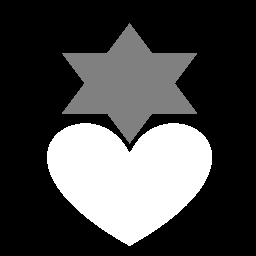}
\end{minipage}}
\subfigure[GT of $I_f$]{\label{fig:syntheticIMd}
\begin{minipage}[b]{0.23\textwidth}
\includegraphics[width=1\textwidth]{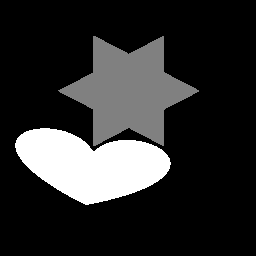}
\end{minipage}}
\caption{Synthetic images and the ground truth(GT) segmentations.}
\label{fig:syntheticIM}
\end{figure}

To show the validity of our proposed model, we first test it on a pair of synthetic images, which are shown in \cref{fig:syntheticIM} (a) and (b). The moving image $I_m$ and the fixed image $I_f$ are both composed of a star shape and a heart shape, while the fixed image is polluted by noise, intensity inhomogeneity and broken shapes. The segmentation task is to separate the two shapes from the background, individually, as shown in \cref{fig:syntheticIM} (c) and (d). This is difficult for the fixed image as the two objects have similar intensities. Before giving the segmentation result of the fixed image by our proposed model with given (a), (b) and (c), we first do some numerical experiments to show the effectiveness of the components by turning off others in the proposed model. 

The first experiment is to show the effectiveness of the segmentation component. \cref{fig:SegModel} shows the comparison of different segmentation energy terms, here ``LEM" stands for the EM algorithm combined with bias correction, i.e. local EM, ``RegEM" stands for the regularized EM algorithm, ``LRegEM" stands for the local and regularized EM algorithm \cite{Jun2013Image}, and ``LRegEMPrior" stands for the segmentation method of our proposed model, i.e. LRegEM combined with a ground truth prior provided by the cross entropy constraint. From images (c)-(l), we can clearly see that bias correction and regularization have great beneficial effects on segmentation. What's more, in the bottom-left of (g), (h), (k) and (l), one can see that the spatial regularization can greatly promote bias correction, i.e. the result of intensity correction depends on the promotion of regularizer to segmentation. However, all of these methods shown in (c)-(l) can not segment the two shapes from background individually. As shown in (m) and (n), our proposed segmentation method can well accomplish this task, showing the validity of spacial priors provided by the cross entropy. 

\begin{figure}[htp]
\centering
\subfigure[Image $I_f$]{\label{fig:SegModela}
\begin{minipage}[b]{0.23\textwidth}
\includegraphics[width=1\textwidth]{figure/starheart_fix_image1_0.01.png}
\end{minipage}}
\subfigure[Ground truth]{\label{fig:SegModelb}
\begin{minipage}[b]{0.23\textwidth}
\includegraphics[width=1\textwidth]{figure/starheart_fix_image_gt.png}
\end{minipage}}
\subfigure[K-means label]{\label{fig:SegModelc}
\begin{minipage}[b]{0.23\textwidth}
\includegraphics[width=1\textwidth]{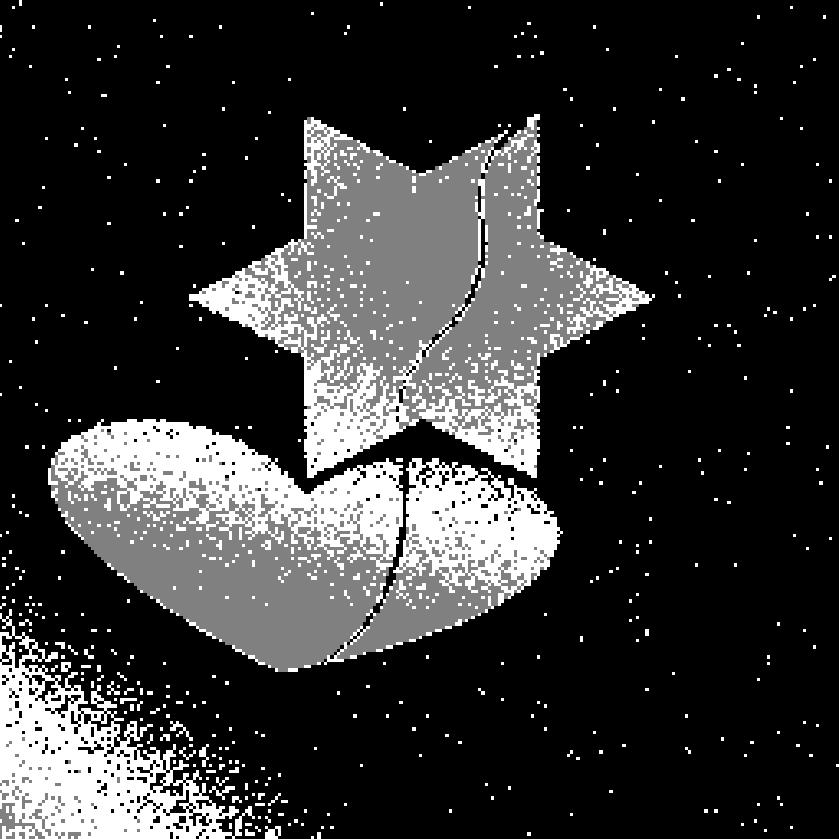}
\end{minipage}}
\subfigure[K-means contour]{\label{fig:SegModeld}
\begin{minipage}[b]{0.23\textwidth}
\includegraphics[width=1\textwidth]{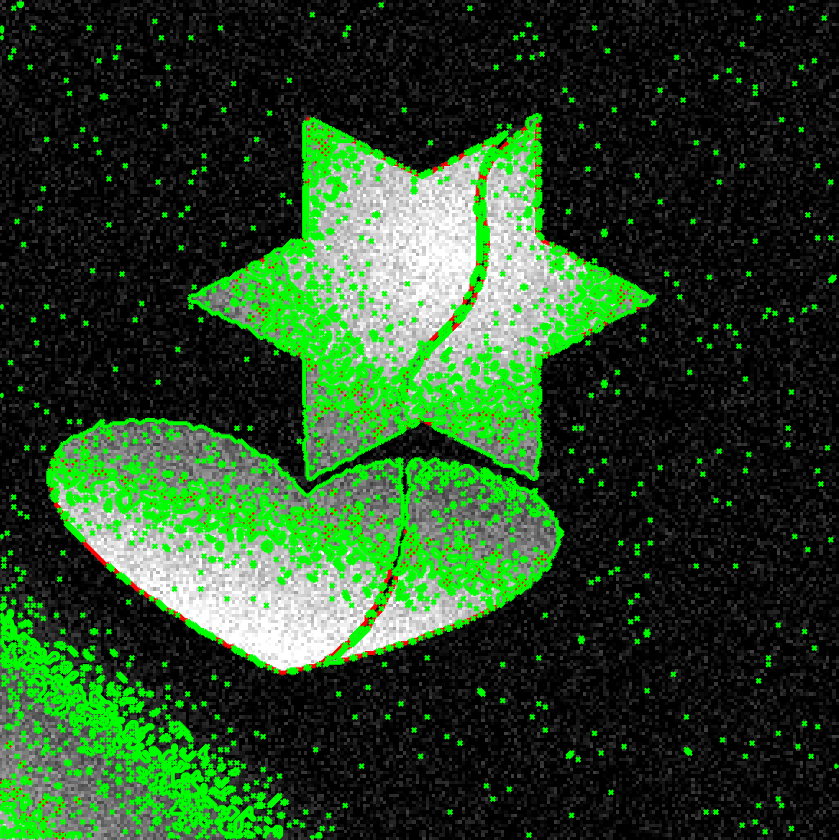}
\end{minipage}}
\subfigure[EM label]{\label{fig:SegModele}
\begin{minipage}[b]{0.23\textwidth}
\includegraphics[width=1\textwidth]{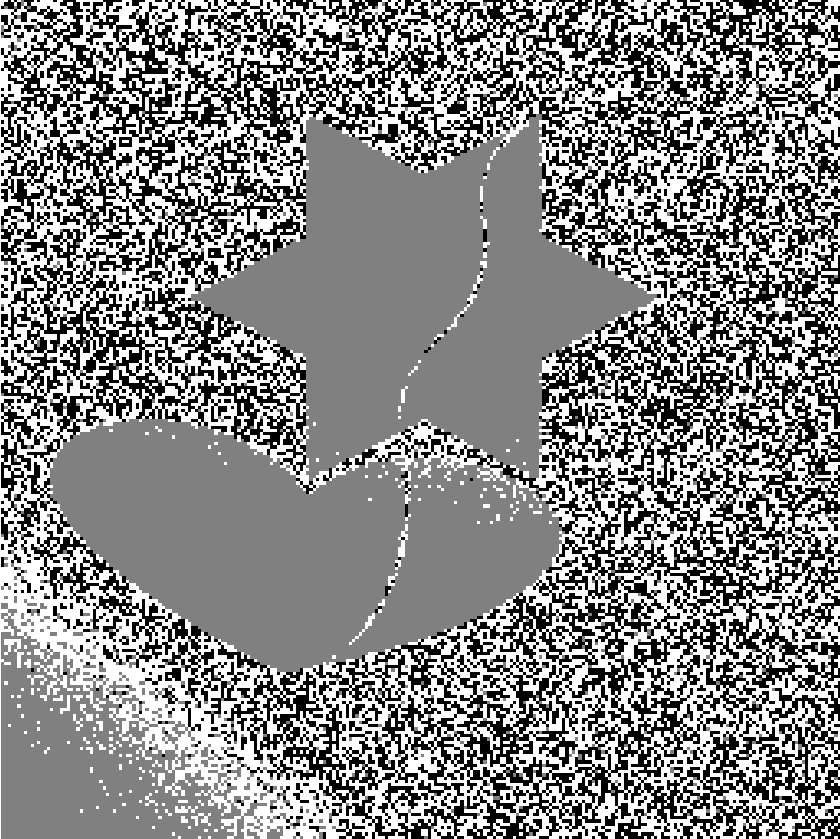}
\end{minipage}}
\subfigure[EM contour]{\label{fig:SegModelf}
\begin{minipage}[b]{0.23\textwidth}
\includegraphics[width=1\textwidth]{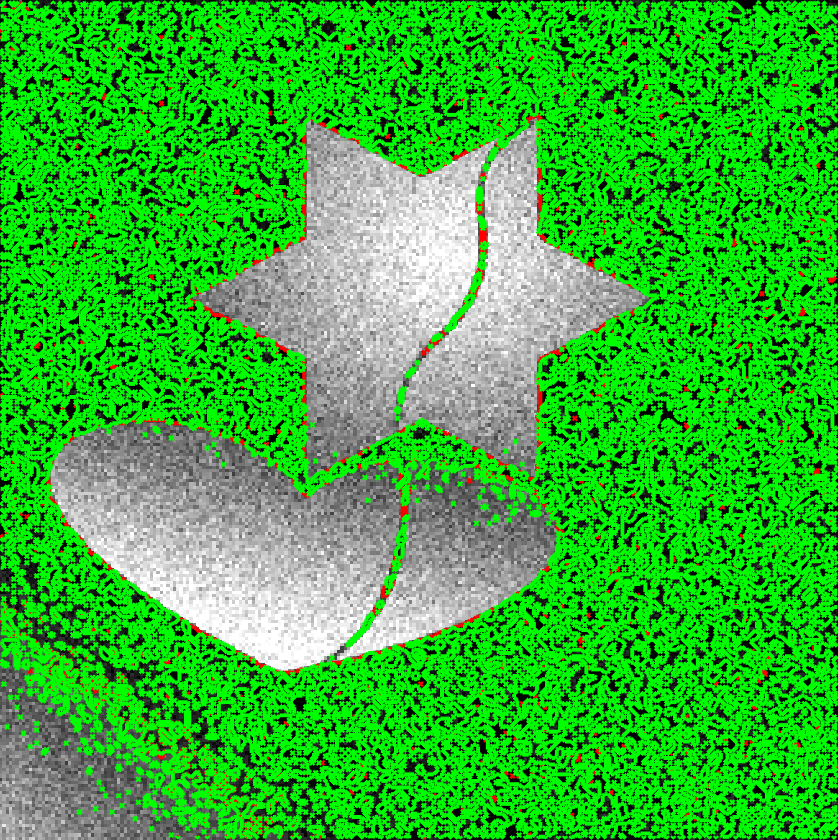}
\end{minipage}}
\subfigure[LEM label]{\label{fig:SegModelg}
\begin{minipage}[b]{0.23\textwidth}
\includegraphics[width=1\textwidth]{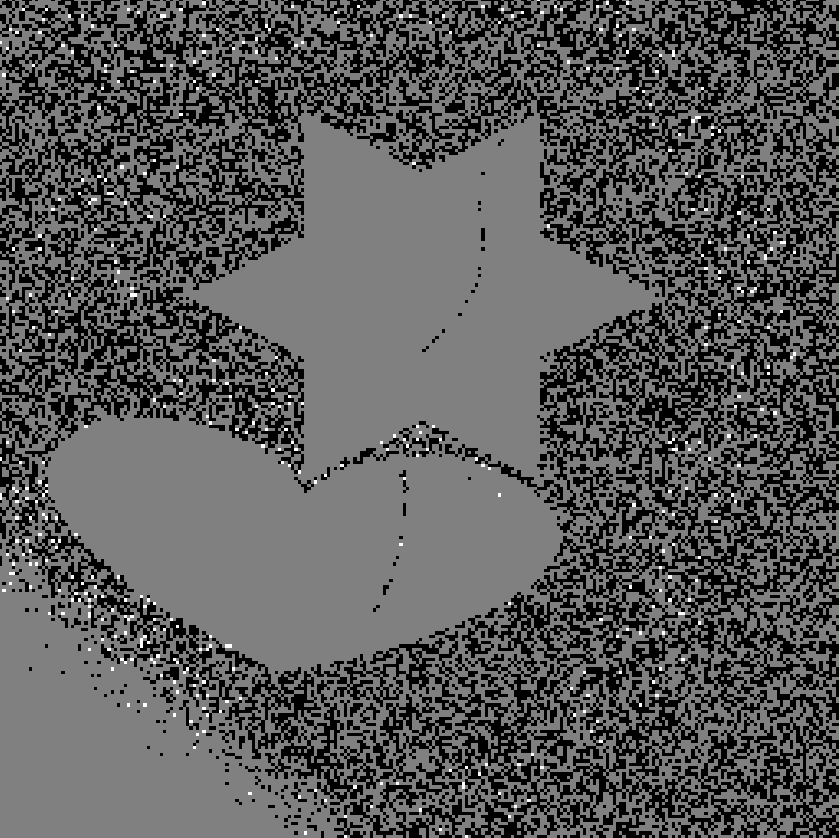}
\end{minipage}}
\subfigure[LEM contour]{\label{fig:SegModelh}
\begin{minipage}[b]{0.23\textwidth}
\includegraphics[width=1\textwidth]{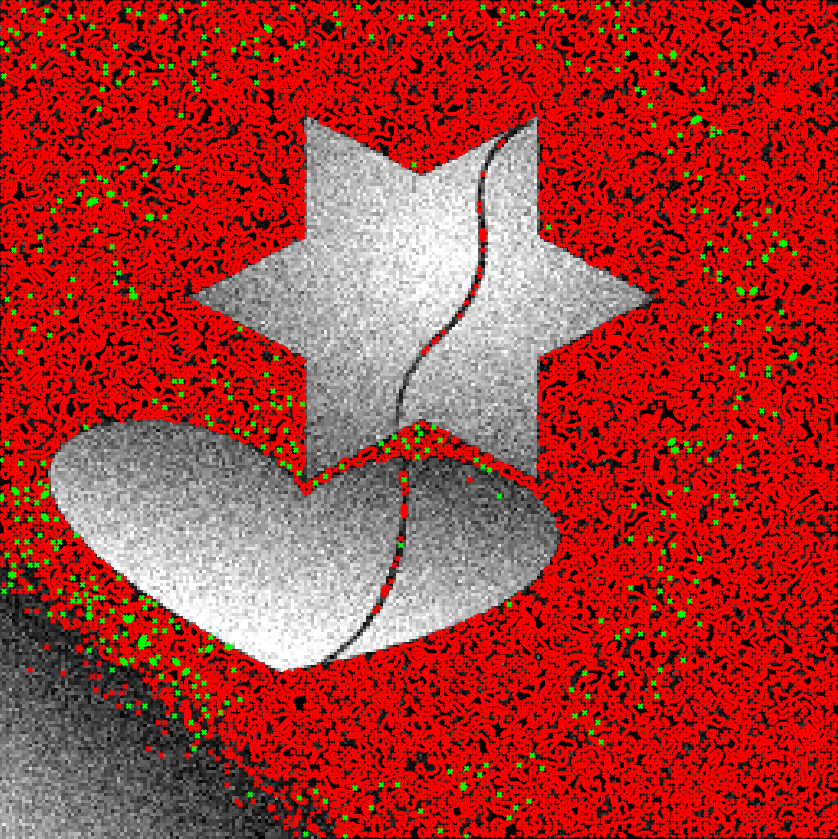}
\end{minipage}}
\subfigure[RegEM label]{\label{fig:SegModeli}
\begin{minipage}[b]{0.23\textwidth}
\includegraphics[width=1\textwidth]{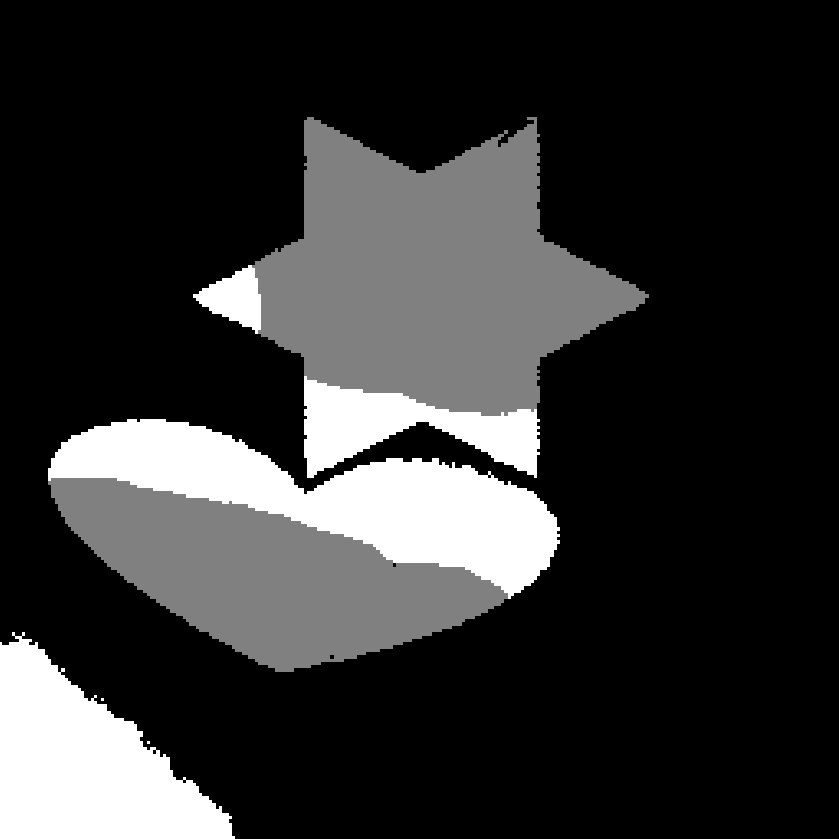}
\end{minipage}}
\subfigure[RegEM contour]{\label{fig:SegModej}
\begin{minipage}[b]{0.23\textwidth}
\includegraphics[width=1\textwidth]{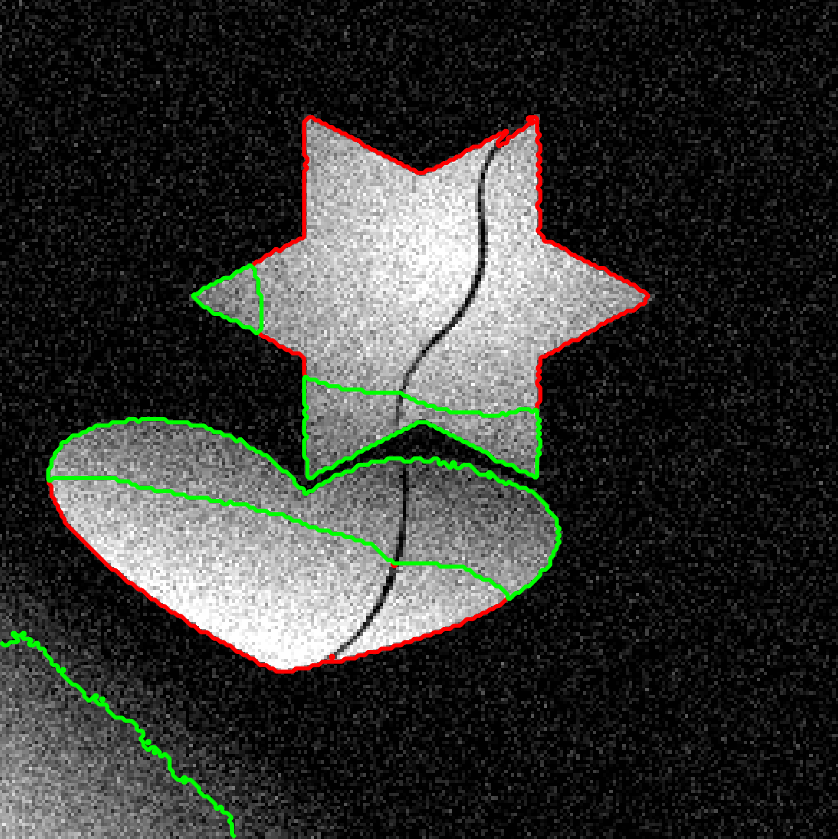}
\end{minipage}}
\subfigure[LRegEM label]{\label{fig:SegModelk}
\begin{minipage}[b]{0.23\textwidth}
\includegraphics[width=1\textwidth]{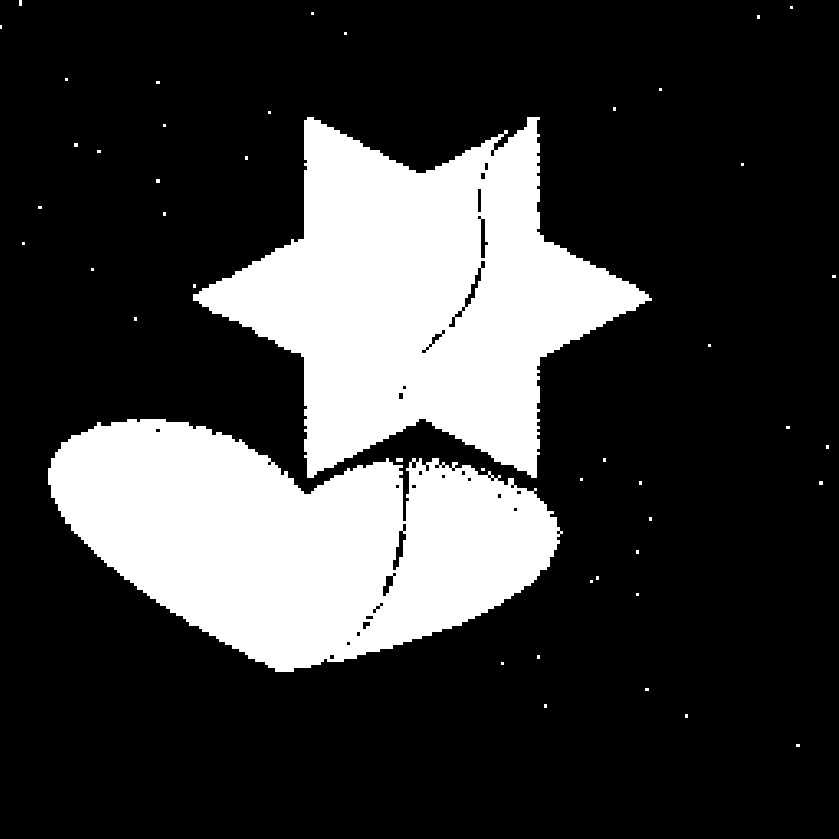}
\end{minipage}}
\subfigure[LRegEM contour]{\label{fig:SegModell}
\begin{minipage}[b]{0.23\textwidth}
\includegraphics[width=1\textwidth]{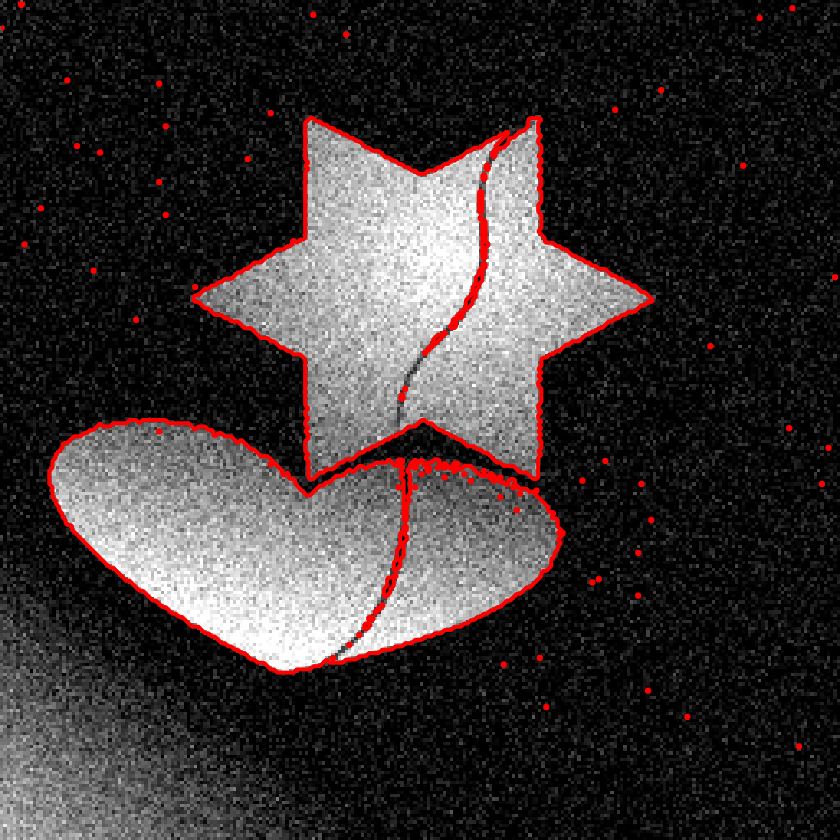}
\end{minipage}}
\subfigure[LRegEMPrior label]{\label{fig:SegModelm}
\begin{minipage}[b]{0.23\textwidth}
\includegraphics[width=1\textwidth]{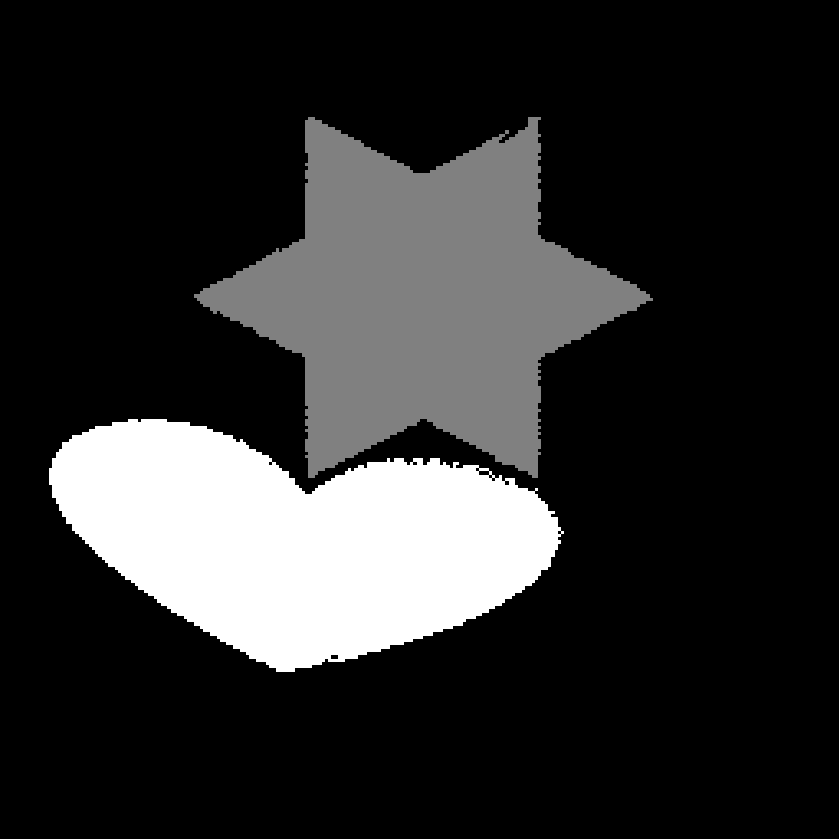}
\end{minipage}}
\subfigure[LRegEMPrior contour]{\label{fig:SegModeln}
\begin{minipage}[b]{0.23\textwidth}
\includegraphics[width=1\textwidth]{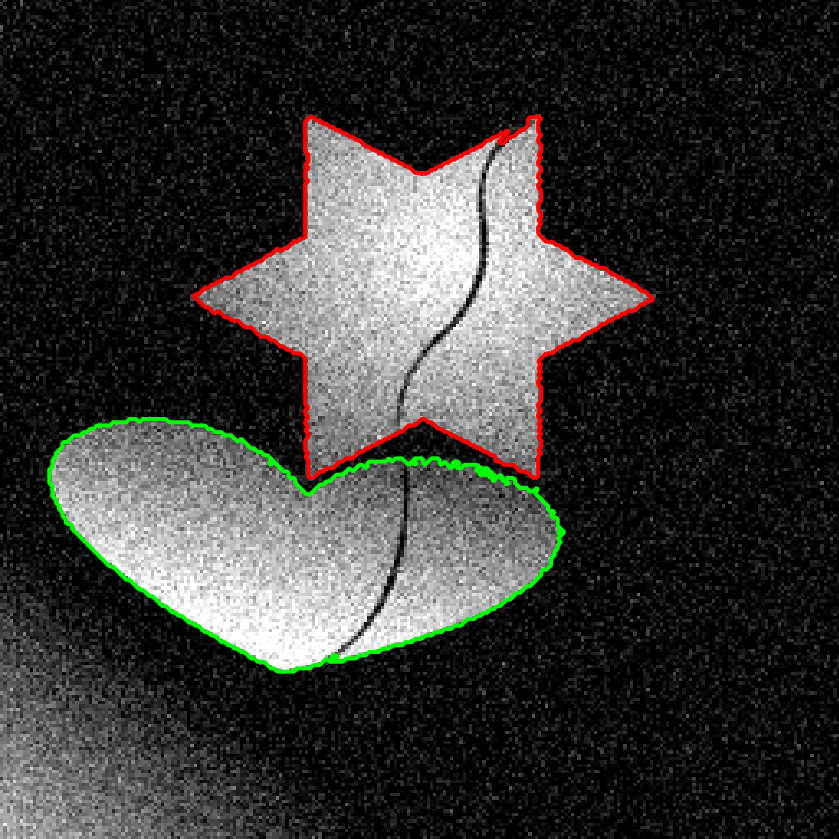}
\end{minipage}}
\subfigure[Bias field $\beta$]{\label{fig:SegModelo}
\begin{minipage}[b]{0.23\textwidth}
\includegraphics[width=1\textwidth]{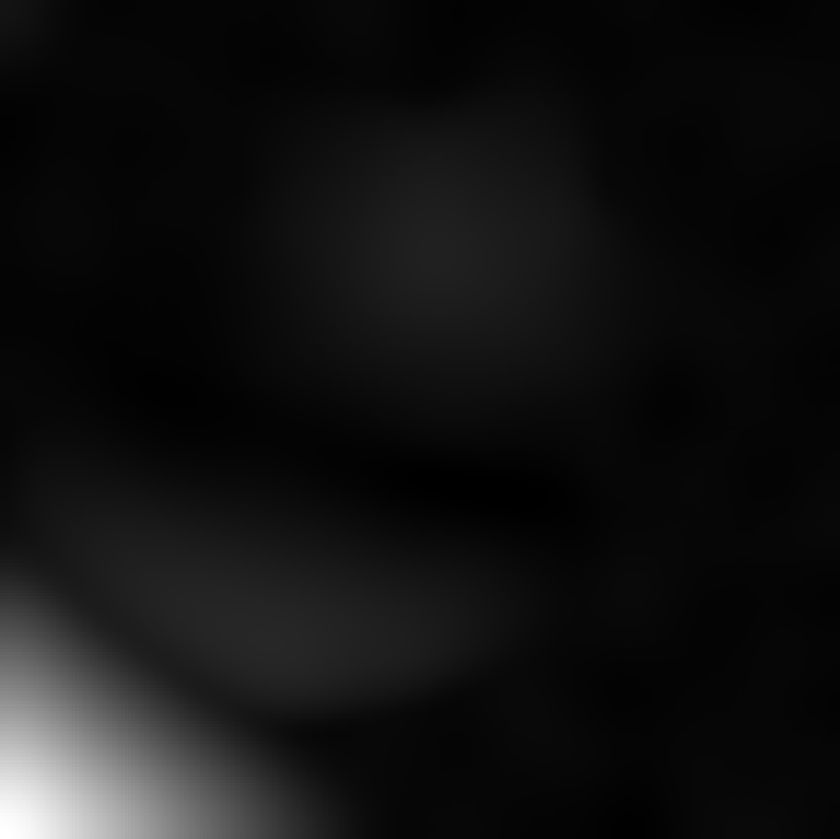}
\end{minipage}}
\subfigure[Revised image $I_f/\beta$]{\label{fig:SegModelp}
\begin{minipage}[b]{0.23\textwidth}
\includegraphics[width=1\textwidth]{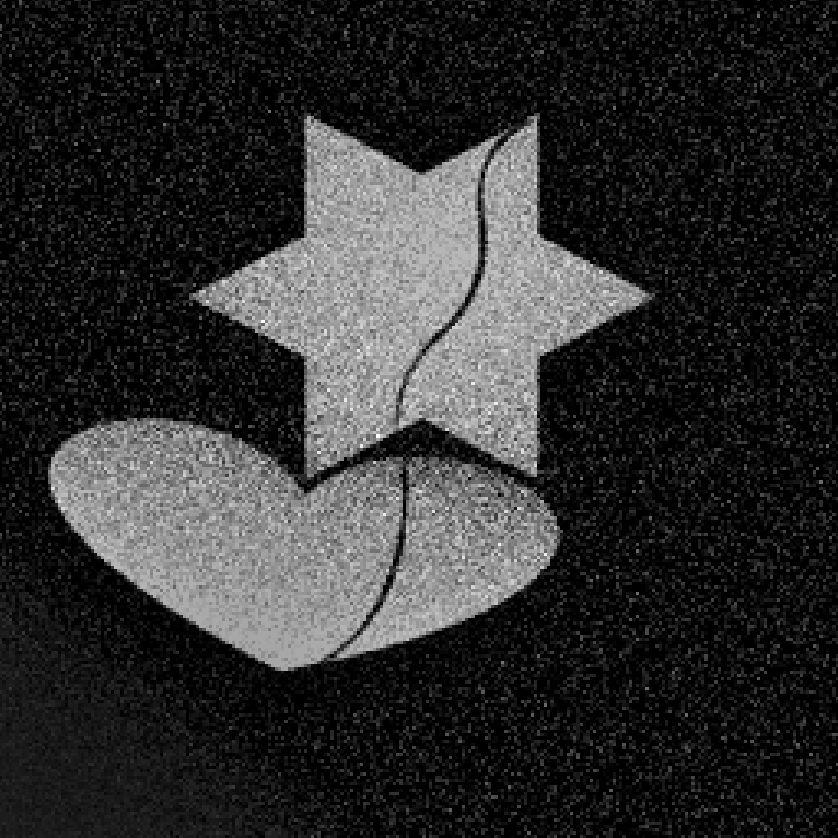}
\end{minipage}}
\caption{Segmentation model test. (a) the fixed image $I_f$, (b) the ground truth segmentation of $I_f$, (c)-(d) K-means, (e)-(f) EM algorithm, (g)-(h) local EM algorithm, i.e. EM combined with bias correction, (i)-(j) regularized EM algorithm, (k)-(l) local and regularized EM algorithm \cite{Jun2013Image}, (m)-(n) proposed method, i.e. local and regularized EM algorithm combined with a ground truth prior provided by the cross entropy constraint, (o)-(p) estimated bias field and revised image $I_f/\beta$ by the proposed segmentation model.} 
\label{fig:SegModel}
\end{figure}

The second numerical experiment is to test the performance of the registration component. \cref{fig:eta} shows the registration results with different regularization parameters $\eta$. As can be seen from this figure, the smaller $\eta$ the more irregular displacement field, and the bigger $\eta$ the smaller displacement, which tends to keep the grid on its original location. As mentioned in \cref{sec:optimization}, the registration is implemented from coarse level to fine level. In \cref{fig:RegProcess}, we show the registration results on odd levels, which indicates that the registration algorithm is stable and progressive. 

\begin{figure}[!htp]
\centering
\subfigure{
\begin{minipage}[b]{0.23\textwidth}
\includegraphics[width=1\textwidth]{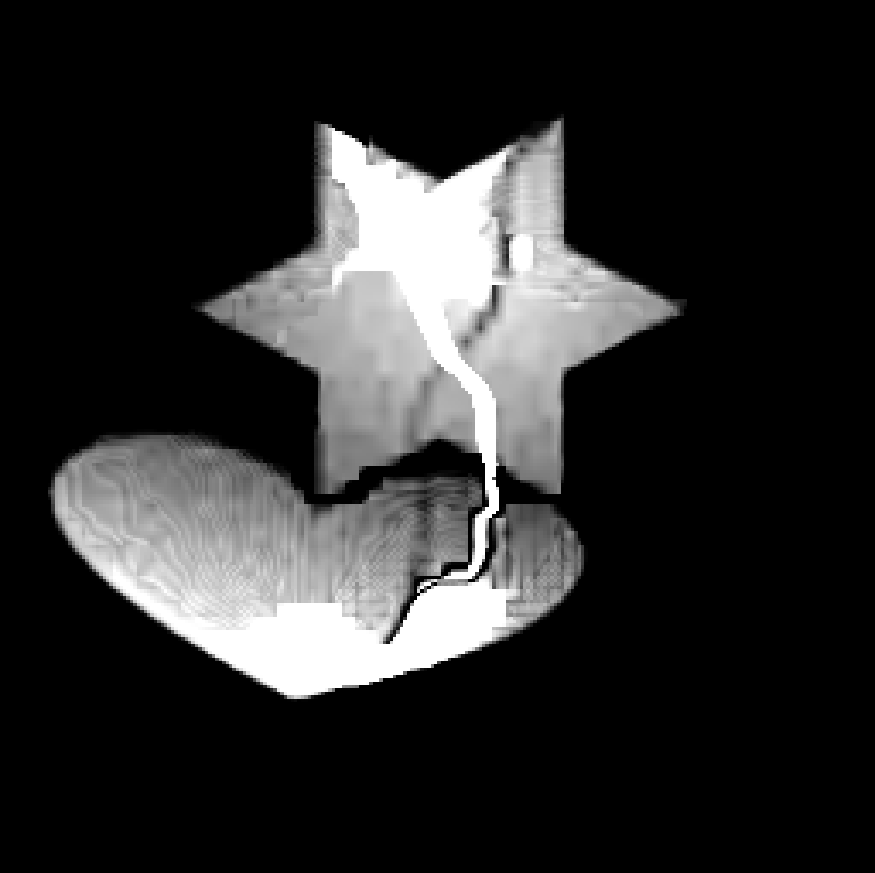}
\end{minipage}}
\subfigure{
\begin{minipage}[b]{0.23\textwidth}
\includegraphics[width=1\textwidth]{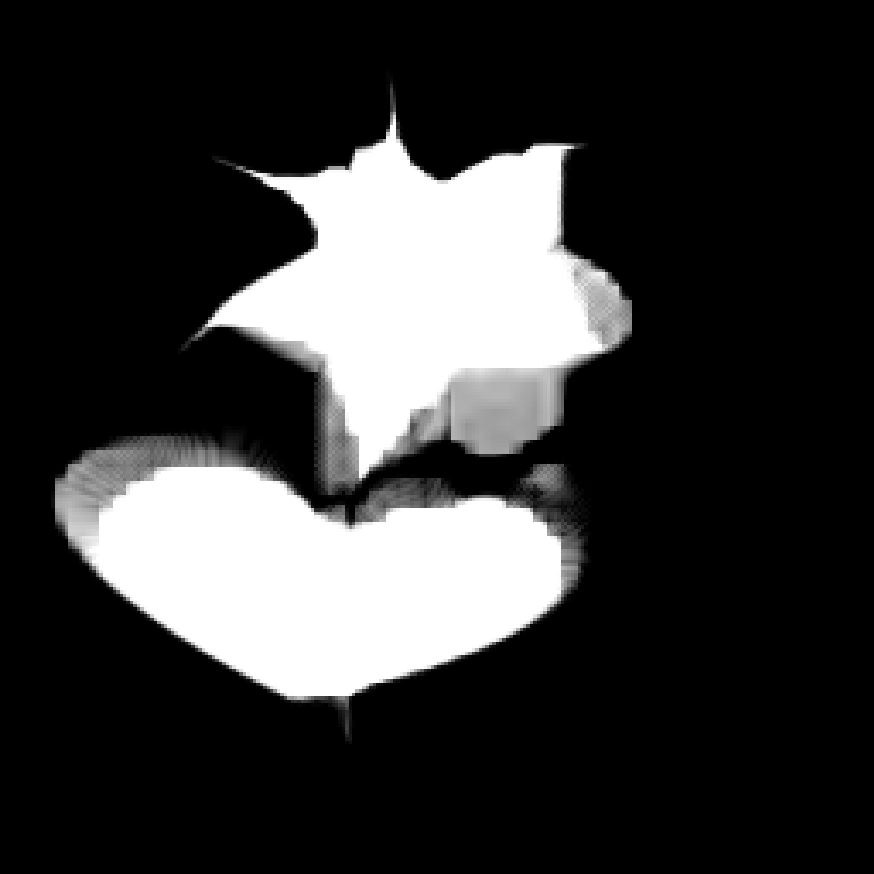}
\end{minipage}}
\subfigure{
\begin{minipage}[b]{0.23\textwidth}
\includegraphics[width=1\textwidth]{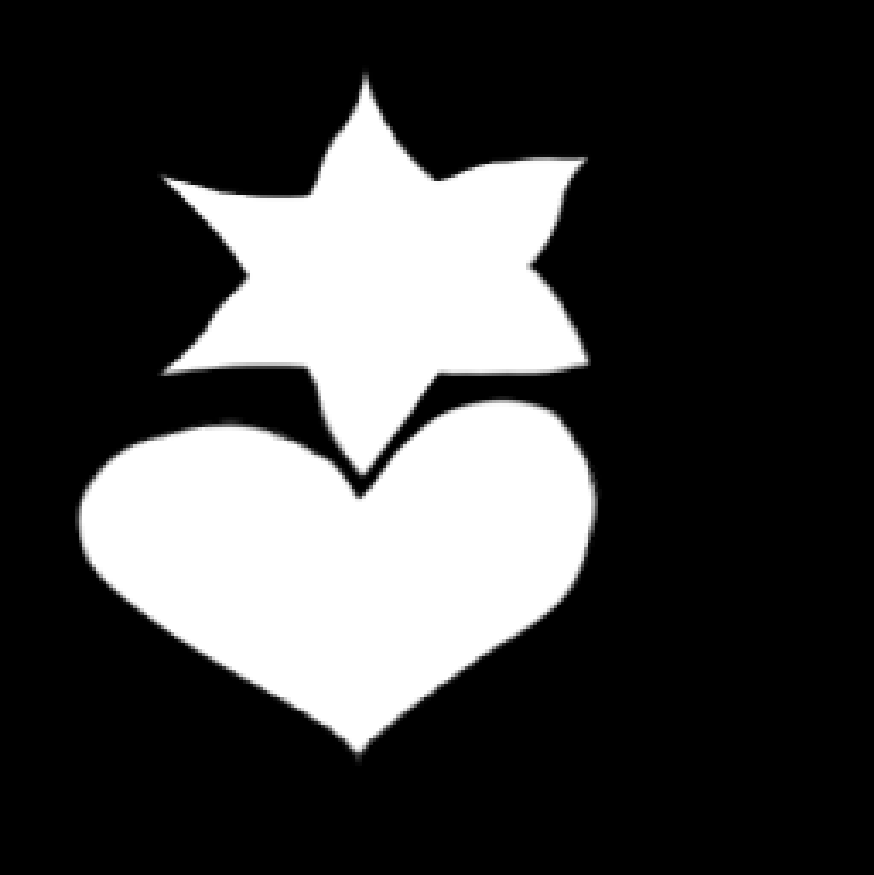}
\end{minipage}}\\
\subfigure{
\begin{minipage}[b]{0.23\textwidth}
\includegraphics[width=1\textwidth]{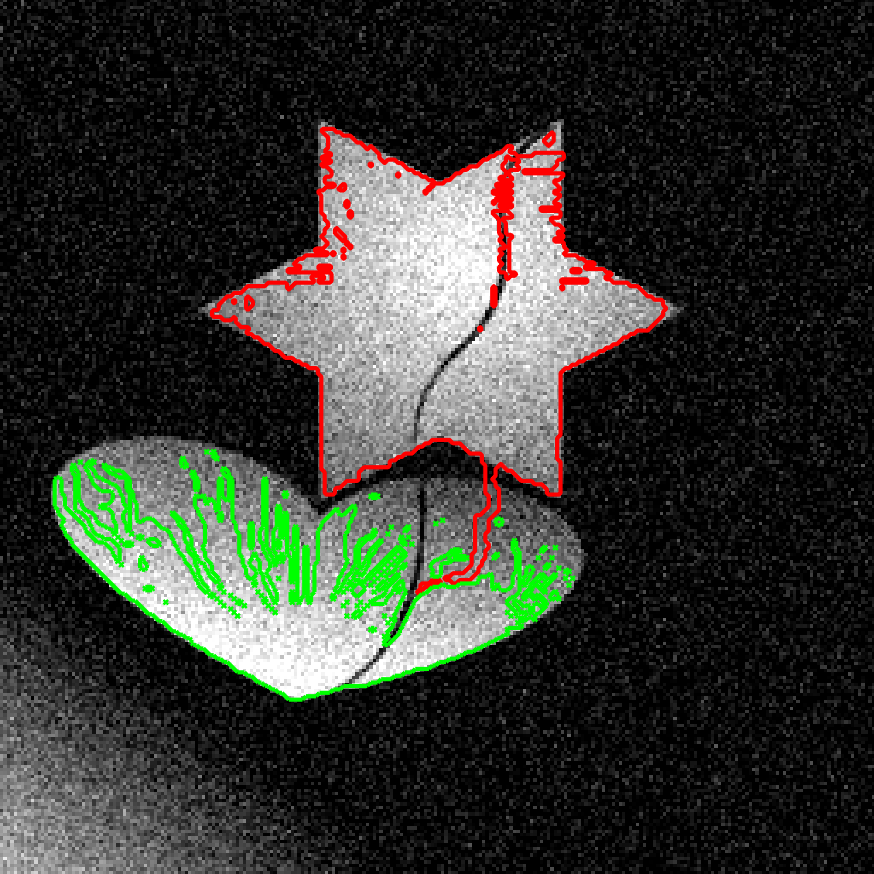}
\end{minipage}}
\subfigure{
\begin{minipage}[b]{0.23\textwidth}
\includegraphics[width=1\textwidth]{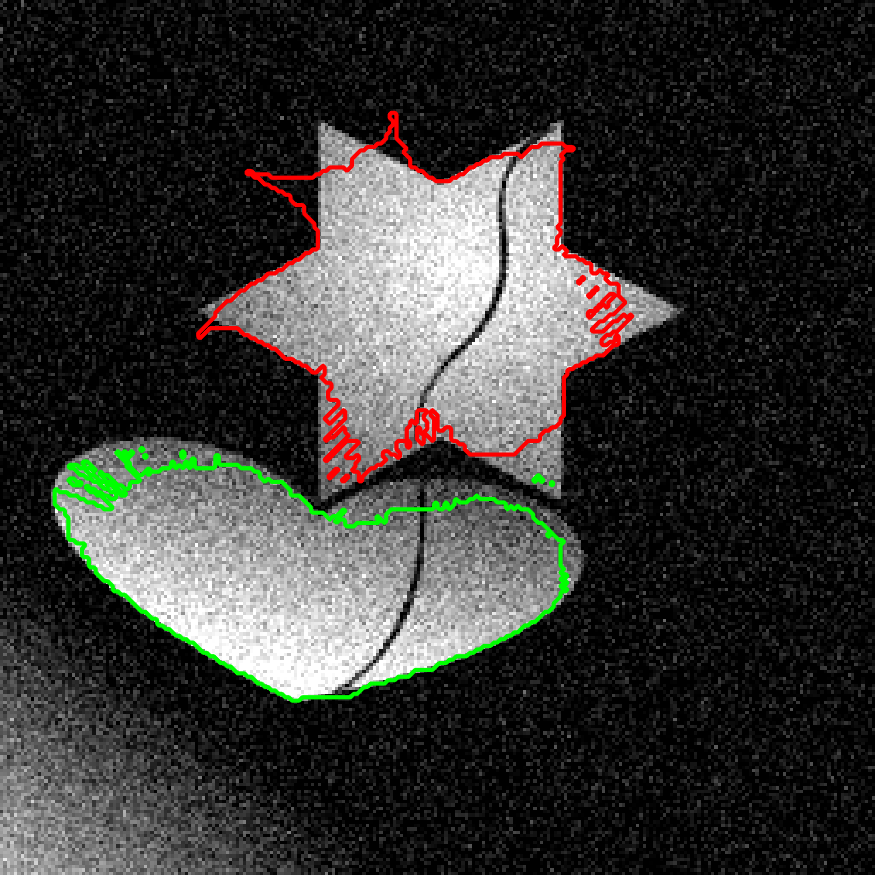}
\end{minipage}}
\subfigure{
\begin{minipage}[b]{0.23\textwidth}
\includegraphics[width=1\textwidth]{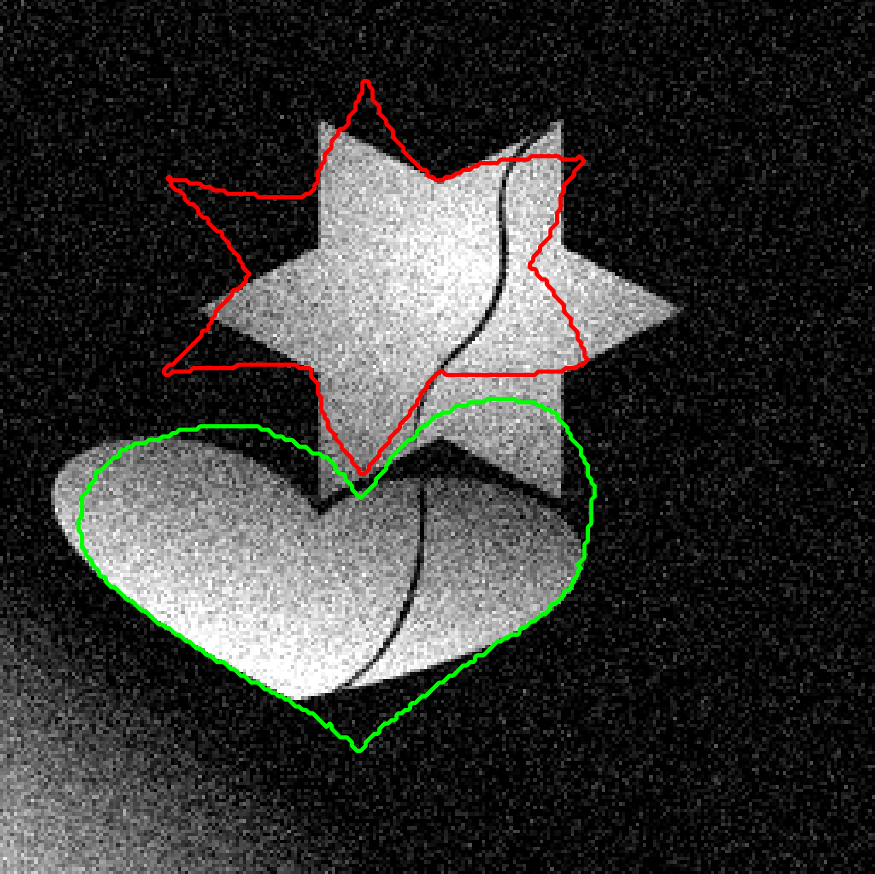}
\end{minipage}}\\
\subfigure{
\begin{minipage}[b]{0.23\textwidth}
\includegraphics[width=1\textwidth]{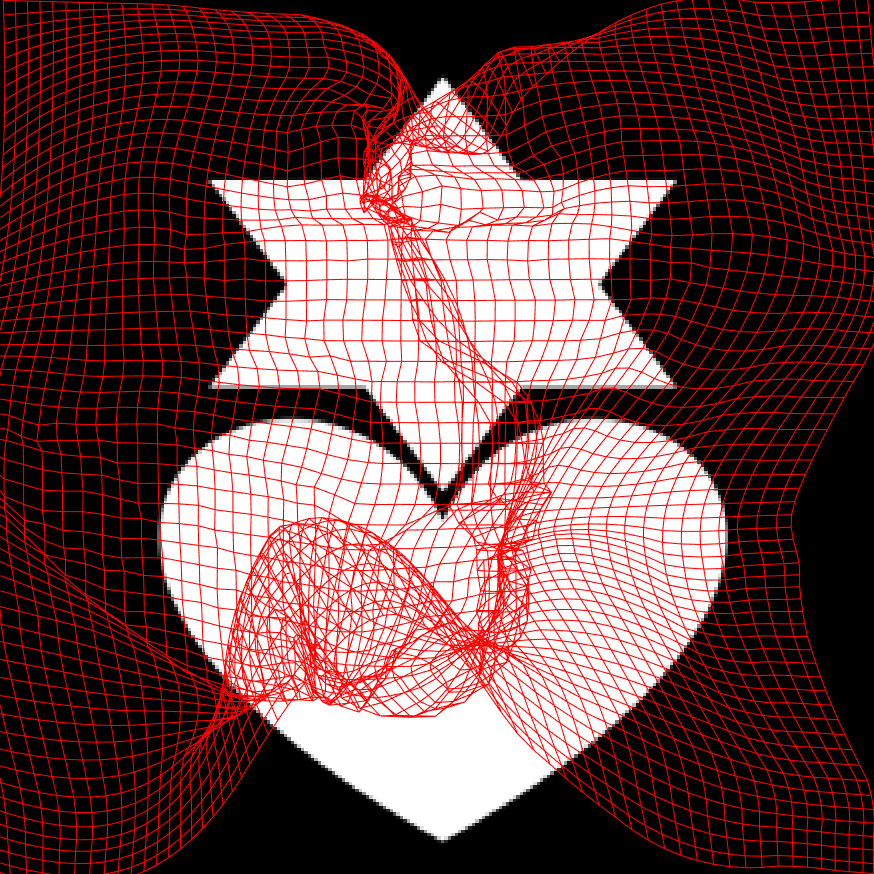}
\end{minipage}}
\subfigure{
\begin{minipage}[b]{0.23\textwidth}
\includegraphics[width=1\textwidth]{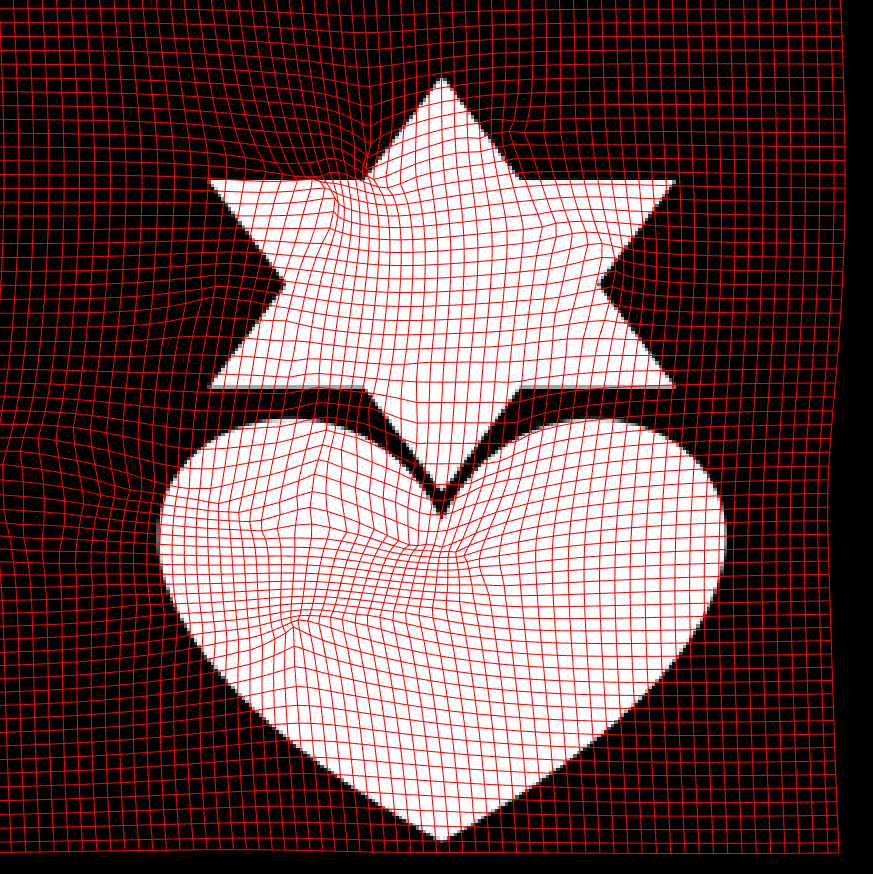}
\end{minipage}}
\subfigure{
\begin{minipage}[b]{0.23\textwidth}
\includegraphics[width=1\textwidth]{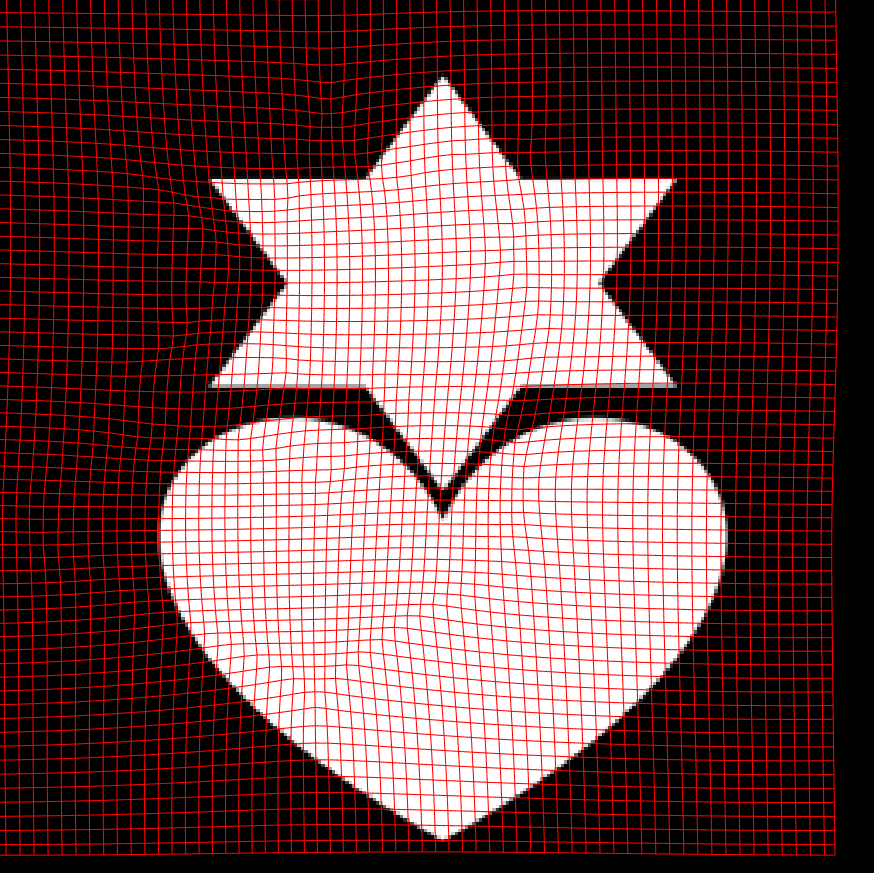}
\end{minipage}}\\
\subfigure{\label{fig:etaa}
\begin{minipage}[b]{0.23\textwidth}
\includegraphics[width=1\textwidth]{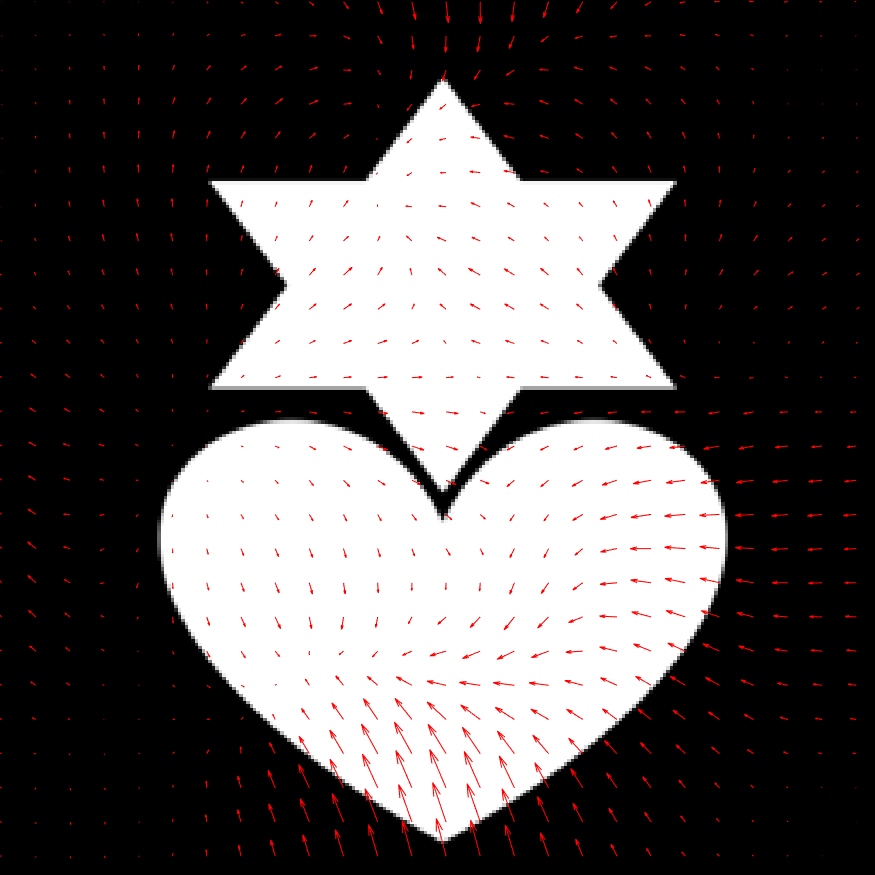}
\caption*{(a) $\eta=0$}
\end{minipage}}
\subfigure{\label{fig:etab}
\begin{minipage}[b]{0.23\textwidth}
\includegraphics[width=1\textwidth]{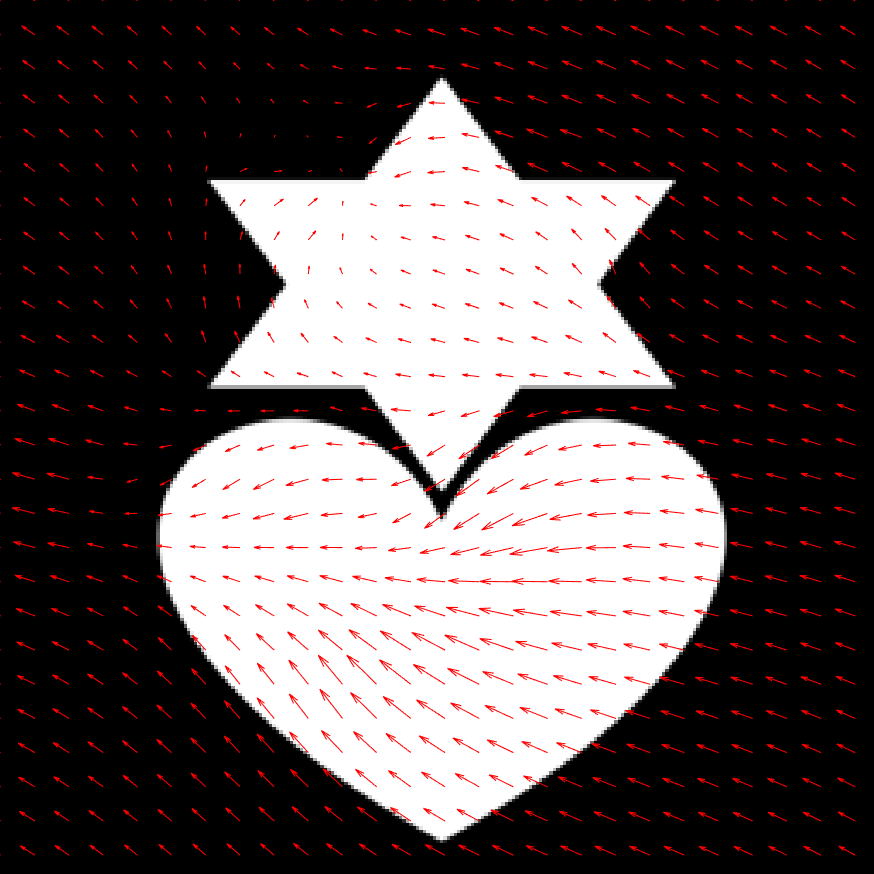}
\caption*{(b) $\eta=0.1$}
\end{minipage}}
\subfigure{\label{fig:etac}
\begin{minipage}[b]{0.23\textwidth}
\includegraphics[width=1\textwidth]{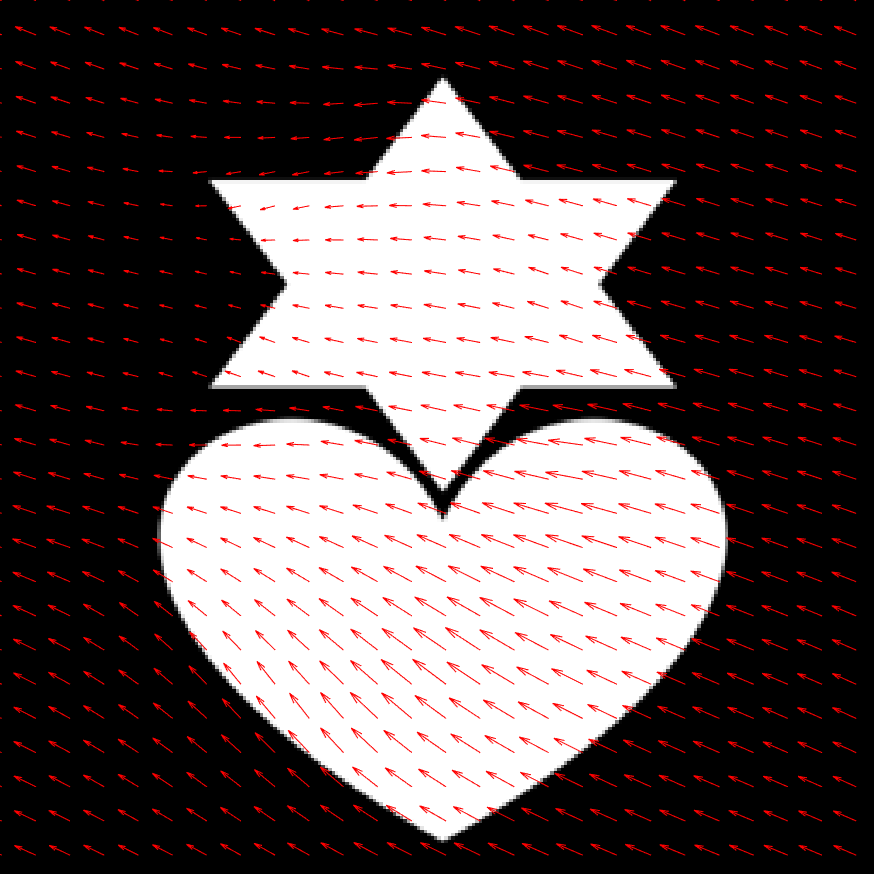}
\caption*{(c) $\eta=0.5$}
\end{minipage}}
\caption{Test of different regularization parameters $\eta$. First line: deformed image of $I_m$; second line: deformed contour; third line: deformed mesh grid; fourth line: displacement field.} 
\label{fig:eta}
\end{figure}

\begin{figure}[htp]
\centering
\subfigure{
\begin{minipage}[b]{0.09\textwidth}
\includegraphics[width=1\textwidth]{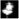}
\end{minipage}}
\subfigure{
\begin{minipage}[b]{0.13\textwidth}
\includegraphics[width=1\textwidth]{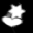}
\end{minipage}}
\subfigure{
\begin{minipage}[b]{0.17\textwidth}
\includegraphics[width=1\textwidth]{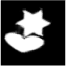}
\end{minipage}}
\subfigure{
\begin{minipage}[b]{0.21\textwidth}
\includegraphics[width=1\textwidth]{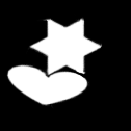}
\end{minipage}}
\subfigure{
\begin{minipage}[b]{0.25\textwidth}
\includegraphics[width=1\textwidth]{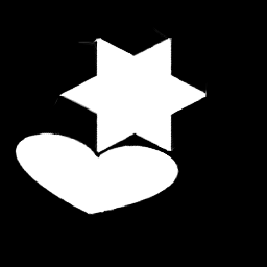}
\end{minipage}}\\
\subfigure{
\begin{minipage}[b]{0.09\textwidth}
\includegraphics[width=1\textwidth]{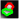}
\end{minipage}}
\subfigure{
\begin{minipage}[b]{0.13\textwidth}
\includegraphics[width=1\textwidth]{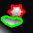}
\end{minipage}}
\subfigure{
\begin{minipage}[b]{0.17\textwidth}
\includegraphics[width=1\textwidth]{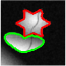}
\end{minipage}}
\subfigure{
\begin{minipage}[b]{0.21\textwidth}
\includegraphics[width=1\textwidth]{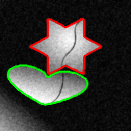}
\end{minipage}}
\subfigure{
\begin{minipage}[b]{0.25\textwidth}
\includegraphics[width=1\textwidth]{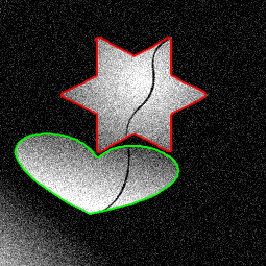}
\end{minipage}}\\
\subfigure{
\begin{minipage}[b]{0.09\textwidth}
\includegraphics[width=1\textwidth]{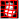}
\end{minipage}}
\subfigure{
\begin{minipage}[b]{0.13\textwidth}
\includegraphics[width=1\textwidth]{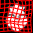}
\end{minipage}}
\subfigure{
\begin{minipage}[b]{0.17\textwidth}
\includegraphics[width=1\textwidth]{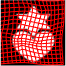}
\end{minipage}}
\subfigure{
\begin{minipage}[b]{0.21\textwidth}
\includegraphics[width=1\textwidth]{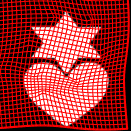}
\end{minipage}}
\subfigure{
\begin{minipage}[b]{0.25\textwidth}
\includegraphics[width=1\textwidth]{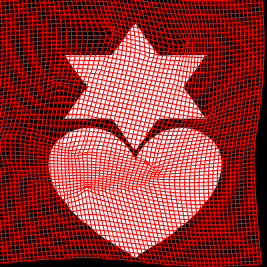}
\end{minipage}}\\
\subfigure{
\begin{minipage}[b]{0.09\textwidth}
\includegraphics[width=1\textwidth]{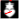}
\caption*{Level 9}
\end{minipage}}
\subfigure{
\begin{minipage}[b]{0.13\textwidth}
\includegraphics[width=1\textwidth]{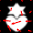}
\caption*{Level 7}
\end{minipage}}
\subfigure{
\begin{minipage}[b]{0.17\textwidth}
\includegraphics[width=1\textwidth]{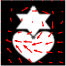}
\caption*{Level 5}
\end{minipage}}
\subfigure{
\begin{minipage}[b]{0.21\textwidth}
\includegraphics[width=1\textwidth]{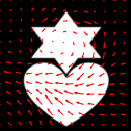}
\caption*{Level 3}
\end{minipage}}
\subfigure{
\begin{minipage}[b]{0.25\textwidth}
\includegraphics[width=1\textwidth]{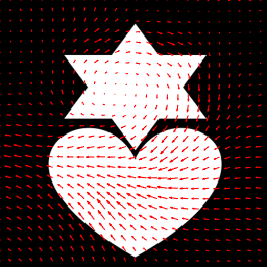}
\caption*{Level 1}
\end{minipage}}
\caption{Registration results on odd levels. First line: deformed image of $I_m$; second line: deformed contour; third line: deformed mesh grid; fourth line: displacement field.}
\label{fig:RegProcess}
\end{figure}

The third experiment is to show the bridge function, cross entropy, between segmentation and registration. \cref{fig:XiParameter} shows the results with different parameters $\xi$. We can see that when $\xi$ is small, the segmentation result is less similar to the registered template. However, when $\xi$ is too large, the strong interaction would force the segmentation result to be too close to the registered template and end up with missing the true boundaries.

\begin{figure}[htp]
\centering
\subfigure{
\begin{minipage}[b]{0.23\textwidth}
\includegraphics[width=1\textwidth]{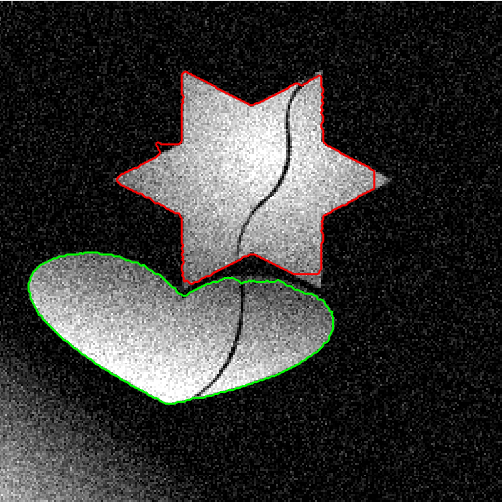}
\end{minipage}}
\subfigure{
\begin{minipage}[b]{0.23\textwidth}
\includegraphics[width=1\textwidth]{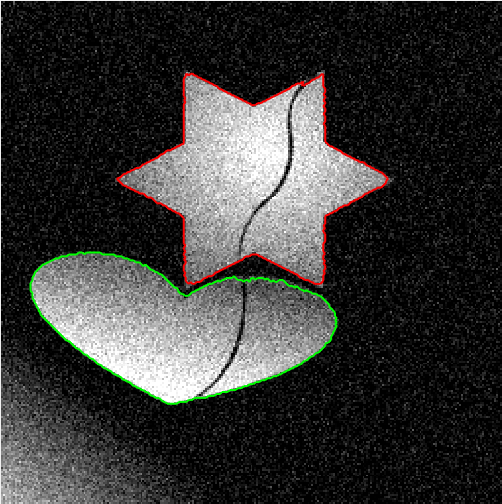}
\end{minipage}}
\subfigure{
\begin{minipage}[b]{0.23\textwidth}
\includegraphics[width=1\textwidth]{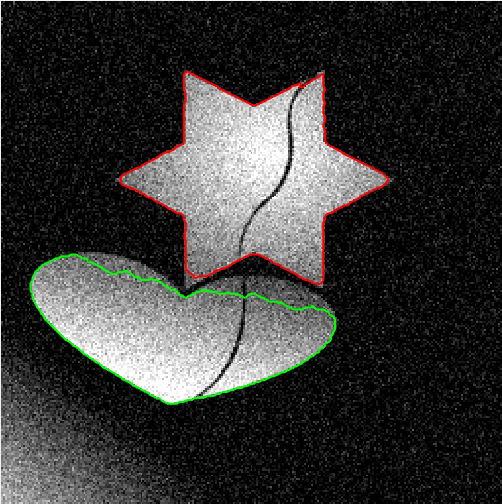}
\end{minipage}}\\
\subfigure{
\begin{minipage}[b]{0.23\textwidth}
\includegraphics[width=1\textwidth]{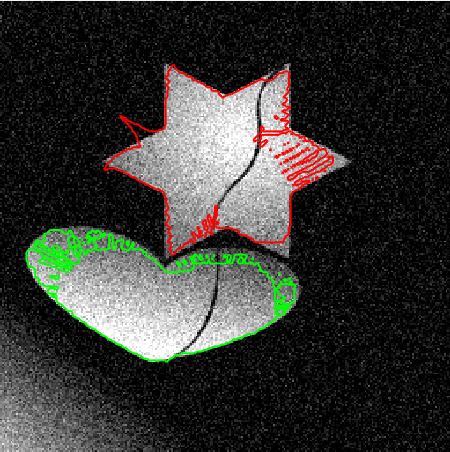}
\caption*{(a) $\xi=0.0001$}
\end{minipage}}
\subfigure{
\begin{minipage}[b]{0.23\textwidth}
\includegraphics[width=1\textwidth]{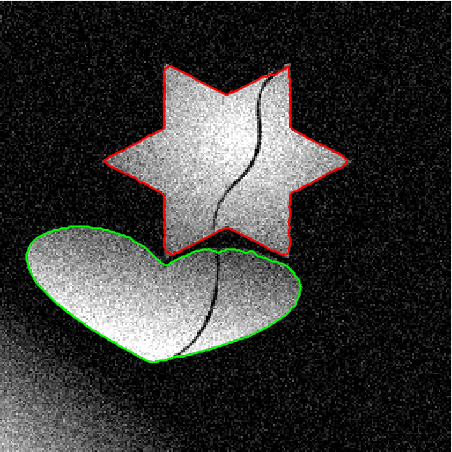}
\caption*{(b) $\xi=0.03$}
\end{minipage}}
\subfigure{
\begin{minipage}[b]{0.23\textwidth}
\includegraphics[width=1\textwidth]{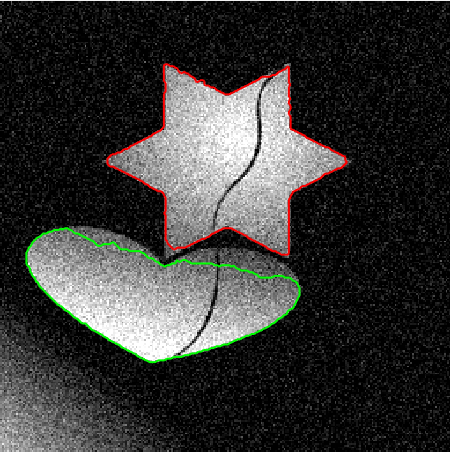}
\caption*{(c) $\xi=3$}
\end{minipage}}
\caption{Test of different cross entropy parameters $\xi$. First line: the contour of segmentation map $\bm u$; second line: the contour of deformed atlas by registration.}
\label{fig:XiParameter}
\end{figure}

The fourth experiment is to illustrate the superiority of the proposed joint model compared with segmentation only or registration only method. In \cref{fig:StarHeart} (d)-(e), we can see that segmentation only method can not separate the two shapes individually as they have similar intensities. \cref{fig:StarHeart} (f)-(g) show that registration only method can not well match the shapes of the objects. However, our proposed joint model can well accomplish this segmentation task, as shown in (h)-(k). \cref{fig:StarHeart} (l)-(q) display the bias corrected image, deformed image, deformed mesh, displacement field and the decay of the energy function of the proposed joint model.

\begin{figure}[htp]
\centering
\subfigure[Image $I_f$]{
\begin{minipage}[b]{0.21\textwidth}
\includegraphics[width=1\textwidth]{figure/starheart_fix_image1_0.01.png}
\end{minipage}}
\subfigure[Image $I_m$]{
\begin{minipage}[b]{0.21\textwidth}
\includegraphics[width=1\textwidth]{figure/starheart_mov_image.png}
\end{minipage}}
\subfigure[GT of $I_m$]{
\begin{minipage}[b]{0.21\textwidth}
\includegraphics[width=1\textwidth]{figure/starheart_mov_image_gt.png}
\end{minipage}}\\
\subfigure[$\min\mathcal{E}_{\mathrm{Seg}}$]{
\begin{minipage}[b]{0.21\textwidth}
\includegraphics[width=1\textwidth]{figure/starheart_LRegEM_lambda=0.2_label.png}
\end{minipage}}
\subfigure[$\min\mathcal{E}_{\mathrm{Seg}}$]{
\begin{minipage}[b]{0.21\textwidth}
\includegraphics[width=1\textwidth]{figure/starheart_LRegEM_lambda=0.2_edge.png}
\end{minipage}}
\subfigure[$\min\mathcal{E}_{\mathrm{Reg}}$]{
\begin{minipage}[b]{0.21\textwidth}
\includegraphics[width=1\textwidth]{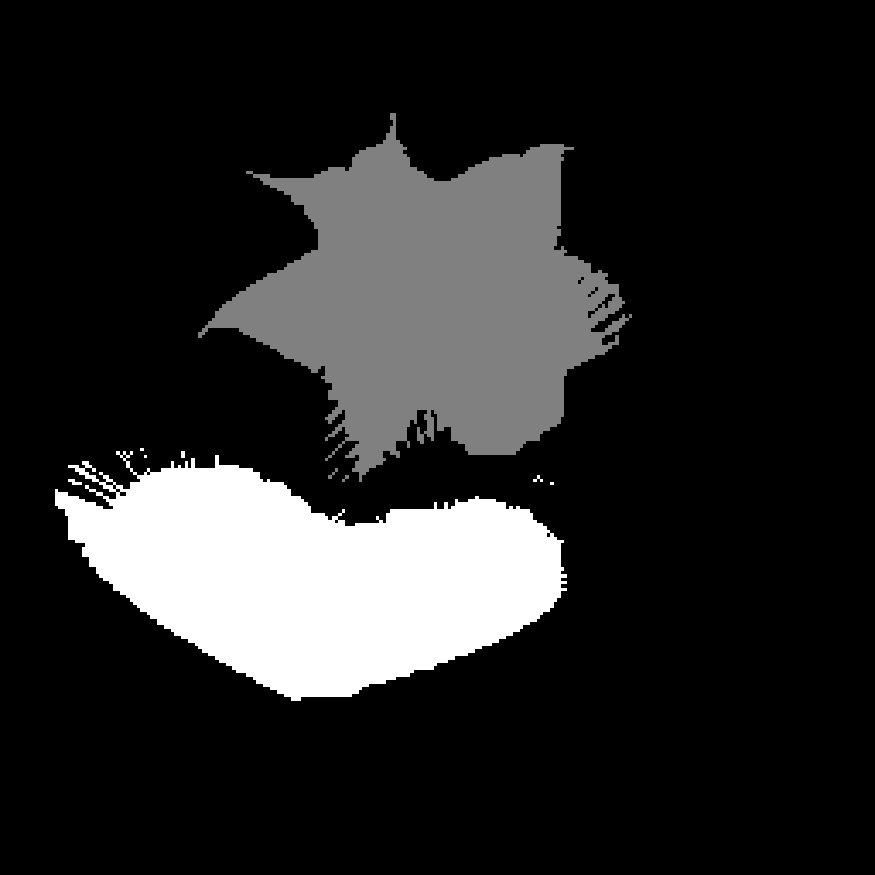}
\end{minipage}}
\subfigure[$\min\mathcal{E}_{\mathrm{Reg}}$]{
\begin{minipage}[b]{0.21\textwidth}
\includegraphics[width=1\textwidth]{figure/starheart_reg_D1L2=0.1_DeformedEdge.png}
\end{minipage}}
\subfigure[Segmentation label by $\min\mathcal{E}$]{
\begin{minipage}[b]{0.21\textwidth}
\includegraphics[width=1\textwidth]{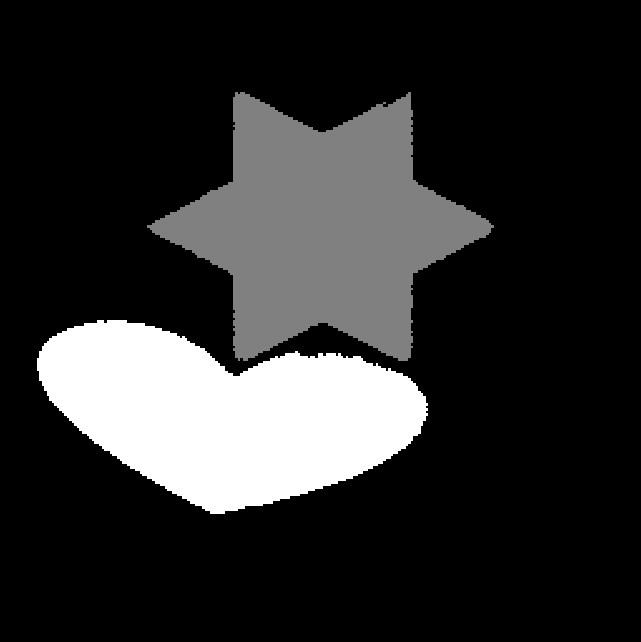}
\end{minipage}}
\subfigure[Segmentation contour by $\min\mathcal{E}$]{
\begin{minipage}[b]{0.21\textwidth}
\includegraphics[width=1\textwidth]{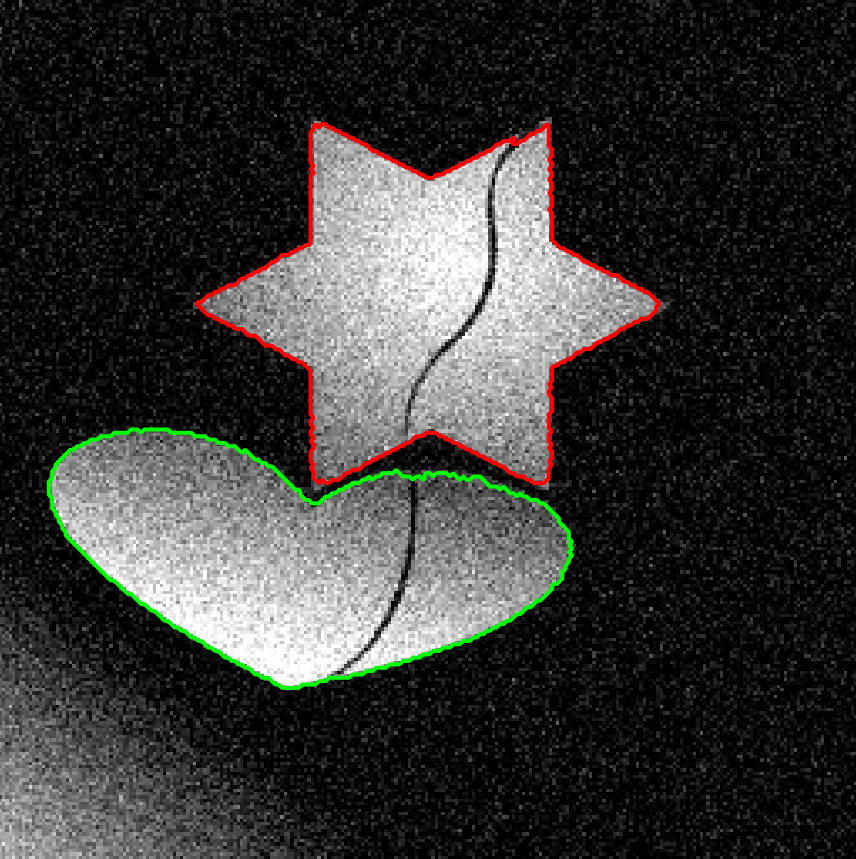}
\end{minipage}}
\subfigure[Deformed label by $\min\mathcal{E}$]{
\begin{minipage}[b]{0.21\textwidth}
\includegraphics[width=1\textwidth]{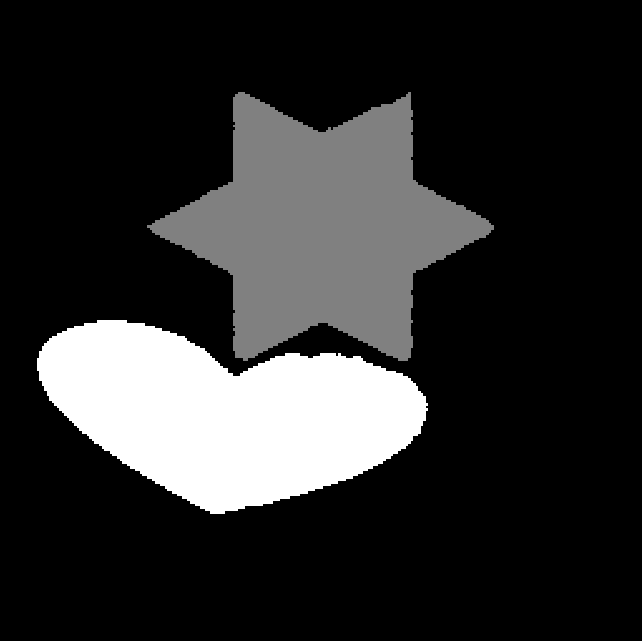}
\end{minipage}}
\subfigure[Deformed contour by $\min\mathcal{E}$]{
\begin{minipage}[b]{0.21\textwidth}
\includegraphics[width=1\textwidth]{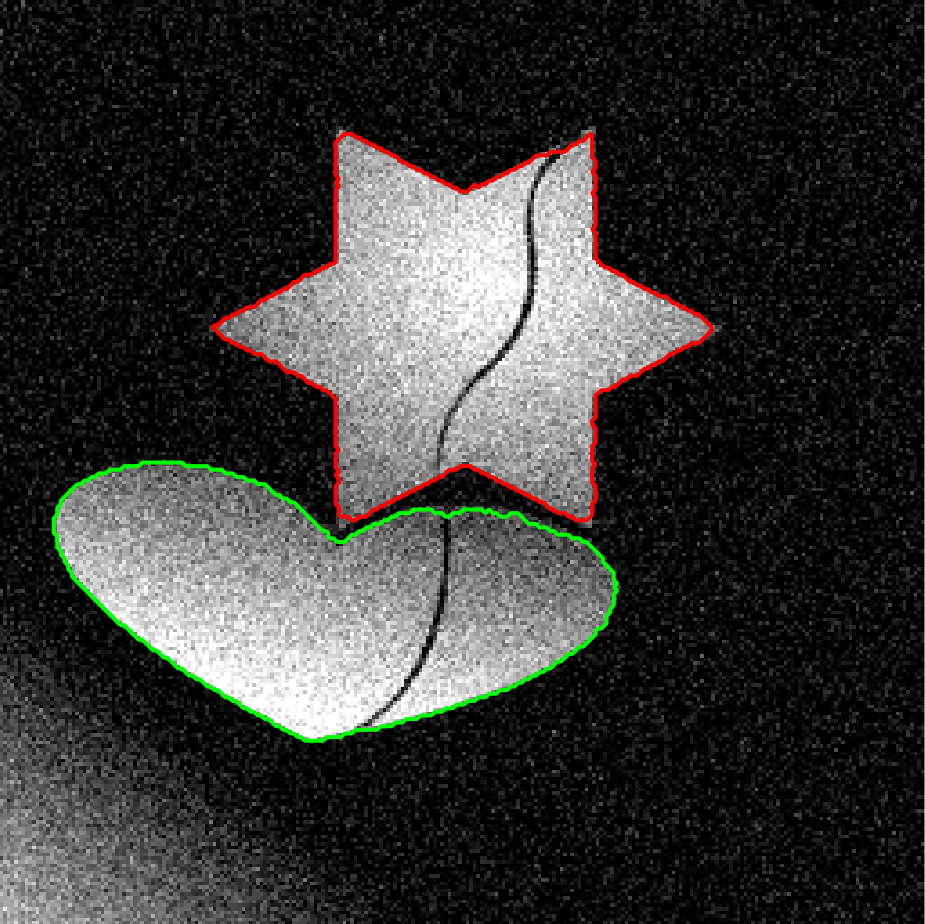}
\end{minipage}}
\subfigure[Revised image $I_f/\beta$]{
\begin{minipage}[b]{0.21\textwidth}
\includegraphics[width=1\textwidth]{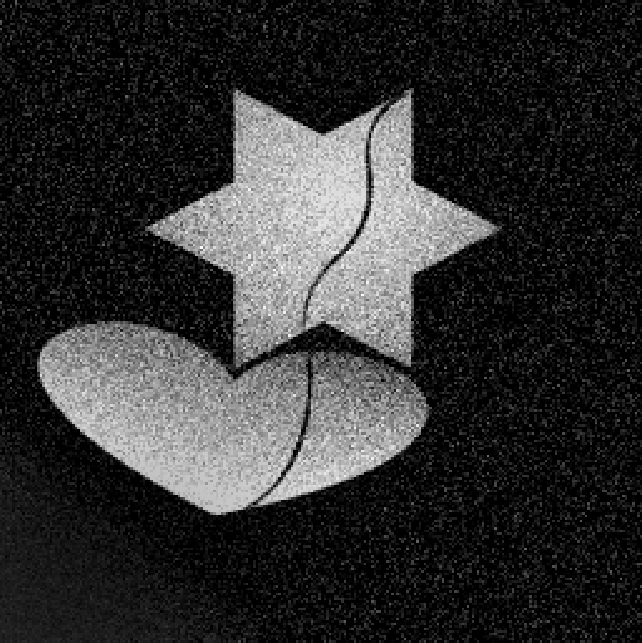}
\end{minipage}}
\subfigure[Deformed image]{
\begin{minipage}[b]{0.21\textwidth}
\includegraphics[width=1\textwidth]{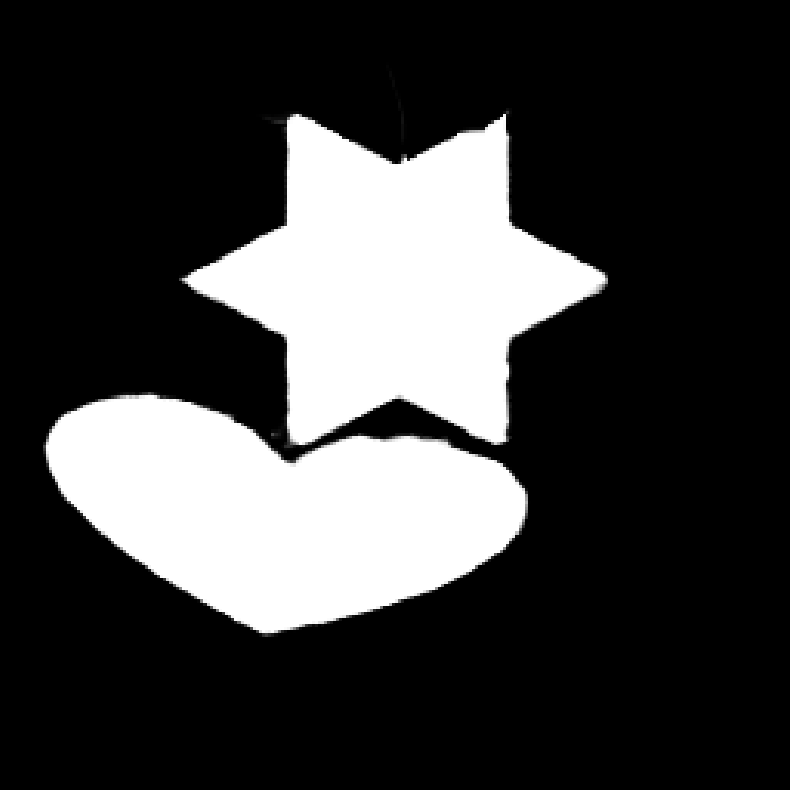}
\end{minipage}}
\subfigure[Deformed mesh]{
\begin{minipage}[b]{0.21\textwidth}
\includegraphics[width=1\textwidth]{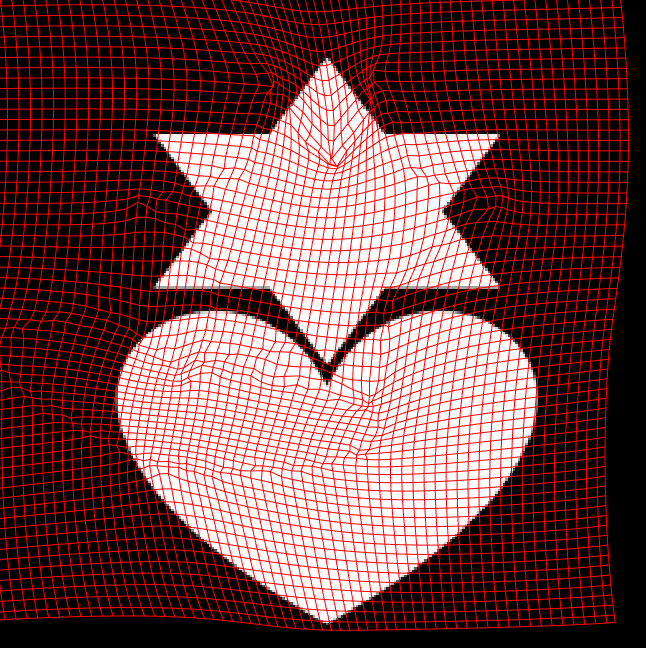}
\end{minipage}}
\subfigure[Displacement field]{
\begin{minipage}[b]{0.21\textwidth}
\includegraphics[width=1\textwidth]{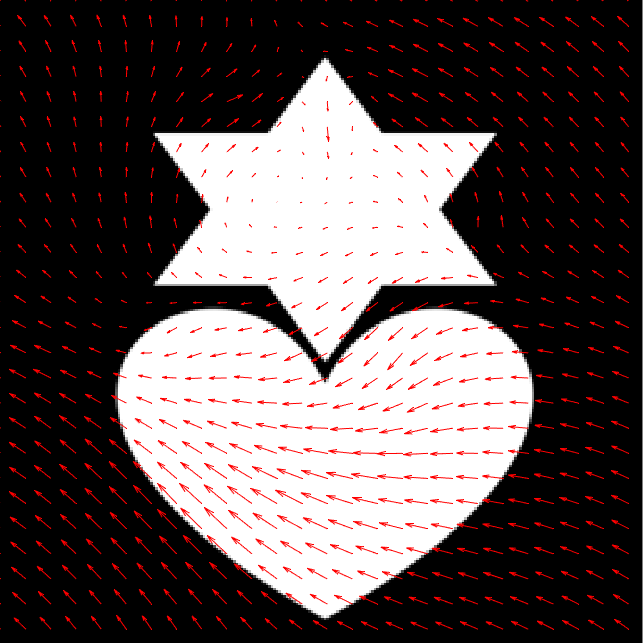}
\end{minipage}}
\subfigure[Variation of each individual term in $\mathcal{E}$]{
\begin{minipage}[b]{0.4\textwidth}
\includegraphics[width=1\textwidth]{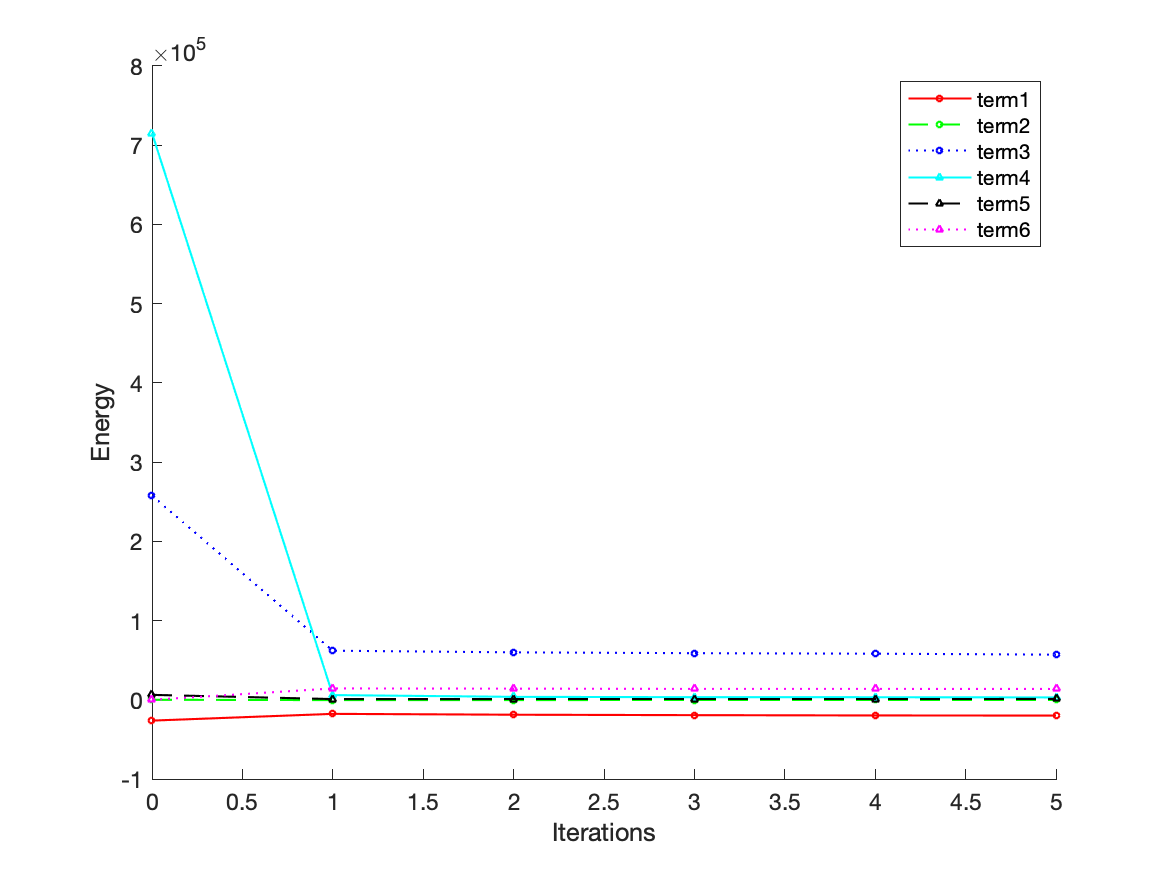}
\end{minipage}}
\subfigure[Decay of the energy function $\mathcal{E}$]{
\begin{minipage}[b]{0.4\textwidth}
\includegraphics[width=1\textwidth]{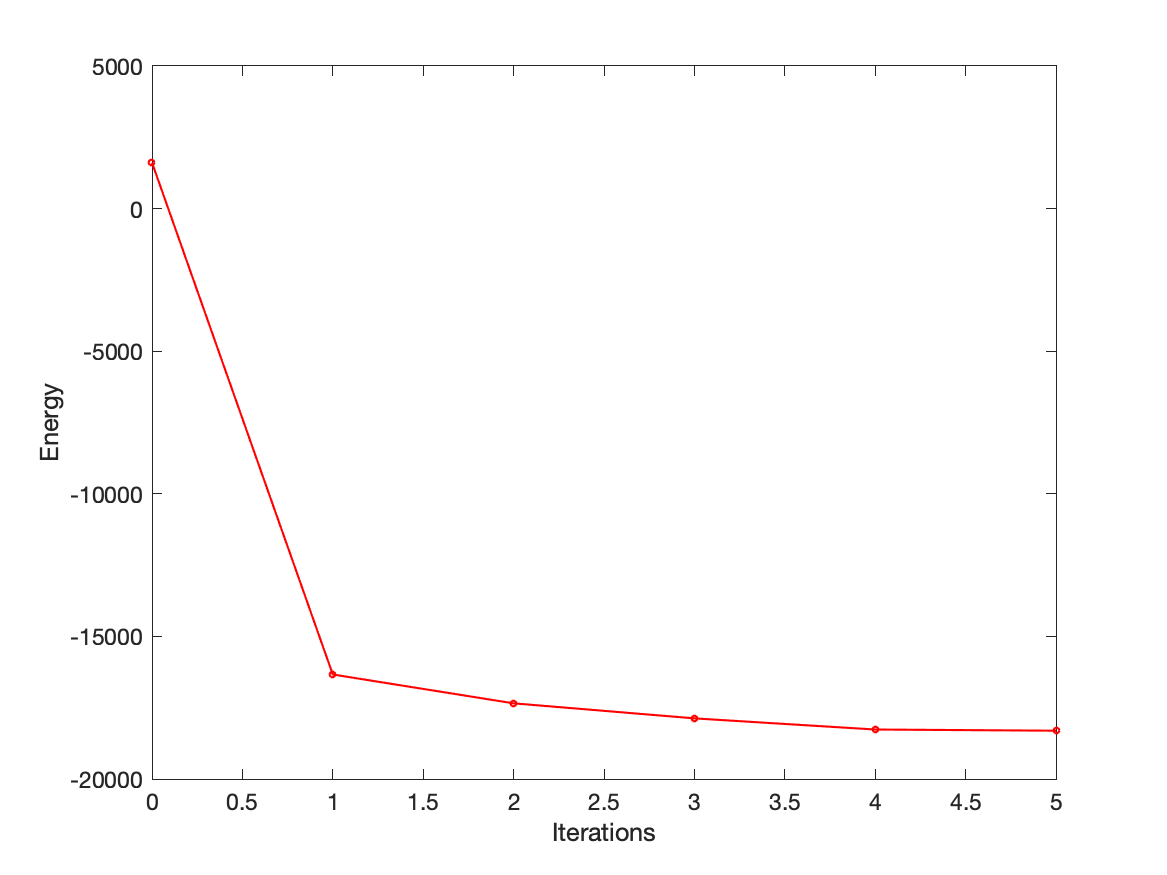}
\end{minipage}}
\caption{Comparison of the proposed joint model with segmentation or registration only method.}
\label{fig:StarHeart}
\end{figure}

\subsection{Test on thigh muscle MR images}\label{sec:TestMuscle}
We evaluate our proposed model on five T1-weighted MRI 3D volumes. We choose one of them as the moving image which has ground truth segmentation by professional doctors through manual annotation, and the other four as fixed images to be segmented for validation. The proposed algorithm is implemented on MATLAB R2019b 9.7 environment. The four segmentation clusters are: quadriceps (vastus medials, vastus lateralis, vastus intermedius, rectus femoris), hamstrings (semimembranosus, semitendinosus, biceps femoris long head, biceps femoris short head), other muscle groups (adductor group, gracilis, sartorius), and the rest (background, fat, cortical bone, bone marrow). For 3D thigh muscle data, we choose the middle axial cross section to visualize the segmentation results. \cref{fig:MuscleCompare} shows the comparison of the proposed model with segmentation only or registration only method. The (b)-(c), (f)-(g) are segmentation results of minimizing $\mathcal{E}_{\mathrm{Seg}}$ and $\mathcal{E}$ respectively. We can see that minimizing $\mathcal{E}_{\mathrm{Seg}}$ can hardly separate the whole muscle from others, as the skin layer of the thigh is also divided into the muscle cluster. However, minimizing $\mathcal{E}$ can well separate out the whole muscle as a prior provided by the registration. (d) and (h) are deformed atlas by minimizing $\mathcal{E}_{\mathrm{Reg}}$ and $\mathcal{E}$ respectively. We can clearly see that the proposed model obtains an accurate segmentation result of thigh muscles with the geometric structure provided by the segmentation, but method with registration only can't. This figure further illustrates the advantages of joint model. In the following experiments, without specification, we use the data shown in \cref{fig:MuscleCompare} (a) as the moving image.
\cref{fig:MuscleProposed} visualizes the results of our proposed model from different aspects. We can see that the deformed image is very similar with the fixed image and the segmentation result is accurate in visual. \cref{fig:MuscleAllSlice} lists the segmentation results of all slices by the proposed model, which shows that all of the cross sections have an accurate segmentation, thus illustrating the reliability and effectiveness of the proposed model.

\begin{figure}[htp]
\centering
\subfigure[Image $I_m$]{\label{fig:MuscleComparea}
\begin{minipage}[b]{0.23\textwidth}
\includegraphics[width=1\textwidth]{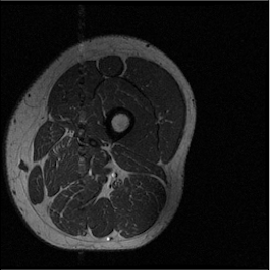}
\end{minipage}}
\subfigure[Label by $\min\mathcal{E}_{\mathrm{Seg}}$]{
\begin{minipage}[b]{0.23\textwidth}
\includegraphics[width=1\textwidth]{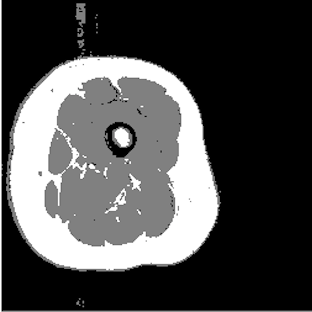}
\end{minipage}}
\subfigure[Contour by $\min\mathcal{E}_{\mathrm{Seg}}$]{
\begin{minipage}[b]{0.23\textwidth}
\includegraphics[width=1\textwidth]{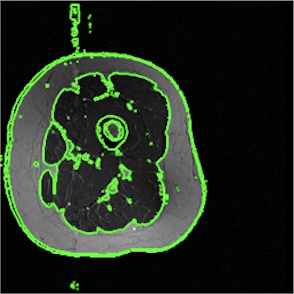}
\end{minipage}}
\subfigure[Deformed atlas by $\min\mathcal{E}_{\mathrm{Reg}}$]{
\begin{minipage}[b]{0.23\textwidth}
\includegraphics[width=1\textwidth]{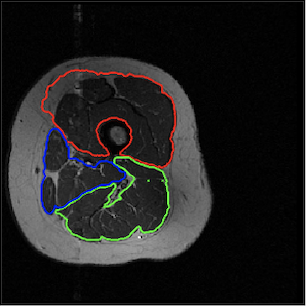}
\end{minipage}}
\subfigure[Image $I_f$]{
\begin{minipage}[b]{0.23\textwidth}
\includegraphics[width=1\textwidth]{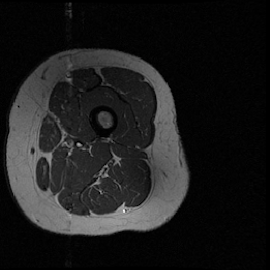}
\end{minipage}}
\subfigure[Label by $\min\mathcal{E}$]{
\begin{minipage}[b]{0.23\textwidth}
\includegraphics[width=1\textwidth]{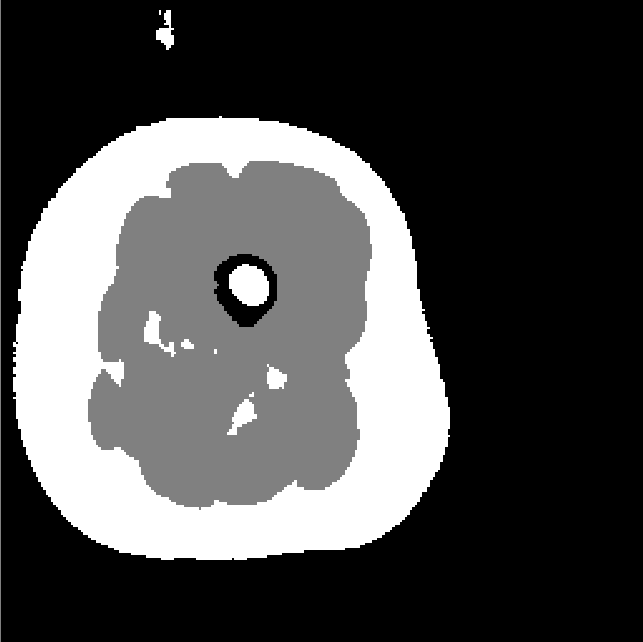}
\end{minipage}}
\subfigure[Contour by $\min\mathcal{E}$]{
\begin{minipage}[b]{0.23\textwidth}
\includegraphics[width=1\textwidth]{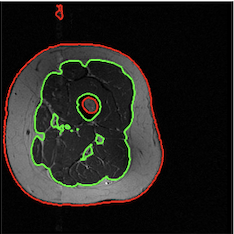}
\end{minipage}}
\subfigure[Deformed atlas by $\min\mathcal{E}$]{
\begin{minipage}[b]{0.23\textwidth}
\includegraphics[width=1\textwidth]{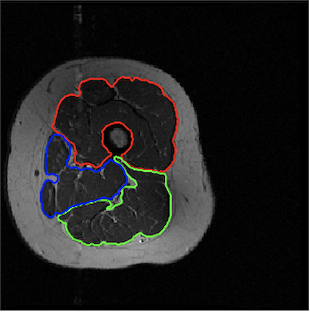}
\end{minipage}}
\caption{Comparison of the proposed model with segmentation or registration only method.} 
\label{fig:MuscleCompare}
\end{figure}

\begin{figure}[htp]
\centering
\subfigure[Image $I_f$]{
\begin{minipage}[b]{0.23\textwidth}
\includegraphics[width=1\textwidth]{figure/muscle10_slice8.png}
\end{minipage}}
\subfigure[Deformed image]{
\begin{minipage}[b]{0.23\textwidth}
\includegraphics[width=1\textwidth]{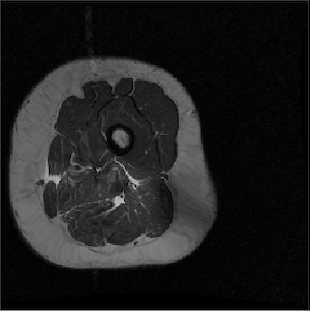}
\end{minipage}}
\subfigure[Deformed mesh]{
\begin{minipage}[b]{0.23\textwidth}
\includegraphics[width=1\textwidth]{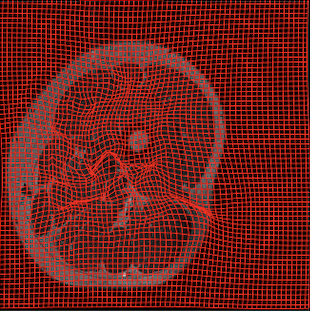}
\end{minipage}}
\subfigure[Displacement field]{
\begin{minipage}[b]{0.23\textwidth}
\includegraphics[width=1\textwidth]{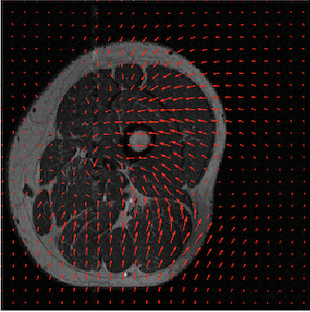}
\end{minipage}}
\subfigure[Difference of $I_f$ and $I_m$]{
\begin{minipage}[b]{0.23\textwidth}
\includegraphics[width=3.6cm,height=3.6cm]{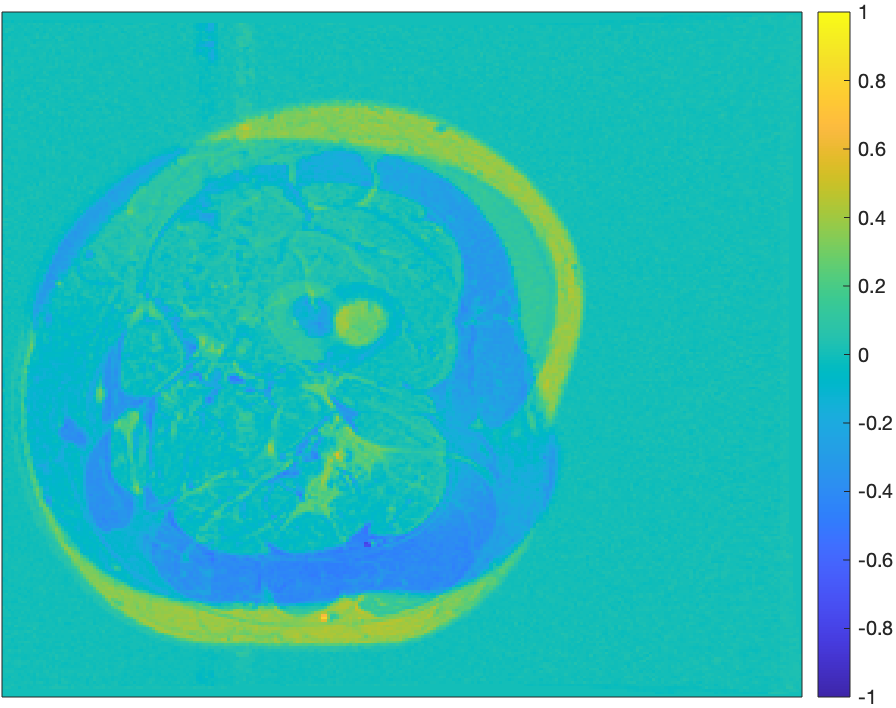}
\end{minipage}}
\subfigure[Difference of $I_f$ and deformed image]{
\begin{minipage}[b]{0.23\textwidth}
\includegraphics[width=3.6cm,height=3.6cm]{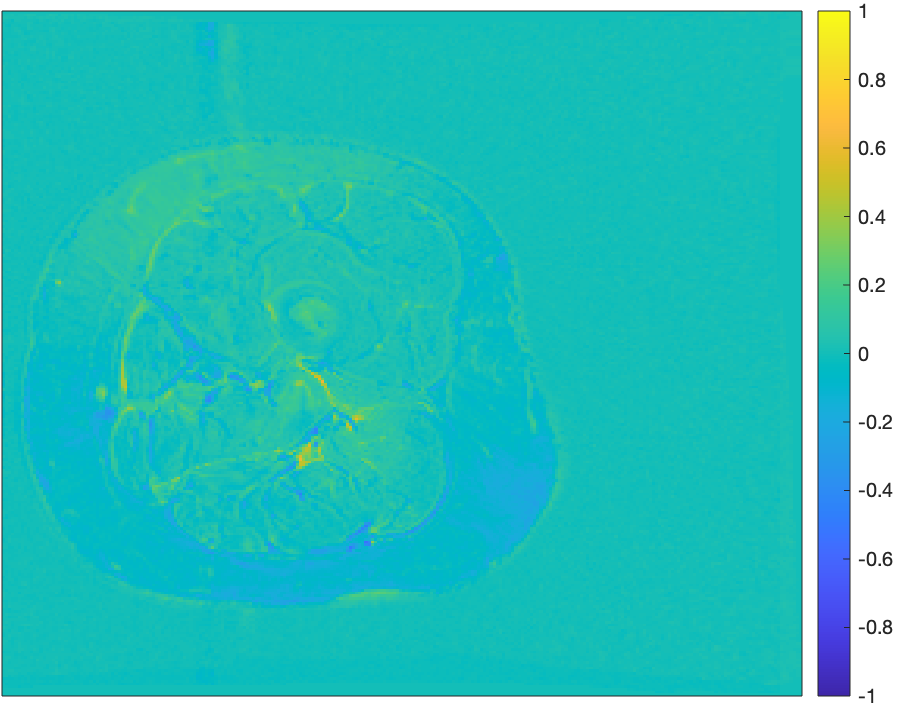}
\end{minipage}}
\subfigure[Segmentation result shown in 2D]{
\begin{minipage}[b]{0.23\textwidth}
\includegraphics[width=1\textwidth]{figure/muscle10_slice8_deformed_edge.png}
\end{minipage}}
\subfigure[Segmentation result shown in 3D]{
\begin{minipage}[b]{0.23\textwidth}
\includegraphics[width=1\textwidth]{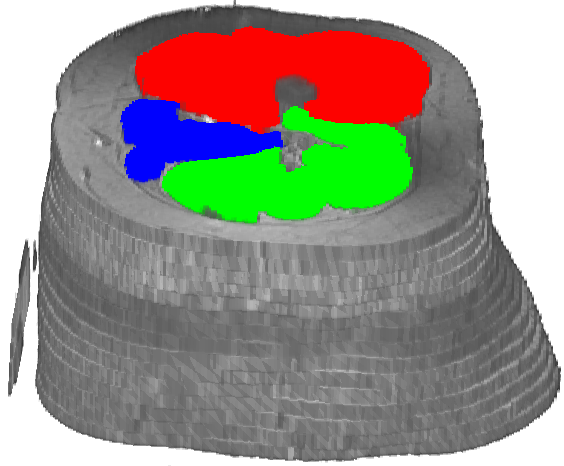}
\end{minipage}}
\caption{Results of the proposed model (the moving image $I_m$ is shown in \cref{fig:MuscleCompare} (a)).} 
\label{fig:MuscleProposed}
\end{figure}

\begin{figure}[htp]
\centering
\subfigure{
\begin{minipage}[b]{0.98\textwidth}
\includegraphics[width=1\textwidth]{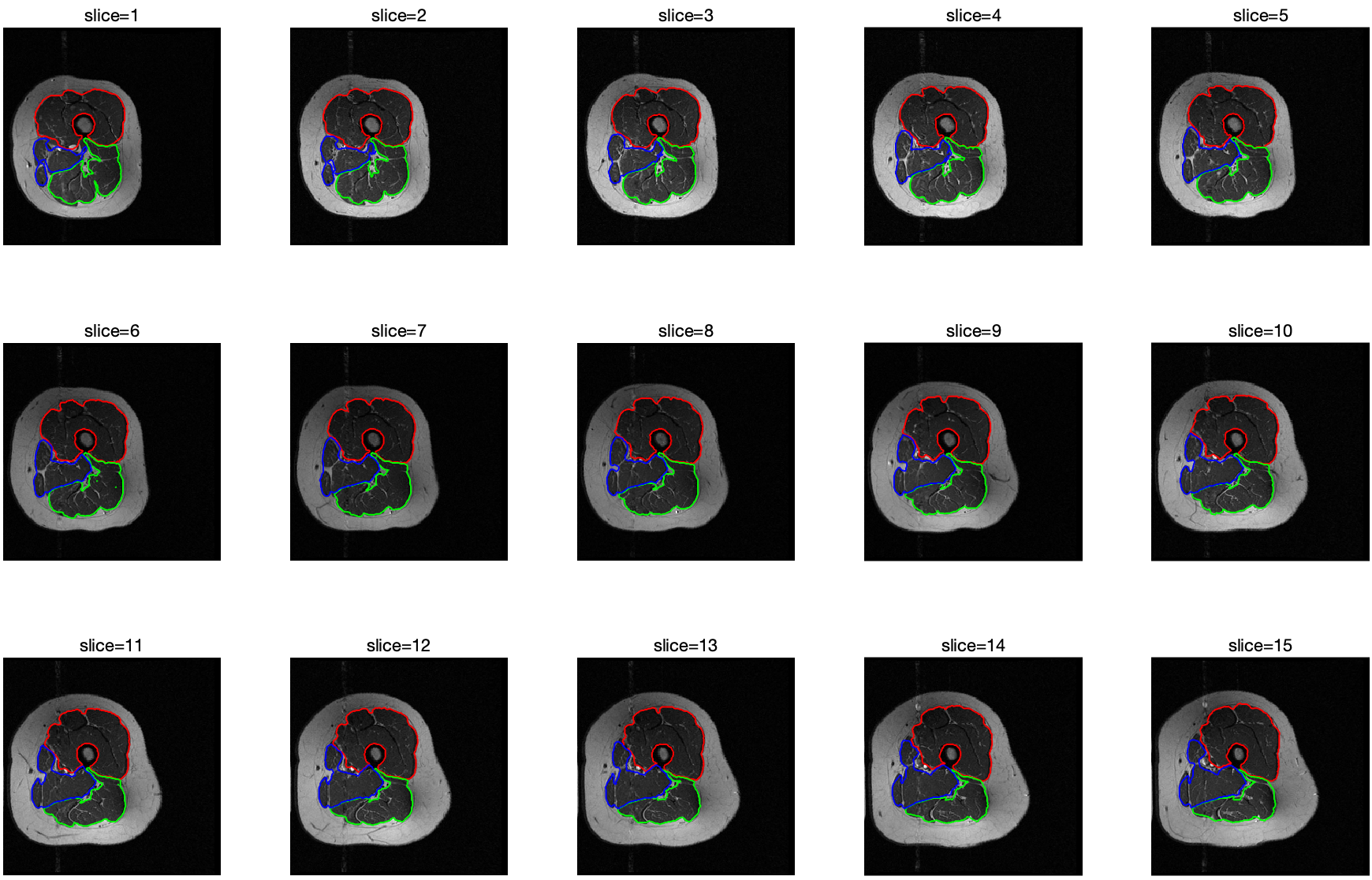}
\end{minipage}}
\caption{The segmentation results of all slices by the proposed model.}\label{fig:MuscleAllSlice}
\end{figure}

In the last, we compare our proposed model with other three methods: a non-parametric diffeomorphic image registration algorithm based on Thirion's demons algorithm \cite{vercauteren2009diffeomorphic} (we denote it as M1 below), a non-rigid image registration based on B-spline composition and level sets \cite{chan2017two} (denoted as M2), and an unified segmentation \cite{ashburner2005unified} introduced in Section 2.2 (denoted as M3). For convenience, we denote our proposed model as M4 in the following. Note that M1 and M2 are both atlas based segmentation methods as they depend on image registration only. M3 is another joint segmentation and registration method. \cref{fig:CompareResult1} and \cref{fig:CompareResult2} show the segmentation results of these methods evaluated on two subjects (shown in \cref{fig:CompareSubjects}). These visual results demonstrate the superiority of the proposed method than others.

\begin{figure}[htp]
\centering
\subfigure[Moving image]{
\begin{minipage}[b]{0.23\textwidth}
\includegraphics[width=1\textwidth]{figure/N1_slice8.png}
\end{minipage}}
\subfigure[Subject 1]{
\begin{minipage}[b]{0.23\textwidth}
\includegraphics[width=1\textwidth]{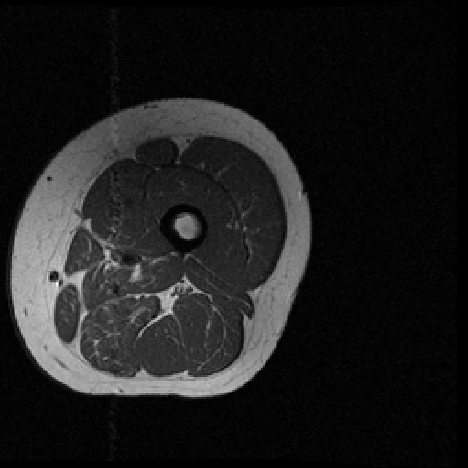}
\end{minipage}}
\subfigure[Subject 2]{
\begin{minipage}[b]{0.23\textwidth}
\includegraphics[width=1\textwidth]{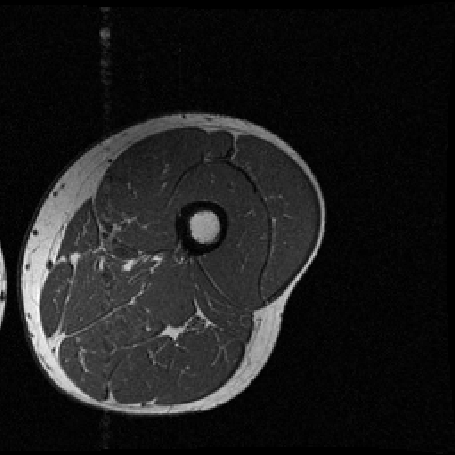}
\end{minipage}}
\caption{Visualization of middle slices of the moving image, Subject 1 and Subject 2.}
\label{fig:CompareSubjects}
\end{figure}

\begin{figure}[htp]
\centering
\subfigure{
\begin{minipage}[b]{0.23\textwidth}
\includegraphics[width=1\textwidth]{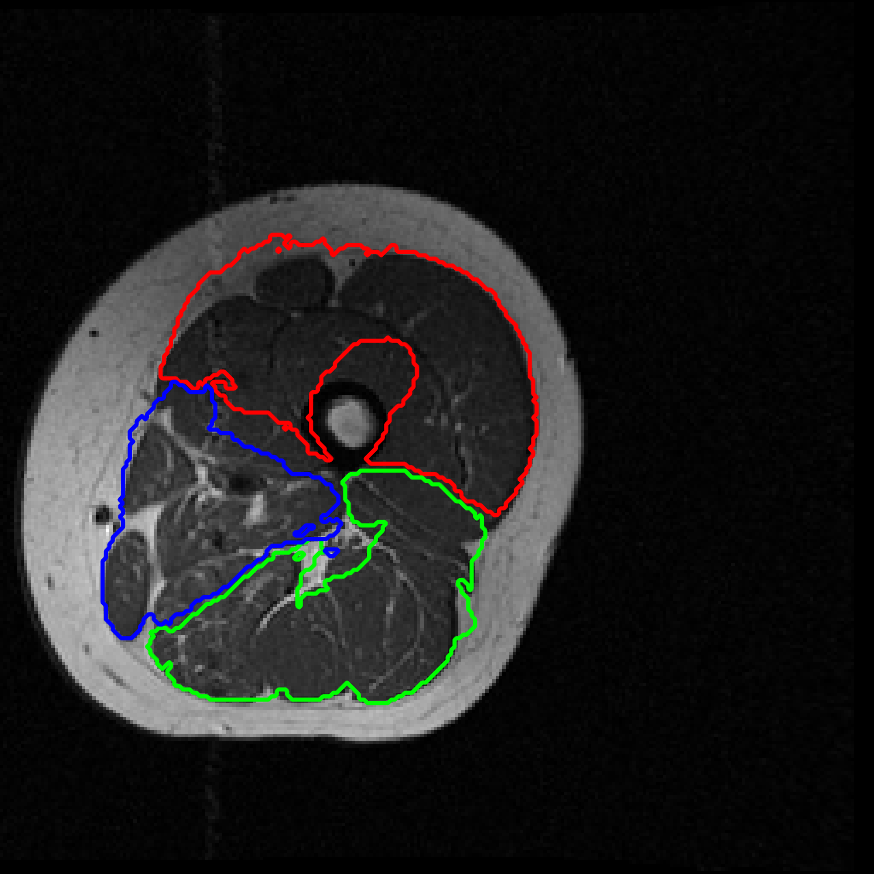}
\end{minipage}}
\subfigure{
\begin{minipage}[b]{0.23\textwidth}
\includegraphics[width=1\textwidth]{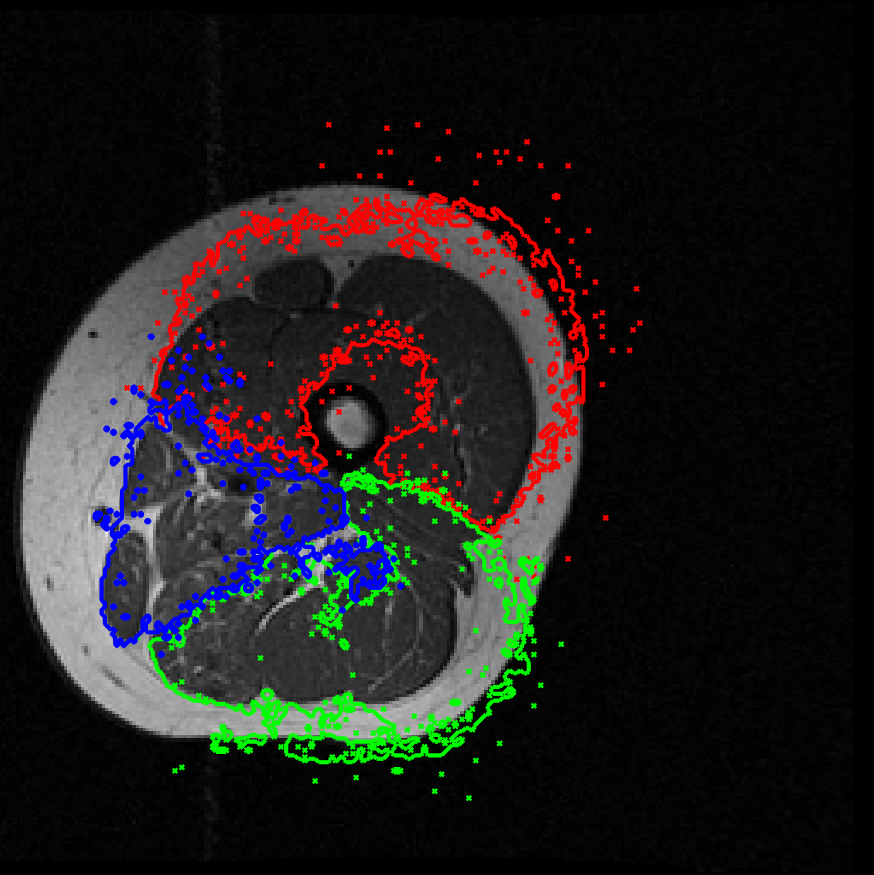}
\end{minipage}}
\subfigure{
\begin{minipage}[b]{0.23\textwidth}
\includegraphics[width=1\textwidth]{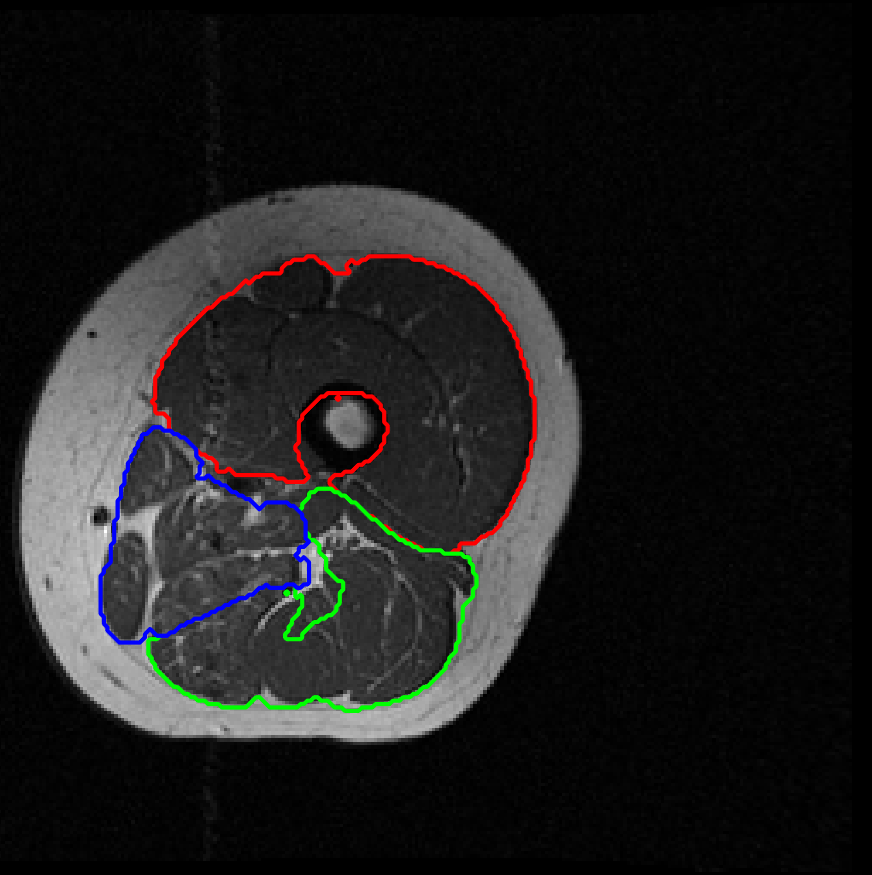}
\end{minipage}}
\subfigure{
\begin{minipage}[b]{0.23\textwidth}
\includegraphics[width=1\textwidth]{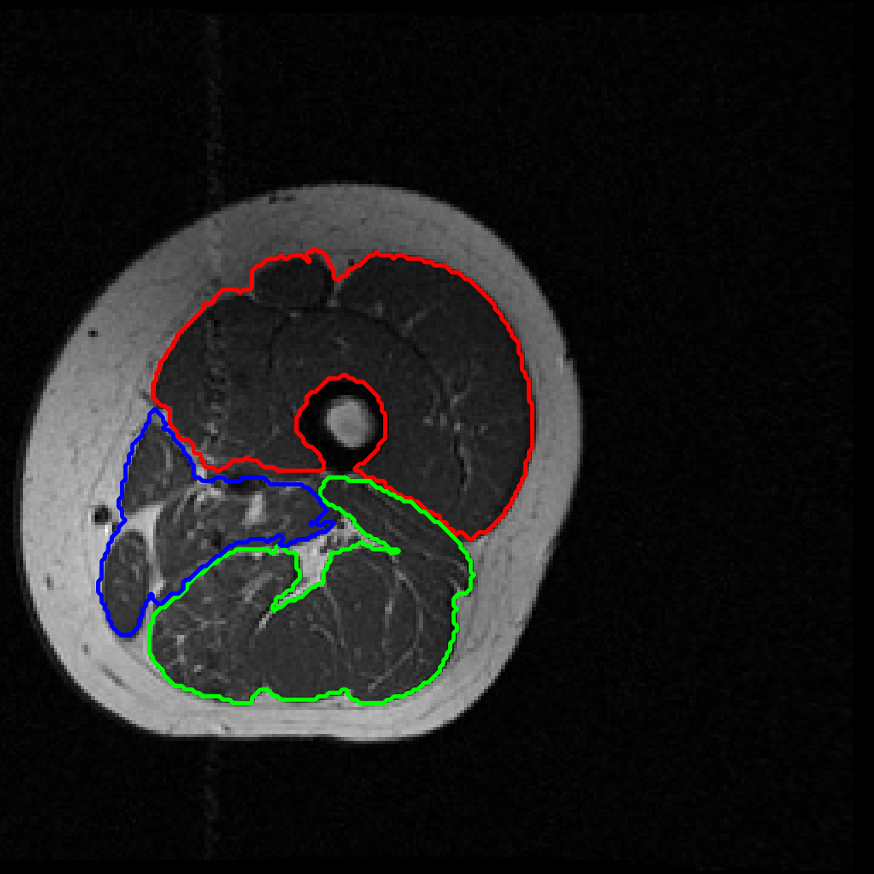}
\end{minipage}}\\
\subfigure{
\begin{minipage}[b]{0.23\textwidth}
\includegraphics[width=1\textwidth]{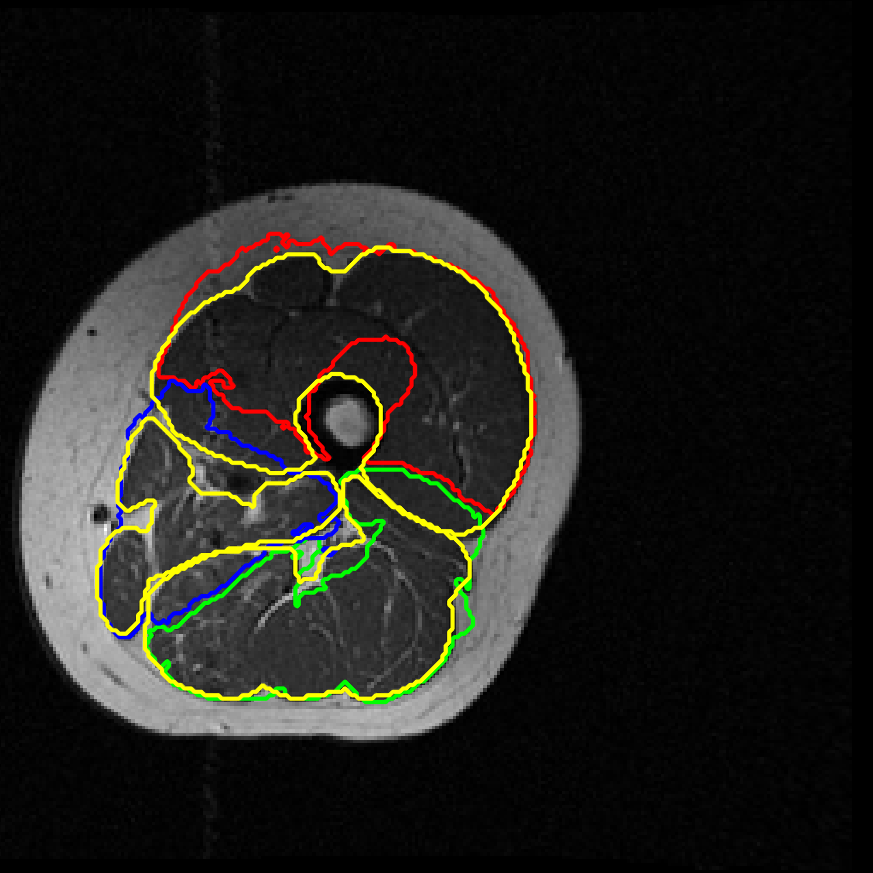}
\caption*{(a) M1 \cite{vercauteren2009diffeomorphic}}
\end{minipage}}
\subfigure{
\begin{minipage}[b]{0.23\textwidth}
\includegraphics[width=1\textwidth]{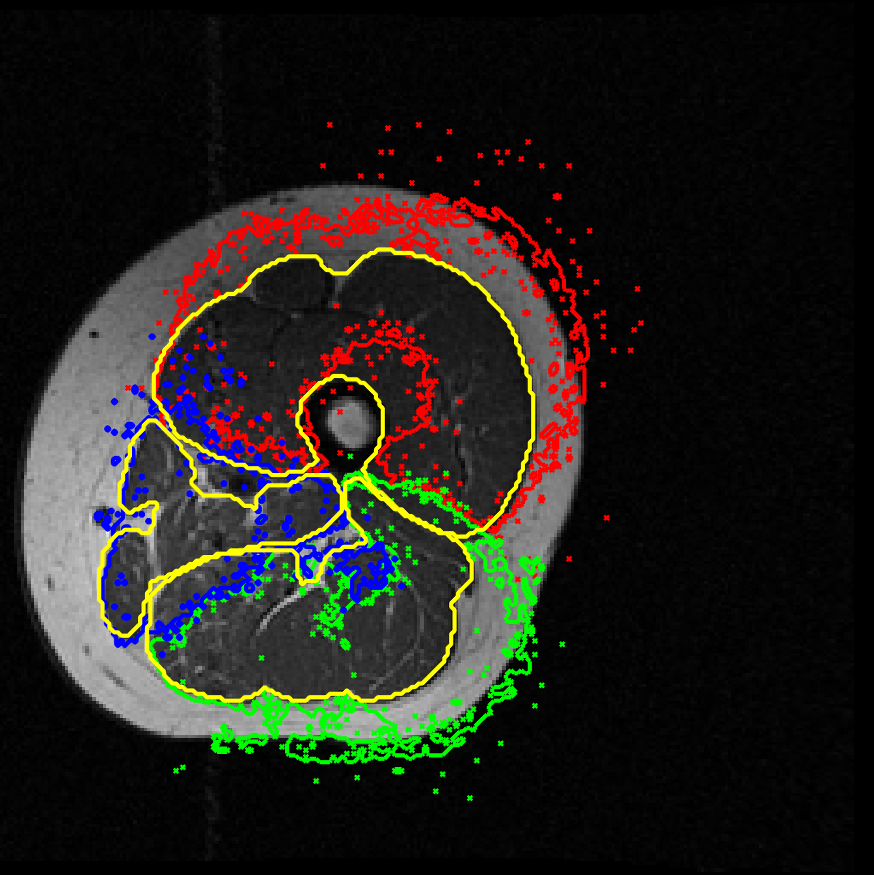}
\caption*{(b) M2 \cite{chan2017two}}
\end{minipage}}
\subfigure{
\begin{minipage}[b]{0.23\textwidth}
\includegraphics[width=1\textwidth]{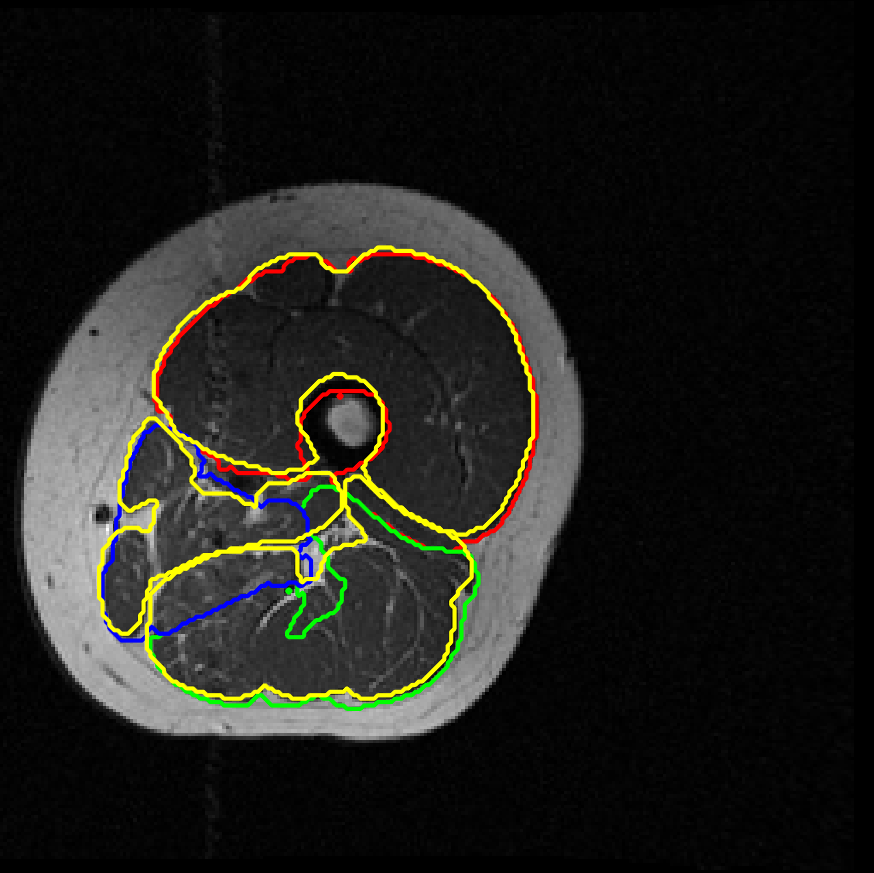}
\caption*{(c) M3 \cite{ashburner2005unified}}
\end{minipage}}
\subfigure{
\begin{minipage}[b]{0.23\textwidth}
\includegraphics[width=1\textwidth]{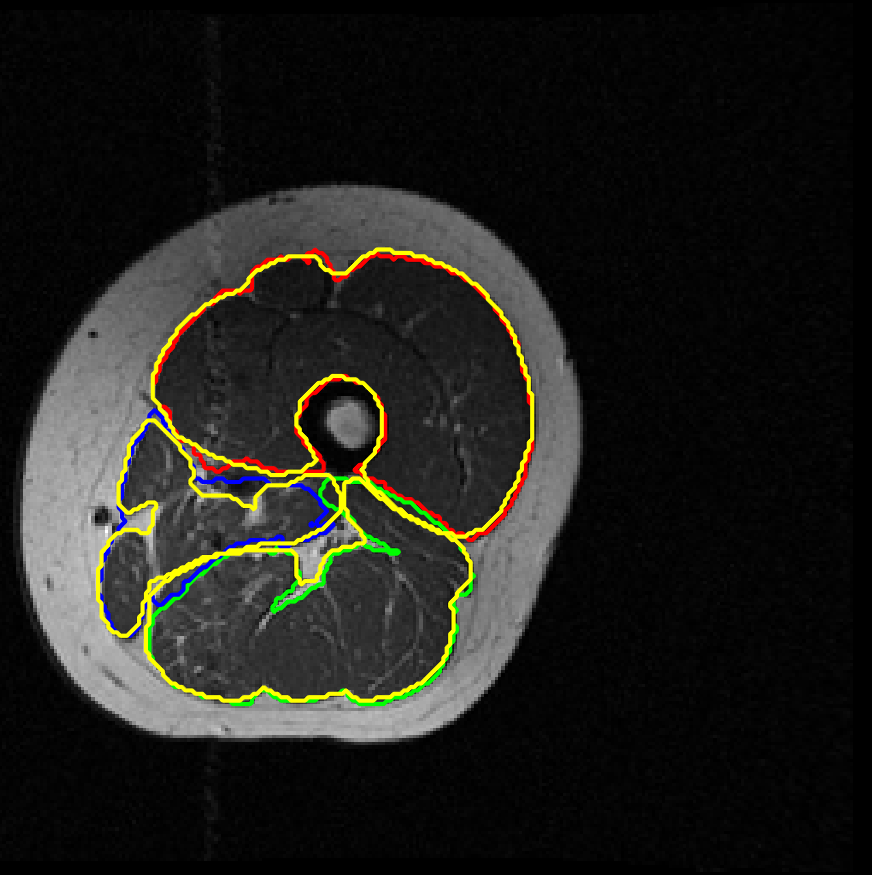}
\caption*{(d) M4 (proposed)}
\end{minipage}}
\caption{Comparison of segmentation results of M1, M2, M3, and M4 on Subject 1 (see \cref{fig:CompareSubjects}). The first line: segmentation results of different methods. The second line: segmentation results with ground truths overlaying in yellow color.}
\label{fig:CompareResult1}
\end{figure}

\begin{figure}[htp]
\centering
\subfigure{
\begin{minipage}[b]{0.23\textwidth}
\includegraphics[width=1\textwidth]{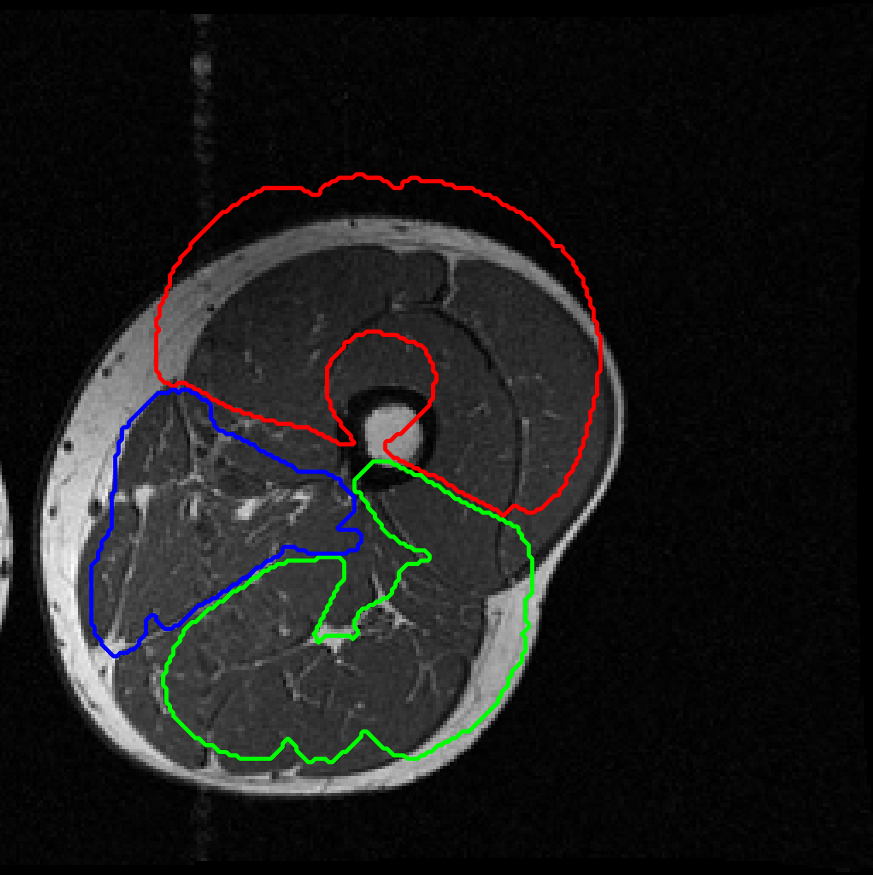}
\end{minipage}}
\subfigure{
\begin{minipage}[b]{0.23\textwidth}
\includegraphics[width=1\textwidth]{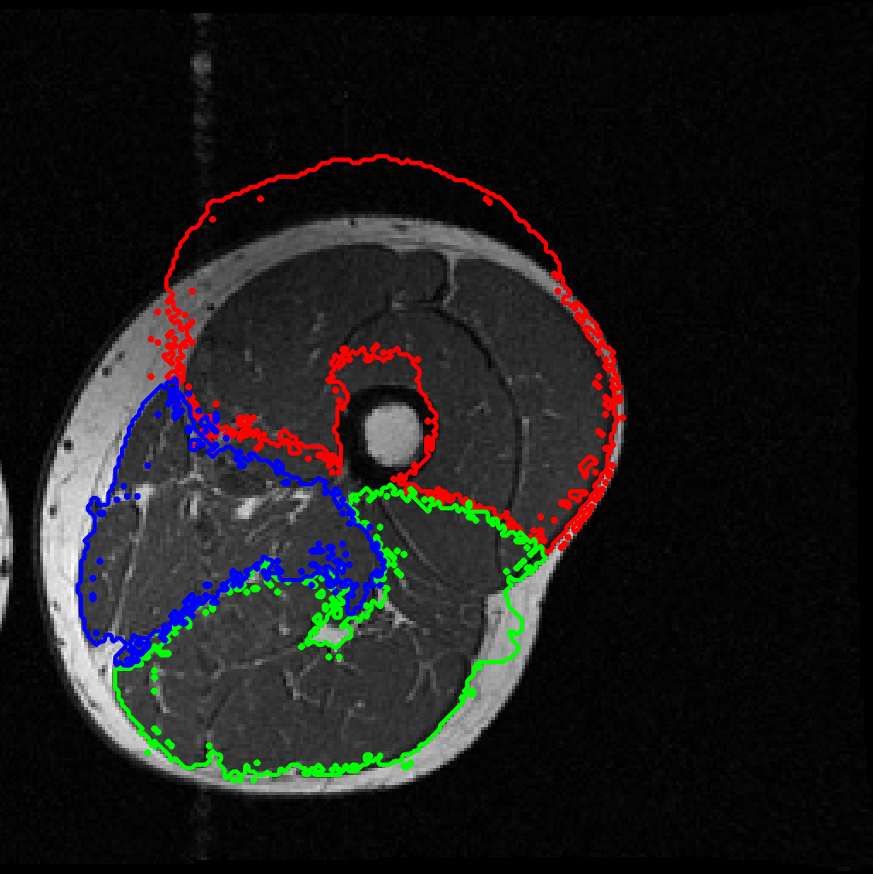}
\end{minipage}}
\subfigure{
\begin{minipage}[b]{0.23\textwidth}
\includegraphics[width=1\textwidth]{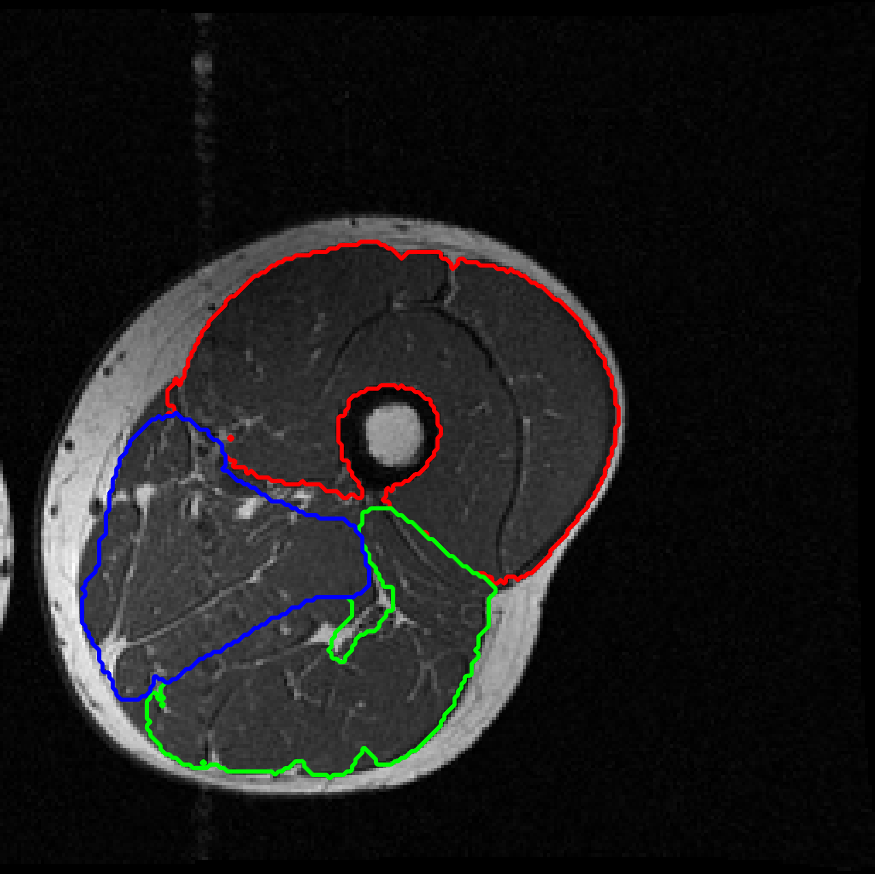}
\end{minipage}}
\subfigure{
\begin{minipage}[b]{0.23\textwidth}
\includegraphics[width=1\textwidth]{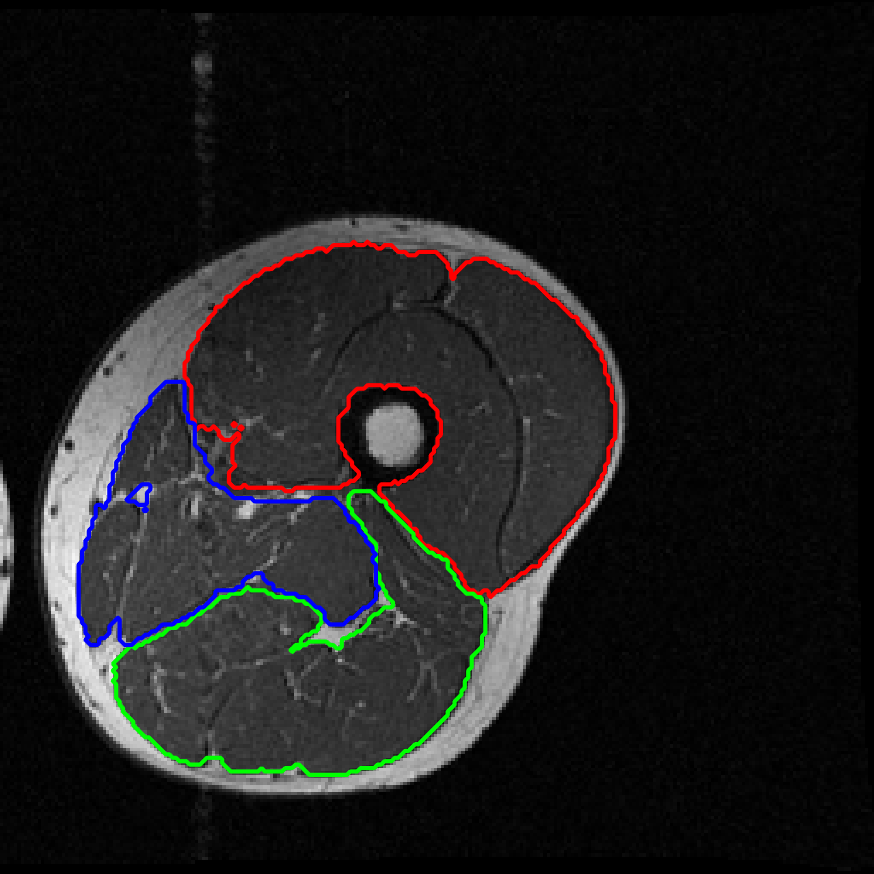}
\end{minipage}}\\
\subfigure{
\begin{minipage}[b]{0.23\textwidth}
\includegraphics[width=1\textwidth]{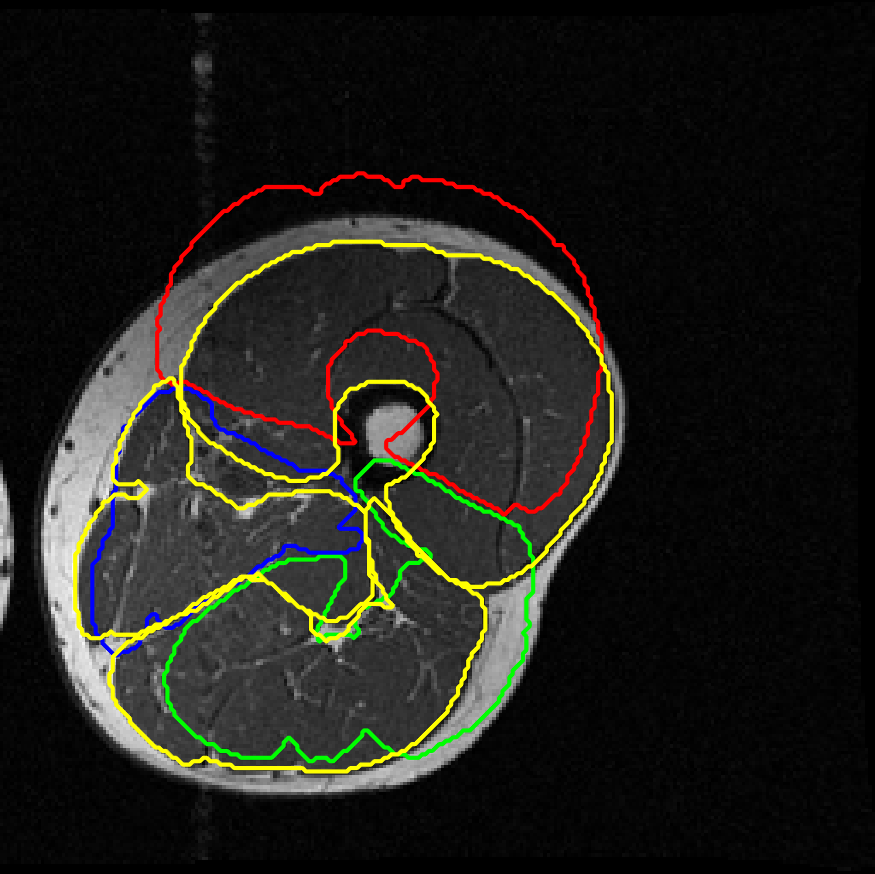}
\caption*{(a) M1 \cite{vercauteren2009diffeomorphic}}
\end{minipage}}
\subfigure{
\begin{minipage}[b]{0.23\textwidth}
\includegraphics[width=1\textwidth]{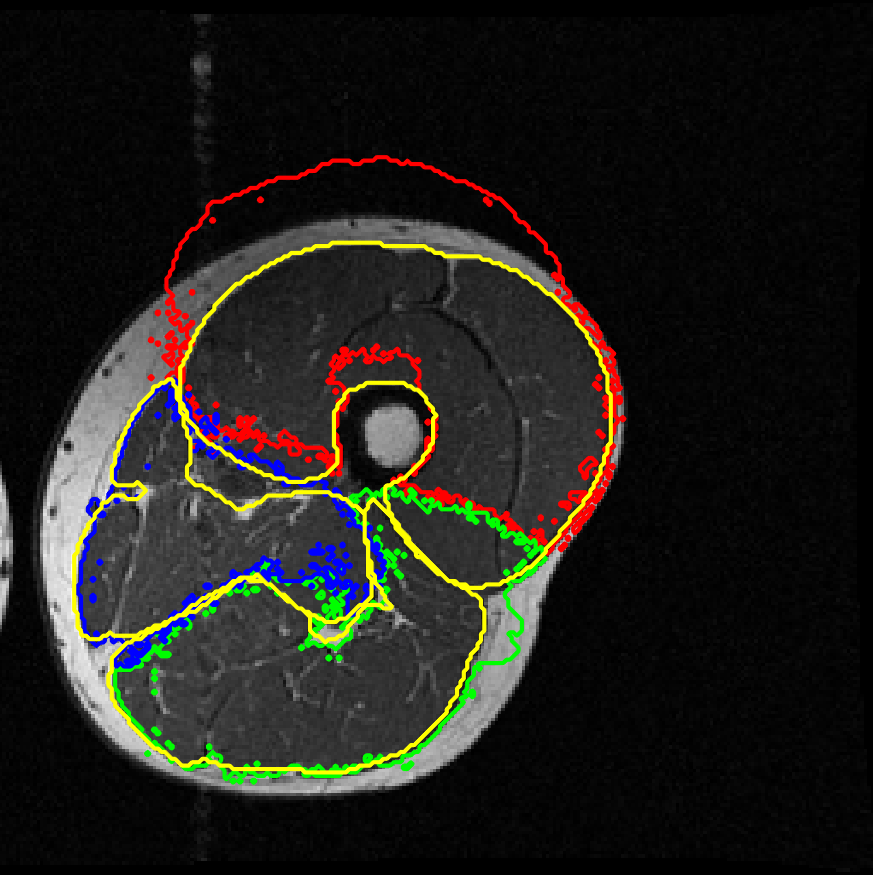}
\caption*{(b) M2 \cite{chan2017two}}
\end{minipage}}
\subfigure{
\begin{minipage}[b]{0.23\textwidth}
\includegraphics[width=1\textwidth]{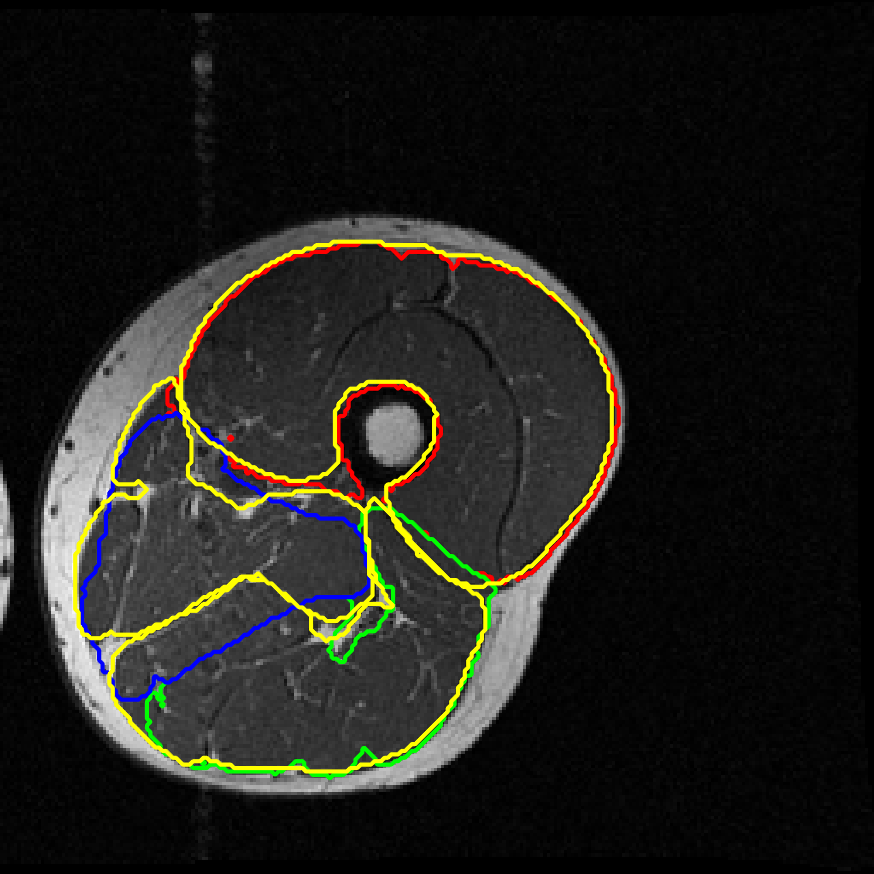}
\caption*{(c) M3 \cite{ashburner2005unified}}
\end{minipage}}
\subfigure{
\begin{minipage}[b]{0.23\textwidth}
\includegraphics[width=1\textwidth]{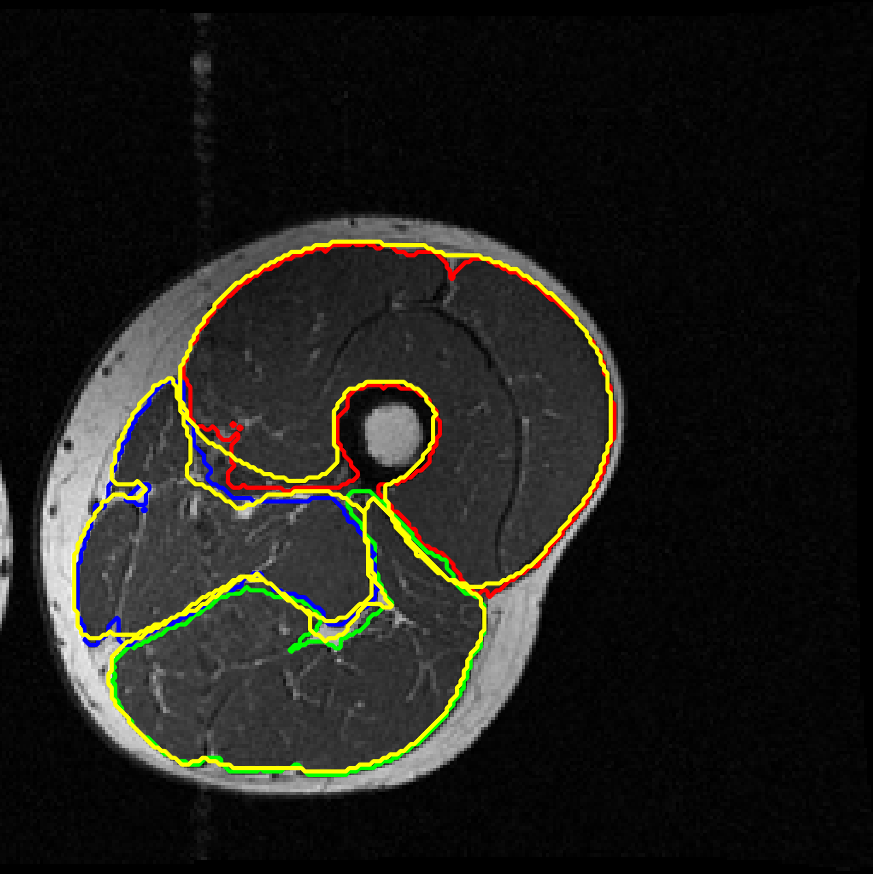}
\caption*{(d) M4 (proposed)}
\end{minipage}}
\caption{Comparison of segmentation results of M1, M2, M3, and M4 on Subject 2 (see \cref{fig:CompareSubjects}). The first line: segmentation results of different methods. The second line: segmentation results with ground truths overlaying in yellow color.}
\label{fig:CompareResult2}
\end{figure}

To evaluate the performance of M1, M2, M3 and the proposed M4, we calculate the segmentation accuracy compared with ground truths (obtained from manual segmentation). We adopt modified Jaccard (mJ), Dice similarity coefficient (DSC), and average surface distance (ASD) as the evaluation indexes. Let $X$ be the segmentation result from the algorithm and $Y$ the ground truth, $\partial X$ the segmentation boundary of $X$, $\partial Y$ the boundary of $Y$. We have definitions:
\begin{itemize}
\item modified Jaccard:
$$\mathrm{mJ}(X,Y)=\frac{|X\cap Y|}{|X|+|Y|-|X\cap Y|}.$$
\item Dice Similarity Coefficient:
$$\mathrm{DSC}(X,Y)=\frac{2|X\cap Y|}{|X|+|Y|}=\frac{2\mathrm{TP}}{\mathrm{FN}+2\mathrm{TP}+\mathrm{FP}}.$$
Here FP represents the number of false positive (i.e. the total number of the misclassified voxels of the background), FN is the number of false negative (i.e. the total number of the misclassified voxels of the object), and TP is the true positive (i.e. total number of the correctly classified pixels).
\item Average Surface Distance:
$$\mathrm{ASD}(\partial X,\partial Y)=\frac{\sum\limits_{x\in \partial X} \min\limits_{y\in \partial Y}\|x-y\|+\sum\limits_{y\in \partial Y} \min\limits_{x\in \partial X}\|y-x\|}{|\partial X|+|\partial Y|}.$$
\end{itemize}
The means and standard deviations of mJ, DSC and ASD by methods M1, M2, M3 and M4, on four data sets, are listed in \cref{tab:evaluation}, respectively. This table shows that our proposed method has the highest mJ and DSC scores and smallest ASD, which illustrates the superiority of the proposed method from the perspective of numerical evaluation.

\begin{table}[htp]
\caption{mJ, DSC and ASD (mean $\pm$ std) for M1 \cite{vercauteren2009diffeomorphic}, M2 \cite{chan2017two}, M3 \cite{ashburner2005unified}, and M4 (proposed).}\label{tab:evaluation}
\begin{center}
\begin{tabular}{cc|c|c|c|c}
\hline
~ & ~ & Quadriceps & Hamstrings & Other muscles & Average \\
\hline
\multirow{3}{*}{M1} & mJ &  $0.3627\pm0.2139$  & $0.5203\pm0.1027$ & $0.4563\pm0.1439$ & $0.4464\pm0.1462$ \\
~& DSC &  $0.5425\pm0.2352$  & $0.6602\pm0.0913$ & $0.5987\pm0.1523$ & $0.6005\pm0.1527$ \\
~& ASD & $5.9953\pm2.7591$  & $4.8623\pm1.7181$ & $4.6537\pm2.8222$ & $5.1704\pm2.4203$ \\
\hline
\multirow{3}{*}{M2} & mJ &  $0.4421\pm0.1521$  & $0.4933\pm0.0811$ & $0.3975\pm0.1783$ & $0.4443\pm0.1322$ \\
~& DSC &  $0.6007\pm0.1570$  & $0.6578\pm0.0696$ & $0.5522\pm0.1749$ & $0.6036\pm0.1289$ \\
~& ASD &  $6.1155\pm2.3083$  & $6.0264\pm1.6284$ & $5.3183\pm3.2121$ & $5.8201\pm2.3625$ \\
\hline
\multirow{3}{*}{M3} & mJ &  $0.7834\pm0.1077$  & $0.6070\pm0.1167$ & $0.4741\pm0.1526$ & $0.6215\pm0.1228$ \\
~& DSC & $0.8753\pm0.0719$  & $0.7502\pm0.0965$ & $0.6313\pm0.1539$ & $0.7523\pm0.1055$ \\
~& ASD & $2.5174\pm1.0276$  & $4.3736\pm1.2004$ & $4.7323\pm2.0927$ & $3.8744\pm1.3975$ \\
\hline
\multirow{3}{*}{M4} & mJ & $\mathbf{ 0.9190\pm0.0071}$  & $\mathbf{ 0.8828\pm0.0181}$ & $\mathbf{ 0.7594\pm0.0412}$ & $\mathbf{ 0.8537\pm0.0154}$ \\
~& DSC &  $\mathbf{ 0.9578\pm0.0039}$  & $\mathbf{ 0.9377\pm0.0102}$ & $\mathbf{ 0.8628\pm0.0261}$ & $\mathbf{ 0.9194\pm0.0097}$ \\
~& ASD &  $\mathbf{ 0.8171\pm0.1185}$  & $\mathbf{ 1.1712\pm0.1606}$ & $\mathbf{ 1.4448\pm0.1751}$ & $\mathbf{ 1.1444\pm0.0336}$ \\
\hline
\end{tabular}
\end{center}
\end{table}

\section{Conclusion and discussion}\label{sec:conclusion}
We propose a novel image segmentation model with adaptive spatial priors from joint registration, which can exploit the strong correlation between segmentation and registration, thus achieving more accurate results than sequential treatment. Besides, this framework absorbs the merits of both variational and statistical methods. The segmentation process combines GMM with spatial smoothness, intensity inhomogeneity and shape prior from registration under a variational framework, and the registration process considers two levels matching, thus making more accurate results. The numerical experiments evaluated on synthetic and thigh muscle MR images demonstrate the superiority of this proposed model. The application to automatic segmentation of thigh muscle only needs a small number of manually segmented label mapping, which is significant for clinical research as obtaining manual segmentation of thigh muscle is a time-consuming task. 

There are also some aspects can be further improved. For example, the registration model can be designed more precisely according to the characteristics of specific dataset, such as considering diffeomorphic registration or adding edge information. We just take a simple model as an example, and more elaborate registration may obtain further benefits. 
Moreover, the variability among individuals can be significant in medical images, and a more meaningful atlas will be welcomed. In this paper, we simply use one person's ground truth segmentation as the atlas. This can be improved by choosing a best one from multiple atlases or constructing an ``average" atlas, which is a meaningful statistical atlas of the global underlying anatomy of thigh muscle from a set of atlas  \cite{debroux2020variational}. More excellent registration method and more meaningful atlas would benefit the segmentation result, which will be considered in our future work.

\appendix
\section{Proof of \cref{thm:EnergyDescent}}\label{sec:proof}
Since $\mathcal{R}(\bm u^{t+1}; \bm u^{t+1})=\mathcal{R}(\bm u^{t+1})$ and $\mathcal{R}(\bm u^{t}; \bm u^{t})=\mathcal{R}(\bm u^{t})$, 
to prove $\hat{\mathcal{E}}(\Theta^{t+1}, \bm u^{t+1})\leq\hat{\mathcal{E}}(\Theta^{t}, \bm u^{t})$ is to prove 
$$\mathcal{F}(\Theta^{t+1}, \bm u^{t+1})+\lambda\mathcal{R}(\bm u^{t+1}; \bm u^{t+1})\leq \mathcal{F}(\Theta^{t}, \bm u^{t})+\lambda\mathcal{R}(\bm u^{t}; \bm u^{t}).$$
According to the second and first formulation of iteration scheme \cref{eq:iteration1}, we have 
$$\mathcal{F}(\Theta^{t+1}, \bm u^{t+1})+\lambda\mathcal{R}(\bm u^{t+1}; \bm u^{t})\leq \mathcal{F}(\Theta^{t+1}, \bm u^{t})+\lambda\mathcal{R}(\bm u^{t}; \bm u^{t})\leq \mathcal{F}(\Theta^{t}, \bm u^{t})+\lambda\mathcal{R}(\bm u^{t}; \bm u^{t}).$$
Next, it suffices to prove that
$$\mathcal{R}(\bm u^{t+1}; \bm u^{t+1})\leq\mathcal{R}(\bm u^{t+1}; \bm u^{t}).$$
It is easy to check that $\mathcal{R}(\bm u)$ is concave if the kernel function $\omega$ is semi-positive definite. The variation of $\mathcal{R}(\bm u)$ is $\delta\mathcal{R}(\bm u)=\omega\ast(1-2\bm u)$ if $\omega$ is a symmetric kernel function such as Gaussian kernel. By the property of concave function, one can have
$$\mathcal{R}(\bm u^{t+1})-\mathcal{R}(\bm u^{t})\leq\langle\bm u^{t+1}-\bm u^{t},\omega\ast(1-2\bm u^t)\rangle.$$
That is
$$\mathcal{R}(\bm u^{t+1};\bm u^{t+1})-\mathcal{R}(\bm u^{t};\bm u^{t})\leq\langle\bm u^{t+1}-\bm u^{t},\omega\ast(1-2\bm u^t)\rangle=\mathcal{R}(\bm u^{t+1};\bm u^{t})-\mathcal{R}(\bm u^{t};\bm u^{t}).$$
Thus we have
$$\mathcal{R}(\bm u^{t+1}; \bm u^{t+1})\leq\mathcal{R}(\bm u^{t+1}; \bm u^{t}),$$
which completes the proof.


\bibliographystyle{siamplain}
\bibliography{references}

\end{document}